\documentclass[journal]{IEEEtran}
\ifCLASSINFOpdf
\else
\fi

\usepackage{amsmath,mathtools,amsfonts,amsthm} 
\usepackage{amssymb}
\usepackage{algorithm}
\usepackage[noend]{algpseudocode}
\usepackage{color}
\usepackage{caption}
\usepackage{subcaption}
\usepackage{pgfplots}
\usepackage{tikz}
\usetikzlibrary{calc,positioning,shapes,shadows,arrows,fit}

\newtheorem{assumption}{Assumption}
\newtheorem{problem}{Problem}

\newtheorem{remark}{Remark}

\newtheorem{theorem}{Theorem}

\begin{document}
%
\title{KDF: Kinodynamic Motion Planning via Geometric Sampling-based Algorithms and Funnel Control}
%
%
%

\author{Christos~K.~Verginis,~\IEEEmembership{Member,~IEEE,}
        Dimos~V.~Dimarogonas,~\IEEEmembership{Senior Member,~IEEE,}
        and~Lydia~E.~Kavraki,~\IEEEmembership{Fellow,~IEEE}
\thanks{C. K. Verginis is with the Oden Institute for Computational  Engineering and Sciences at the University of Texas at Austin, TX, USA. e-mail: cverginis@utexas.edu. }
\thanks{D. V. Dimarogonas is with the School of Electrical and Engineering and Computer Science, KTH Royal Institute of Technology, Stockholm, Sweden, e-mail:  dimos@kth.se }
\thanks{L. E. Kavraki is with the Department of Computer Science at Rice University, Houston, TX, USA, e-mail: kavraki@rice.edu}
\thanks{This work was supported by the H2020 ERC Starting Grant BUCOPHSYS, the European Union’s Horizon 2020 Research and Innovation Programme under the GA No. 731869 (Co4Robots), the Swedish Research Council (VR), the Knut och Alice Wallenberg Foundation (KAW), the Swedish Foundation for Strategic Research (SSF), and the {National Science Foundation project NSF 2008720 (LEK)}.}
\thanks{Manuscript received April 19, 2005; revised August 26, 2015.}}

%
%

\markboth{Journal of \LaTeX\ Class Files,~Vol.~14, No.~8, August~2015}%
{Shell \MakeLowercase{\textit{et al.}}: Bare Demo of IEEEtran.cls for IEEE Journals}
%



\maketitle

\begin{abstract}
	We integrate sampling-based planning techniques with funnel-based feedback control to develop KDF,   a new framework for solving the kinodynamic motion-planning problem via funnel control. The considered systems evolve subject to complex, nonlinear, and uncertain dynamics (aka differential constraints). 
	Firstly, we use a \textit{geometric} planner to obtain a high-level safe path in a user-defined extended free space. Secondly, we develop a low-level funnel control algorithm that guarantees safe tracking of the path by the system. Neither the planner nor the control algorithm use information on the underlying dynamics of the system, which makes the proposed scheme easily distributable to a large variety of different systems and scenarios.  
	Intuitively, the funnel control module is able to implicitly accommodate  the dynamics of the system, allowing hence the deployment of purely geometrical motion planners. Extensive computer simulations and experimental results with a $6$-DOF robotic arm validate the proposed approach. 

\end{abstract}


\IEEEpeerreviewmaketitle

\section{Introduction}
%
%
%
%
\IEEEPARstart{M}{otion} planning of autonomous systems is one of the most fundamental problems in robotics, with numerous applications such as exploration, autonomous driving, robotic manipulation, autonomous warehouses, and multi-robot coordination \cite{lavalle2006planning,choset2005principles}. It has been extensively studied in the related literature; works have been continuously developed for the last three decades, 
exploring plenty of variations, including feedback control, discrete planning, uncertain environments, and multi-agent systems. {One important and active area of research consists of \textit{kinodynamic} motion planning, i.e., when the planning algorithm takes into account the underlying system dynamics, also known as differential constraints \cite{lavalle2006planning}. }

{In this paper we develop KDF, an algorithmic framework for the kinodynamic motion-planning problem by integrating sampling-based algorithms with intelligent feedback control. We consider 
systems that evolve subject to high-dimensional dynamics, which are highly nonlinear and uncertain. 
{The proposed framework is the integration of the following three modules. The first module is the KDF sampling-based motion planners (KDF-MP), which is a family of geometric planners that produce a path in an {``extended"} free space. By ``extended" we mean that the obtained path has some clearance with respect to the workspace obstacles. The second module is the smoothening of the derived path and its endowment with time constraints in order to produce a smooth time-varying trajectory. The third module is a funnel-based, feedback-control scheme that achieves safe tracking of this trajectory within the clearance of the extended free space. 
Neither of the aforementioned modules uses any information on the dynamics of the system. 
}
The proposed framework guarantees that the system will follow safely the derived path, free from collisions. Loosely speaking, we augment geometric motion planning algorithms with extended free-space capabilities and intelligent feedback control to provide a new solution to the kinodynamic motion-planning problem. The incorporation of the control scheme relieves the sampling-based motion planner from the system dynamics and their uncertainties.}

%
%
%

Feedback control is a popular methodology to tackle motion-planning problems, since it simultaneously solves the planning and control problems, offering a closed-form policy for the control input of the system. Artificial potential fields constitute the main tool of closed-form feedback control methods. Early works develop the so-called ``navigation functions" \cite{rimon1992exact}, appropriately constructed terms whose gradient constitutes a vector field that takes the system safely to the goal configuration from almost all initial conditions (except for a set of initial conditions that has measure zero). Navigation functions can accommodate sphere worlds (spherical obstacles), as well as star-shaped obstacles via appropriate diffeomorphic transformations \cite{rimon1991construction}. Several works build on the notion of navigation functions, and propose harmonic-based potential fields as well as point-world transformations \cite{loizou2017navigation,vlantis2018robot,tanner2003nonholonomic}. Potential field-based feedback control schemes have also accommodated multi-robot systems as well as higher order dynamics and model uncertainties \cite{loizou2014multi,dimarogonas2006feedback,tanner2005towards,panagou2016distributed,verginis2017decentralized,verginis2021adaptive}. Optimization-based feedback control techniques, such as Model Predictive Control (MPC) and barrier functions have  also been employed to tackle the motion planning problem \cite{filotheou2020robust,verginis2018communication,wang2017safety,panagou2015distributed,verginisLCSS}.

Although feedback control is a promising and convenient tool for motion planning problems, it is usually restricted to simple robot shapes, such as spheres or ellipsoids. For more complex structures, such as high-dimensional robotic manipulators, the aforementioned strategies can guarantee safety but suffer from local minima configurations, 
{or high computation times that render them impractical.}
Randomized planning has been introduced to tackle such scenarios; \cite{kavraki1996probabilistic} and \cite{hsu1997path,lavalle1998rapidly} develop the notions of probabilistic roadmaps (PRM) and {trees (Expansive Space Trees - EST, Rapidly Exploring Random Trees - RRT),} constituting efficient and probabilistically complete geometric solutions to high-dimensional motion-planning problems. These methodologies build a discrete graph that spans the free space by incremental sampling, providing thus a safe path to be followed by the robot. 
Variants of sampling-based algorithms have been also proposed in order to improve their attributes; RRT-connect \cite{kuffner2000rrt} computes two trees in the free space (from the initial and goal configuration, respectively), speeding up the convergence to the goal, and 
asymptotically optimal algorithms, such as RRT*, PRM*, provide a path whose length becomes shorter (more optimal) as the number of samples increases \cite{karaman2010incremental,karaman2011sampling,salzman2016asymptotically}. 
The initial versions of the aforementioned algorithms are geometrical, without accommodating the dynamics (differential constraints) of the system. To that end, {tree-based algorithms such as RRTs and ESTs} have been extended to kinodynamic planning \cite{hsu1997path,lavalle2001randomized,csucan2009kinodynamic,vidal2019online}. Kinodynamic planning algorithms simulate forward the dynamics of the system by randomly sampling control inputs, in order to find a feasible path in the state space. Moreover, similarly to PRM, \cite{tedrake2010lqr} and \cite{reist2016feedback} introduced the framework of LQR-trees, i.e., trees of trajectories that probabilistically cover the free space. By linearizing the system dynamics and using optimal control techniques, every point of the free space is assigned a funnel corresponding to its region of attraction with respect to the goal configuration. Similar ideas are used in \cite{wu2020r3t}, where dynamics linearization and reachability sets are employed to develop an optimal kinodynamic algorithm. Sampling-based algorithms have been also integrated with receding horizon optimization techniques \cite{tanner2010randomized} whereas \cite{yershov2016asymptotically} develops a Hamilton-Jacobi-Bellman approach. {In this work we leverage geometric sampling-based motion planning-techniques and feedback control; we integrate the two, efficiently combining and exploiting their benefits and avoiding thus high computation times and undesired local minima configurations.}

Another important disadvantage of the {majority of the} related works in motion planning is their strong dependence on the considered model of the system dynamics; {the respective algorithms use partial or full information on the underlying dynamic models.} Optimization-based algorithms usually employ dynamics linearization or simulate forward the dynamical model to obtain an optimal control input. The latter is similar to kinodynamic sampling-based motion planning algorithms, which simulate forward the model using random inputs to obtain feasible samples in the free space. The accurate identification of the system dynamics of real robots is a tedious procedure, due to the high uncertainty in the several components (dynamic parameters, friction terms) and unknown exogenous disturbances. Hence, the considered dynamic models used in the aforementioned algorithms do not match sufficiently enough the dynamics of the actual system. As a result, the actual trajectories of the robotic system might deviate from the predicted/planned ones, jeopardizing thus safety and degrading performance. Similar to LQR-trees, \cite{majumdar2017funnel} proposes an algorithm that builds trees of
funnels based on the (known) bounds of model disturbances, restricted however
to polynomial robot dynamics. Planning under uncertainty has been also considered in a stochastic framework, where bounds on the collision probability are
derived \cite{du2011probabilistic,pairet2018uncertainty}. Stochastic uncertainties
are also taken into account via belief trees, where the disturbances and the states
are modeled as Gaussian distributions \cite{bry2011rapidly,agha2014firm}. These approaches, however, usually deal with linearized dynamics, and/or propagate the uncertainties on the planning horizon, constraining thus
excessively the free space. {The framework developed in this paper does not use any linearization or uncertainty propagation techniques. In fact, neither the planning nor the feedback-control module use any information on the underlying system dynamics or their bounds, providing thus robustness to model uncertainties and unknown external disturbances, and applicability to a large variety of different systems and scenarios. More specifically, the proposed framework exhibits the following important characteristics:}
\begin{enumerate}
	\item The (unknown) dynamics of the system are not simulated forward in time and are hence decoupled from the motion planner. This results in the latter being purely geometrical, depending on the geometry of the configuration space and user-defined funnel bounds that are set a priori and define the extended free space.
	\item Even though $k$th-order dynamics are considered, the motion planner searches for a path only in the configuration space, since the dynamics are appropriately compensated by the designed feedback control protocol.
	\item We do not resort to linearization of the dynamics and computation of basins
	of attraction around the output trajectories, since the designed feedback
	control protocol applies directly to the nonlinear model.
\end{enumerate}

{Note that, since the sampling-based motion planner involved in our framework is purely geometric, it is expected to yield lower complexity than standard kinodynamic planning algorithms. Such algorithms sample points {in a space of larger dimension, including random states and control inputs}, and simulate forward the underlying dynamics; hence they usually require more computational resources than geometric planners. It should be noted that similar ideas were pursued in  \cite{le2012sequential,luders2010bounds}, without however considering the complex unknown systems adopted in this work.}
We validate the proposed methodology an Unmanned Aerial Vehicle (UAV) and a $6$DOF UR5 manipulator in V-REP environment \cite{Vrep}, as well as a $6$DOF Hebi-Robotics manipulator. 
This paper is an extension of our recent work \cite{verginis21sampling} along the following directions. Firstly, the bounds that define the funnel where the system evolves in, which also define the extended free space in the developed motion planner, are \textit{a priori user-defined}. This is in contrast to \cite{verginis21sampling}, where the bounds depended on the system dynamics and gain tuning was needed to shrink the funnel and produce less conservative trajectories. 
Finally, we use extensive hardware experiments to validate the efficiency of the proposed framework.

\section{Problem Formulation} \label{sec:pf}

Consider a robotic system characterized by the {configuration} vector $q_1\in \mathbb{T} \times \subset \mathbb{R}^n$, $n\in\mathbb{N}$. Usual robotic structures (e.g., robotic manipulators) might consist of translational and rotational joints, which we define here as $q^\mathfrak{t} = [q^\mathfrak{t}_1,\dots,q^\mathfrak{t}_{n_{tr}}]^\top \in \mathbb{R}^{n_{tr}}$ and $q^\mathfrak{r} = [q^\mathfrak{r}_1,\dots,q^\mathfrak{r}_{n_r}]^\top \in [0,2\pi)^{n_r}$, respectively, with $n_{tr} + n_r = n$, and hence $\mathbb{T} \coloneqq \mathcal{W}_{tr} \times [0,2\pi)^{n_r}$, where $\mathcal{W}_{tr}$ is a closed subset of $\mathbb{R}^{n_{tr}}$. Without loss of generality, we assume that $q_1 = [(q^\mathfrak{t})^\top,(q^\mathfrak{r})^\top]^\top$.

We consider $k$th-order systems, with $k \geq 2$, of the form
\begin{subequations} \label{eq:dynamics}
	\begin{align} 
	\dot{q}_i &= f_i(\bar{q}_i,t) + g_i(\bar{q}_i,t)q_{i+1}, \ \ \forall i\in\{1,\dots,k-1\}\\	
	\dot{q}_k &= f_k(\bar{q}_k,t) + g_k(\bar{q}_k,t)u,
	\end{align}
	\end{subequations}
	where $\bar{q}_i \coloneqq [q_1^\top,\dots,q_i^\top]^\top \in \mathbb{T}\times\mathbb{R}^{n(i-1)}$, for all $i\in\{1,\dots,k\}$, and $u \in \mathbb{R}^n$ is the control input of the system. Note that the $k$th-order model \eqref{eq:dynamics} generalizes the simpler $2$nd-order Lagrangian dynamics, which is commonly used in the related literature.
	
	The vector fields $f_i$, $g_i$, which represent various terms in robotic systems (inertia, Coriolis, friction, gravity, centrifugal) 
	are considered to be completely unknown to the designer/planner, for all $i\in\{1,\dots,k\}$.
	The only assumptions we make for the system are mild continuity and controllability conditions, as follows:
	\begin{assumption} \label{ass:dynamics}
	The maps $\bar{q}_i \mapsto f_i(\bar{q}_i,t): \mathbb{R}^{n(i-1)} \to \mathbb{R}^n$, $\bar{q}_i \mapsto g_i(\bar{q}_i,t): \mathbb{R}^{n(i-1)} \to \mathbb{R}^{n\times n}$  are continuously differentiable for each fixed $t\in\mathbb{R}_{\geq 0}$ and the maps $t \mapsto f_i(\bar{q}_i,t) : \mathbb{R}_{\geq 0} \to \mathbb{R}^n$, $t \mapsto g_i(\bar{q}_i,t) : \mathbb{R}_{\geq 0} \to \mathbb{R}^{n\times n}$ are piecewise continuous and uniformly bounded for each fixed $\bar{q}_i \in \mathbb{R}^{n(i-1)}$, for all $i\in\{1,\dots,k\}$, by \textit{unknown} bounds. 
	\end{assumption}
	\begin{assumption} \label{ass:g pd}
		It holds that 
		\begin{equation*}
		\lambda_{\min}\bigg(g_i(\bar{q}_i,t) + g_i(\bar{q}_i,t)^\top\bigg) \geq \underline{\lambda} > 0,
		\end{equation*}
		for a positive constant $\underline{\lambda}$, for all $\bar{q}_i \in \mathbb{R}^{n(i-1)}$, $t \geq 0$, $i\in\{1,\dots,k\}$, where $\lambda_{\min}(\cdot)$ is the minimum eigenvalue of a matrix.
	\end{assumption}
	{Assumption \ref{ass:dynamics} intuitively states that the terms $f_i(\cdot)$, $g_i(\cdot)$ are sufficiently smooth in the state $\bar{q}_i$ and bounded in time $t$. 
	The smoothness in $\bar{q}_i$ is satisfied by standard terms that appear in the dynamics of robotic systems (inertia, Coriolis, gravity); friction terms might pose an exception, since they are usually modeled by \textit{discontinuous} functions of the state \cite{de1995new}. Although smooth friction approximations can be employed \cite{makkar2005new}, the proposed control design can be adapted to account for discontinuous dynamics (as, e.g., in \cite{verginis2020asymptotic}), we consider smooth terms for ease of exposition. Moreover, the incorporation of time dependence in $f_i(\cdot)$, $g_i(\cdot)$ reflects a time-varying and bounded external disturbance (e.g., wind or adversarial perturbations). } 
	
	Assumption \ref{ass:g pd} is a sufficiently controllability condition for \eqref{eq:dynamics}; intuitively, it states that the input matrices $g_i$ do not change the direction imposed to the system by $q_{i+1}$ when the latter are viewed as inputs (with $q_{k+1} = u$). Note that standard holonomic Lagrangian systems satisfy this condition. {Examples include robotic manipulators, omnidirectional mobile robots, and fully actuated aerial vehicles. Systems not covered by \eqref{eq:dynamics} consist of underactuated or non-holonomic robots, such as unicycles, underactuated aerial or underwater vehicles. Each of these systems requires special attention and cannot be framed into the general framework presented in this work. Funnel-control works for such systems can be found in \cite{zambelli2014posture,verginis2015decentralized,bechlioulis2016trajectory,berger2020tracking}. }
		

We consider that the robot operates in a workspace $\mathcal{W} \subset \mathbb{R}^3$ filled with obstacles occupying a closed set $\mathcal{O} \subset \mathbb{R}^3$. We denote the set of points that consist the volume of the robot at configuration $q_1$ as $\mathcal{A}(q_1) \subset \mathbb{R}^3$. The collision-free space is defined as the open set $\mathcal{A}_\text{free} \coloneqq \{ q_1\in\mathbb{T} : \mathcal{A}(q_1) \cap \mathcal{O} = \emptyset \}$. Our goal is to achieve safe navigation of the robot to a predefined goal region $Q_g \subset \mathcal{A}_\text{free}$ from an initial configuration $q_1(0) \in \mathcal{A}_\text{free}$ via a path ${\boldsymbol{q}_\text{p}:[0,\sigma] \to \mathcal{A}_\text{free}}$ satisfying $\boldsymbol{q}_\text{p}(0) = q_1(0)$ and $\boldsymbol{q}_\text{p}(\sigma) \in Q_g$, for some positive $\sigma$.

The problem we consider is the following:

\begin{problem} \label{prob:1}	
	Given $q(0))\in \mathcal{A}_{\text{free}}$ and $Q_g \subset \mathcal{A}_\text{free}$, respectively, design a control trajectory $u:[0,t_f] \to \mathbb{R}^{n}$, for some {finite} $t_f > 0$, such that the solution $q^\ast(t)$ of \eqref{eq:dynamics} satisfies $q^\ast(t) \in \mathcal{A}_\text{free}$, for all $t\in[0,t_f]$, and $q^\ast(t_f) \in Q_g$.
\end{problem}

The feasibility of Problem \ref{prob:1} is established in the following assumption.

\begin{assumption} \label{ass:path ass}
	There exists a (at least twice differentiable) path $\boldsymbol{q}_\textup{p}:[0,\sigma]\to \mathcal{A}_\textup{free}$ such that $\boldsymbol{q}_\textup{p}(0) = q(0)$ and $\boldsymbol{q}_\textup{p}(\sigma) \in Q_g$.
\end{assumption}

\section{Main Results}

We present here the proposed solution for Problem \ref{prob:1}. Our methodology follows a two-layer approach, consisting of a robust trajectory-tracking control design and a higher-level sampling-based motion planner.
Firstly, we design an adaptive control protocol that compensates for the uncertain dynamical parameters of the {robot} and forces the system to evolve in a funnel around a desired trajectory, whose size can be a priori chosen by the user/designer, and is completely independent from the system dynamics \eqref{eq:dynamics}. We stress that this constitutes the main difference from our previous work \cite{verginis2021adaptive}, where the derived funnel depends on the bounds of the various dynamic terms and the external disturbances. 
Secondly, we develop a geometric sampling-based motion planner that uses this funnel to find a collision free trajectory from the initial to the goal configuration. Intuitively, the robust control design {helps} the motion planner procedure, which does not have to take into account the complete dynamics \eqref{eq:dynamics}.
{Section \ref{sec:ppc control} gives some preliminary background on funnel control and provides the control design, while Section \ref{sec:motion planner} provides the motion planner. }


\subsection{Control Design} \label{sec:ppc control}

In order to tackle the unknown dynamics of \eqref{eq:dynamics} we use the methodology of 
funnel control \cite{bechlioulis2008robust,ilchmann2007tracking}. 
Funnel control aims at achieving containment of a scalar tracking error $e(t)$ in a user-prespecified time-varying set, defined by certain functions of time, as  
\begin{equation}
	-\rho(t) < e(t) < \rho(t),~\forall t\geq0, \label{eq:ppc}
\end{equation}
where $\rho(t)$ denotes a smooth and bounded function of time that satisfies $\rho(t_0) >|e(t_0)|$ and $\rho(t) > 0$, for all $t\geq t_0$, called funnel function (or performance function in \cite{bechlioulis2008robust}). Fig.~\ref{fig:funnel control} illustrates the aforementioned statements. 
Since the funnel set is user defined a priori, it can be set to converge to an arbitrarily small residual set with speed no less than a prespecified value, e.g., by using the funnel function $\rho(t) \coloneqq (\rho_0-\rho_\infty)e^{-\lambda t}+\rho_\infty$. The parameter $\rho_\infty \coloneqq \lim_{t\rightarrow\infty} \rho(t) > 0$ represents the maximum allowable value of the steady state error and can be set to a value reflecting the resolution of the measurement device, so that the error $e(t)$ practically converges to zero. 
Moreover, the constant $\lambda$ determines the decreasing rate of $\rho(t)$ and thus is used to set a lower bound on the convergence rate of $e(t)$. 
Therefore, the appropriate selection of the function $\rho(t)$ imposes certain transient and steady state performance characteristics on the tracking error $e(t)$. 
{Intuitively, larger $\lambda$ and small $\rho_\infty$ improve the performance of the system, yielding fast convergence close to zero. Although these constants can be arbitratrily set by a user, their values affect significantly the stress imposed on the system, and hence they should be chosen according to the system's capabilities.}
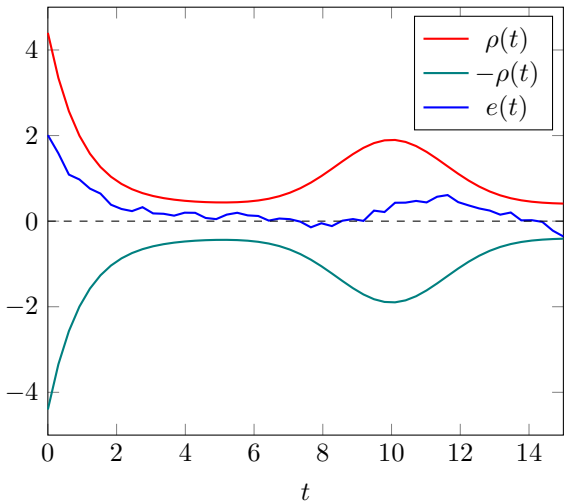
\begin{figure}
	\begin{tikzpicture}
		\begin{axis}[xmax=15,xmin=0,ymin=-5,ymax= 5,xlabel={$t$},
			samples=50]
			\addplot[domain=0:15, red, thick] ({\x}, {4*exp(-\x)+0.4 + 1.5*exp(-(\x-10)^2/5) }) ;
			\addplot[domain=0:15, teal, thick] ({\x}, {-4*exp(-\x)-0.4 - 1.5*exp(-(\x-10)^2/5) }) ;
			\addplot[domain=0:15, blue, thick]  ({\x}, {2*exp(-\x) - 0.015*\x^2 + 0.1*x +  .1*rand + 1.5*exp(-(\x-15)^2/5) + 1.25*exp(-(\x-11.5)^2/5)}      );
			\addplot[domain=0:15, black, dashed] ({\x},{0});
			\legend{$\rho(t)$,$-\rho(t)$,$e(t)$}
		\end{axis}
	\end{tikzpicture}
	\caption{Illustration of funnel control, where the error $e(t)$ is confined in the prescribed funnel defined by the function $\rho(t)$. \label{fig:funnel control}}
\end{figure}

The key point in funnel control is a transformation of the tracking error $e(t)$ that modulates it with respect to the corresponding funnel specifications, encapsulated in the function $\rho(t)$. This is achieved by converting the constrained problem to an unconstrained one via a transformation of the form $\mathsf{T}\left(\frac{e(t)}{\rho(t)}\right)$, where $\mathsf{T} : (-1, 1)\rightarrow(-\infty,\infty)$ is a strictly increasing, odd and bijective mapping. 
Then the funnel specifications are met by simply preserving the boundedness of $\mathsf{T}\left(\frac{e(t)}{\rho(t)}\right)$. 
Most funnel control schemes do not employ any information on the system dynamics, using a high-gain approach. That is, the control action approaches infinity as the state approaches the funnel boundary, ``pushing" thus the system to remain inside the funnel. In this work we extend the funnel control design to apply for the system \eqref{eq:dynamics} and the manifold $\mathbb{T}$, and we show how such a design can be used in the motion planning of the uncertain robotic system \eqref{eq:dynamics}. 

Let $q_\text{d}\coloneqq [(q^\mathfrak{t}_{\text{d}})^\top, (q^\mathfrak{r}_\text{d})^\top]^\top \coloneqq [q^\mathfrak{t}_{\text{d}_1},\dots,q^\mathfrak{t}_{\text{d}_{n_{tr}}},q^\mathfrak{r}_{\text{d}_1},\dots,q^\mathfrak{r}_{\text{d}_{n_r}}]^\top$ $:[t_0,t_0+t_f] \to \mathbb{T}$ be a smooth (at least $k$-times continuously differentiable) reference trajectory, with $q^\mathfrak{t}_\text{d}\in\mathbb{R}^{n_{tr}}$ and $q^\mathfrak{r}_\text{d}\in[0,2\pi)^{n_r}$ {being} its translational and rotational parts, respectively. {Such a trajectory will be derived by smoothening and adding time constraints to the output path of the sampling-based motion-planning algorithm that will be developed in the next section. Note that the smoothness assumption on $q_\textup{d}$ is not restrictive, since the smoothening of geometric paths eases the resulting robot motion and is hence common practice in real applications. 
Nevertheless, we stress that the proposed control algorithm can be applied separately on the raw path segments produced by the local collision-checking planner of the motion-planning algorithm, without requiring any smoothening. This might, however, induce discontinuities on the control algorithm, which might be problematic for robot actuators.}

 We wish to design the control input $u$ of \eqref{eq:dynamics} such that $q(t)$ converges to $q_\text{d}(t)$, despite the unknown terms $f_i$, $g_i$.
We start by defining the appropriate error metric between $q_1$ and $q_\text{d}$, which represents their distance. Regarding the translational part, we define the standard Euclidean error $e^\mathfrak{t} \coloneqq q^\mathfrak{t} - q^\mathfrak{t}_\text{d}$. For the rotation part, however, the same error $e^\mathfrak{r} \coloneqq q^\mathfrak{r} - q^\mathfrak{r}_{\text{d}}$ does not represent the minimum distance metric, since $q^\mathfrak{r}$ evolves on the $n_r$-dimensional sphere, and its use might cause conservative or infeasible results in the planning layer. Hence, unlike standard control schemes, which drive the Euclidean difference $e^\mathfrak{r}$ to zero (e.g., \cite{tomei1999robust,slotine1987adaptive}), we use the chordal metric $$d_C(x) \coloneqq 1 - \cos(x) \in [0,2], \forall x\in[0,2\pi),$$ extended for vector arguments $x = [x_1,\dots,x_\mathsf{n}] \in [0,2\pi)^\mathsf{n}$ to
\begin{equation} \label{eq:chordal}
	\bar{d}_C(x) \coloneqq \sum_{\ell\in\{1,\dots,\mathsf{n}\}}d_C(x_j).
\end{equation}  

Note that the chordal metric induces a limitation with respect to tracking on the unit sphere. Consider $$\eta^{\mathfrak{r}}_\ell \coloneqq d_C(e^{\mathfrak{r}}_\ell) = 1 - \cos(e^{\mathfrak{r}}_\ell),$$ where $e^{\mathfrak{r}}_\ell \coloneqq q^\mathfrak{r}_\ell - q^\mathfrak{r}_{\text{d}_\ell}$ is the $\ell$th element of $e^\mathfrak{r}$, $\ell\in\{1,\dots,n_r\}$. Differentiation yields $$\dot{d}_C(e^{\mathfrak{r}}_\ell) =\sin(e^\mathfrak{r}_\ell)\dot{e}^\mathfrak{r}_\ell,$$  $\forall \ell\in\{1,\dots,n_r\}$,
which is zero when $e^{\mathfrak{r}}_\ell = 0$ or $e^\mathfrak{r}_\ell = \pi$. The second case is an undesired equilibrium, which  implies that the point $e^\mathfrak{r}_\ell = 0$ cannot be stabilized from \textit{all} initial conditions using a continuous controller. This is an inherent property of dynamics on the unit sphere due to topological obstructions \cite{bhat2000topological}. {In the following, we devise a control scheme that, except for guaranteeing that $e^\mathfrak{r}_\ell(t)$ evolves in a predefined funnel, guarantees that $e^\mathfrak{r}_\ell(t) \neq \pi$, $\forall t\in (t_0,t_f]$, provided that $e^\mathfrak{r}_\ell(t_0) \neq \pi$, $\forall \ell\in\{1,\dots,n_r\}$.
	
The funnel is defined by the functions $\rho^\mathfrak{t}_j: [t_0,t_0 + t_f] \to [\underline{\rho}^\mathfrak{t}_j, \bar{\rho}^\mathfrak{t}_j]$, $\rho^\mathfrak{r}_\ell : [t_0,t_0+T] \to [\underline{\rho}^\mathfrak{r}_\ell, \bar{\rho}^\mathfrak{r}_\ell]$ {with initial and final values $\underline{\rho}^\mathfrak{t}_j$, $\underline{\rho}^\mathfrak{r}_\ell$, and $\bar{\rho}^\mathfrak{t}_j$, $\bar{\rho}^\mathfrak{r}_\ell$, respectively, i.e.,}
\begin{subequations} \label{eq:funnel constraints}
\begin{align}
	&\rho^\mathfrak{t}_j(t_0) = \bar{\rho}^\mathfrak{t}_j,  \rho^\mathfrak{r}_\ell(t_0) = \bar{\rho}^\mathfrak{r}_\ell \\
	&\rho^\mathfrak{t}_j(t_0+t_f) = \underline{\rho}^\mathfrak{t}_j,  \rho^\mathfrak{r}_\ell(t_0+t_f) = \underline{\rho}^\mathfrak{r}_\ell \\
	&0 < \underline{\rho}^\mathfrak{t}_j \leq \bar{\rho}^\mathfrak{t}_j,  0 < \underline{\rho}^\mathfrak{r}_\ell \leq \bar{\rho}^\mathfrak{r}_\ell < 2 
\end{align}
{and being consistent with the errors initially, i.e., }
\begin{align} \label{eq:funnel constraints initially}
	|e^\mathfrak{t}_j(t_0)| < \bar{\rho}^\mathfrak{t}_j, \ \  \eta^\mathfrak{r}_\ell(t_0) < \bar{\rho}^\mathfrak{r}_\ell 
\end{align}
\end{subequations}
$\forall j\in\{1,\dots,n_{tr}\}$, $\ell\in\{1,\dots,n_r\}$. Our aim is to design a control protocol such that 
\begin{subequations} \label{eq:ppc specs}	
\begin{align}
|e^\mathfrak{t}_j(t)| & < \rho^\mathfrak{t}_j(t), \ \ \forall j\in\{1\dots,n_{tr}\}, \label{eq:ppc_spec_1}\\ 	
\eta^\mathfrak{r}_\ell(t)  & < \rho^\mathfrak{r}_\ell(t), \ \ \forall \ell \in\{1\dots,n_r\}, \label{eq:ppc_spec_2} 
\end{align}
\end{subequations}
$\forall t\in[t_0,t_0 + t_f]$. Note that, since $\bar{\rho}_{\mathfrak{r}_\ell} < 2$, guaranteeing \eqref{eq:ppc_spec_2} ensures that $\eta^\mathfrak{r}_\ell(t) < 2$, i.e., $e^\mathfrak{r}_\ell(t) \neq \pi$, $\forall t\in[t_0,t_0 + t_f], \ell\in\{1,\dots,n_r\}$ and avoidance of the respective singularity. The funnel functions can be defined a priori by a user, specifying the performance of the system in terms of overshoot and steady-state value of the errors $e^\mathfrak{t}_j$, $e^\mathfrak{r}_\ell$. {For instance, for the exponentially decaying $\rho^\mathfrak{t}_j(t) = (\bar{\rho}^\mathfrak{t}_j - \underline{\rho}^\mathfrak{t}_j)\exp(-\lambda_jt) + \underline{\rho}^\mathfrak{t}_j$, $\forall t\in[t_0,t_0+t_f]$, a user can choose the constants $\underline{\rho}^\mathfrak{t}_j$, $\bar{\rho}^\mathfrak{t}_j$, $\lambda_j$ dictating the maximum error steady-state value, overshoot, and speed of convergence.} The only hard condition is property \eqref{eq:funnel constraints initially} above, stating that the errors needs to respect the funnel constraints initially. Note also that the funnel functions do not depend on the robot dynamics, and can converge to values $\underline{\rho}^\mathfrak{t}_j$, $\underline{\rho}^\mathfrak{r}_\ell$ arbitrarily close to zero at $t_0 + t_f$, achieving thus practical stability. {We provide more details on the choice of the funnels after the control-design algorithm, which is described next.}

Let us define first the normalized errors as
\begin{subequations} \label{eq:xi}
\begin{align}
&\xi^\mathfrak{t}_j \coloneqq \frac{e^\mathfrak{t}_j}{\rho^\mathfrak{t}_j}, \ \ \forall j\in\{1\dots,n_{tr}\}, \\
&\xi^\mathfrak{r}_\ell \coloneqq \frac{\eta^\mathfrak{r}_\ell}{\rho^\mathfrak{r}_\ell}, \ \  \forall \ell\in\{1\dots,n_r\}.
\end{align} 
\end{subequations}
{Note that, in order for the errors $e^\mathfrak{t}_j$ and $\eta^\mathfrak{r}_\ell$ to satisfy the funnel constraints, the control design must guarantee that $\xi_j^\mathfrak{t} \in (1,1)$ and $\xi^\mathfrak{r}_\ell \in [0,1)$. In order to do that, we define 	
the transformed errors and signals} 
\begin{subequations} \label{eq:epsilon+r}
\begin{align}
&\varepsilon^\mathfrak{t}_j \coloneqq \ln\left(\frac{1 + \xi^\mathfrak{t}_j}{1 - \xi^\mathfrak{t}_j}\right), \ \ \forall j\in\{1\dots,n_{tr}\},  \\
&\varepsilon^\mathfrak{r}_\ell \coloneqq \ln\left(\frac{1}{1 - \xi^\mathfrak{r}_\ell} \right), \ \  \forall \ell\in\{1\dots,n_r\}, \\
&r^\mathfrak{t}_j \coloneqq \frac{\partial \varepsilon^\mathfrak{t}_j}{\partial \xi^\mathfrak{t}_j} = \frac{2}{1-(\xi^\mathfrak{t}_j)^2}, \ \ \forall j\in\{1\dots,n_{tr}\},  \\
&r^\mathfrak{r}_\ell \coloneqq \frac{\partial \varepsilon^\mathfrak{r}_\ell}{\partial \xi^\mathfrak{r}_\ell} = \frac{1}{1 - \xi^\mathfrak{r}_\ell}, \ \  \forall \ell\in\{1\dots,n_r\}.
\end{align}
\end{subequations}
{Note that $\varepsilon^\mathfrak{t}_j$, $r^\mathfrak{t}_j$, and $\varepsilon^\mathfrak{r}_\ell$, $r^\mathfrak{r}_\ell$ diverge to infinity as $\xi_j^\mathfrak{t}$ and $\xi^\mathfrak{r}_\ell$ approach $1$, respectively. The control design exploits this property; it aims to keep these signals bounded in order to achieve $\xi_j^\mathfrak{t}(t) \in (1,1)$ and $\xi^\mathfrak{r}_\ell(t) \in [0,1)$. We now proceed with a back-stepping methodology \cite{krstic1995nonlinear}. Since $q_2$ is part of the system state and cannot be designed, we set a desired reference signal that we want $q_2$ to track. In particular, we define the reference signal for $q_2$ as $\alpha_1 \coloneqq [(\alpha^{\mathfrak{t}})^\top,(\alpha^\mathfrak{r})^\top]^\top$, with}
\begin{subequations} \label{eq:alpha}
\begin{align}
\alpha^{\mathfrak{t}} \coloneqq - K^{\mathfrak{t}} \widetilde{r}^{\mathfrak{t}} (\widetilde{\rho}^{\mathfrak{t}})^{-1} \varepsilon^{\mathfrak{t}} \\
\alpha^\mathfrak{r} \coloneqq - K^\mathfrak{r} \widetilde{s}^\mathfrak{r} (\widetilde{\rho}^{\mathfrak{r}})^{-1} r^\mathfrak{r},
\end{align}
\end{subequations}
where $K^{\mathfrak{t}} \coloneqq \text{diag}\{[k^\mathfrak{t}_j]_{j\in\{1,\dots,n_{tr}\}}\} \in \mathbb{R}^{n_{tr} \times n_{tr}}$, $K_\mathfrak{r} \coloneqq \text{diag}\{[k^\mathfrak{r}_\ell]_{\ell \in\{1,\dots,n_{r}\}}\} \in\mathbb{R}^{n_r \times n_r}$ are diagonal positive definite gain matrices, $\widetilde{r}^{\mathfrak{t}} \coloneqq \text{diag}\{ [r^\mathfrak{t}_j]_{j\in\{1,\dots,n_{tr}\}} \}$, $r^\mathfrak{r} \coloneqq  [r^\mathfrak{r}_1, \dots, r^\mathfrak{r}_{n_r}]$, $\widetilde{\rho}^{\mathfrak{t}} \coloneqq \text{diag}\{ [\rho^\mathfrak{t}_j]_{j\in\{1,\dots,n_{tr}\}} \}$, $\widetilde{\rho}^\mathfrak{r} \coloneqq \text{diag}\{ [\rho^\mathfrak{r}_\ell]_{\ell\in\{1,\dots,n_r\}} \}$, $\varepsilon^\mathfrak{t} \coloneqq [\varepsilon^\mathfrak{t}_1,\dots,\varepsilon^\mathfrak{t}_{n_{tr}}]^\top$, and $\widetilde{s}^\mathfrak{r} \coloneqq \text{diag}\{ [\sin(e^\mathfrak{r}_\ell)]_{\ell\in\{1,\dots,n_r\}} \}$.

The rest of the algorithm proceeds recursively: 
for $i\in \{2,\dots,k\}$, we define the error 
\begin{equation} \label{eq:e_i}
e_i \coloneqq \begin{bmatrix}
e_{i_1} \\ 
\vdots \\
e_{i_n}
\end{bmatrix} \coloneqq q_i - \alpha_{i-1} \in \mathbb{R}^n,
\end{equation}
where $\alpha_{i-1}$ will be given subsequently in \eqref{eq:alpha_i}. 
We design funnel functions $\rho_{i_m}: [t_0,t_0+t_f] \to [\underline{\rho}_{i_m},\bar{\rho}_{i_m}]$, $\underline{\rho}_{i_m} \leq \bar{\rho}_{i_m} $, such that $\rho_{i_m}(t_0)=\bar{\rho}_{i_m} > |e_{i_m}(t_0)|$\footnote{Note that $e_{i_m}(t_0)$ can be measured at the time instant $t_0$ and the functions $\rho_{i_m}$ can be designed accordingly.}, $\forall m\in\{1,\dots,n\}$, and define 
\begin{subequations} \label{eq:xi_i,epsilon_i,r_i}
\begin{align}
&\xi_i \coloneqq \begin{bmatrix}
\xi_{i_1} \\ 
\vdots \\
\xi_{i_n}
\end{bmatrix} \coloneqq \rho_i^{-1} e_i, \\
&\varepsilon_i \coloneqq\begin{bmatrix}
\varepsilon_{i_1} \\ 
\vdots \\
\varepsilon_{i_n}
\end{bmatrix} \coloneqq 
\begin{bmatrix}
\ln\left( \frac{1 + \xi_{i_1}}{1 - \xi_{i_1}} \right) \\
\vdots  \\
\ln\left( \frac{1 + \xi_{i_n}}{1 - \xi_{i_n}} \right)
\end{bmatrix} \\
& r_i \coloneqq \text{diag}\left\{  \left[ \frac{\partial \varepsilon_{i_m}}{\partial \xi_{i_m}} \right]_{m\in\{1,\dots,n\}} \right\}, 
\end{align}
\end{subequations}
where $\rho_i \coloneqq \textup{diag}\{[\rho_{i_m}]_{i\in\{1,\dots,n\}}\} \in \mathbb{R}^{n\times n}$. 
Finally, we design the intermediate reference signals as 
\begin{equation}\label{eq:alpha_i}
	\alpha_{i} \coloneqq -K_i \rho_i^{-1} r_i \varepsilon_i, \forall i\in\{2,\dots,k-1\},
\end{equation}
and the control law
\begin{equation}\label{eq:control law}
	u = -K_k \rho_k^{-1} r_k \varepsilon_k,
\end{equation}
{where $K_i \in \mathbb{R}^{n \times n}$, $i\in{2,\dots,k}$, are positive definite gain matrices. }

{The control algorithm of this subsection is summarized in Algorithm \ref{alg:ppc pseudocode}, taking as inputs the system order $k$, the desired trajectory $q_\textup{d}(t)$, the funnel functions $\rho^\mathfrak{t}_j$, $\rho^\mathfrak{r}_\ell$, and the control gain matrices $K^\mathfrak{t}$, $K^\mathfrak{r}$, $K_i$, for all $j\in\{1,\dots,n_{tr}\}$, $\ell \in \{1,\dots,n_r\}$, $i\in\{2,\dots,k\}$, and outputting the control input $u$ at each time instant $t\in[t_0,t_0+t_f]$.  We have added the $\mathsf{GetFeedback}()$ function (line 4), which returns the sensing information for the state $q_1,\dots,q_k$. Note also that the funnels defined by the functions $\rho_i$ are defined at $t=t_0$ (line 22), since they require the value of $e_i(t_0)$ in order to guarantee $\rho_{i_m}(t_0) > |e_{i_m}(t_0)|$, for $m\in\{1,\dots,n\}$; the type and structure of the functions $\rho_i$, however, can be determined a priori by a user. A standard choice is exponentially decaying $\rho_{i_m}(t) = (\bar{\rho}_{i_m} - \underline{\rho}_{i_m})\exp(-\lambda_{i_m}t) + \underline{\rho}_{i_m}$, or constant ones $\rho_{i_m}(t) = \bar{\rho}_{i_m}$, with the rule $\rho_{i_m}(t_0) = \bar{\rho}_{i_m} = |e_{i_m}(t_0)| + \alpha$, for some $\alpha > 0$. Finally, although tight funnels $\rho^\mathfrak{t}_j$, $\rho^\mathfrak{r}_\ell$ might be desired to achieve close proximity of $q_1(t)$ to $q_\textup{d}(t)$, the funnels defined by $\rho_i$ are only required to be bounded; convergence to very small values (e.g., by choosing very small values $\underline{\rho}_{i_m}$ for exponentially-decaying funnels) does not have a direct physical interpretation in the system's configuration space and can overstress the system causing unnecessarily large control inputs.}

\begin{algorithm}
	\caption{$\mathsf{FunnelControl}$}\label{alg:ppc pseudocode}
	\begin{algorithmic}[1]
		\Require{$k$, $t_f$, $q_\textup{d}$, $\rho^\mathfrak{t}_j$, $\rho^\mathfrak{r}_\ell$, $K^\mathfrak{t}$, $K^\mathfrak{r}$, $K_i$, $j\in\{1,\dots,n_{tr}\}$, $\ell \in \{1,\dots,n_r\}$ }
		\Ensure{$u(t)$}
		\Procedure{Funnel Control Design}{}	
		\For {$t \in [t_0,t_0+t_f] $} 
		\For {$j\in\{1\dots,n_{tr}\}$, $\ell\in\{1\dots,n_r\}$} 
		\State $(q_1,\dots,q_k) \gets \mathsf{GetFeedback}()$;
		\State $e^\mathfrak{t} \gets q^\mathfrak{t}(t) - q^\mathfrak{t}_\textup{d}(t)$;
		\State $\xi^\mathfrak{t}_j \gets \frac{e^\mathfrak{t}_j}{\rho^\mathfrak{t}_j(t)}$; \hspace{12.5mm} $\xi^\mathfrak{r}_\ell \gets \frac{\eta^\mathfrak{r}_\ell}{\rho^\mathfrak{r}_\ell(t)}$;
		\State $\varepsilon^\mathfrak{t}_j \gets \ln\left(\frac{1 + \xi^\mathfrak{t}_j}{1 - \xi^\mathfrak{t}_j}\right)$; \hspace{4.5mm} $\varepsilon^\mathfrak{r}_\ell \gets \ln\left(\frac{1}{1 - \xi^\mathfrak{r}_\ell} \right)$;				
		\State $r^\mathfrak{t}_j \gets  \frac{2}{1-(\xi^\mathfrak{t}_j)^2}$; \hspace{9mm} $r^\mathfrak{r}_\ell \gets \frac{1}{1 - \xi^\mathfrak{r}_\ell}$;				
		\EndFor
		\State $\widetilde{\rho}^{\mathfrak{t}} \gets \text{diag}\{ [\rho^\mathfrak{t}_j(t)]_{j\in\{1,\dots,n_{tr}\}} \}$;
		\State $\widetilde{\rho}^\mathfrak{r} \gets \text{diag}\{ [\rho^\mathfrak{r}_\ell(t)]_{\ell\in\{1,\dots,n_r\}} \}$;			
		\State $\varepsilon^\mathfrak{t} \gets [\varepsilon^\mathfrak{t}_1,\dots,\varepsilon^\mathfrak{t}_{n_{tr}}]^\top$;
		\State $\widetilde{r}^{\mathfrak{t}} \coloneqq \text{diag}\{ [r^\mathfrak{t}_j]_{j\in\{1,\dots,n_{tr}\}} \}$;
		\State $r^\mathfrak{r} \coloneqq  [r^\mathfrak{r}_1, \dots, r^\mathfrak{r}_{n_r}]$; 	
		\State $\widetilde{s}^\mathfrak{r} \coloneqq \text{diag}\{ [\sin(e^\mathfrak{r}_\ell)]_{\ell\in\{1,\dots,n_r\}} \}$;			
		\State $\alpha^{\mathfrak{t}} \gets - K^{\mathfrak{t}} \widetilde{r}^{\mathfrak{t}} (\widetilde{\rho}^{\mathfrak{t}})^{-1} \varepsilon^{\mathfrak{t}}$;
		\State $\alpha^\mathfrak{r} \gets - K^\mathfrak{r} \widetilde{s}^\mathfrak{r} (\widetilde{\rho}^{\mathfrak{r}})^{-1} r^\mathfrak{r}$;
		\State $a_1 \gets [(\alpha^\mathfrak{t})^\top, (\alpha^\mathfrak{r})^\top]^\top$;
		\For {$i \in \{2,\dots,k\}$} 
		\State $e_i \gets q_i - a_{i-1}$;
		\For {$m\in\{1,\dots,n\}$}					 
		\If {$t=t_0$}				
		\State Define $\rho_{i_m}$ such that $\rho_{i_m}(t_0) >  |e_{i_m}|$;
		\EndIf 
		\State $\xi_{i_m} \gets \frac{e_{i_m}}{\rho_{i_m}}$; 
		\State $\varepsilon_{i_m} \gets \ln\left( \frac{1 + \xi_{i_m}}{1 - \xi_{i_m}} \right)$;
		\State $r_{i_m} \gets \frac{2}{1 - \xi_{i_m}^2}$, 
		\EndFor						
		\State $\rho_i \gets \textup{diag}\{ [\rho_{i_m}]_{m\in\{1,\dots,n\}} \}$;
		\State $r_i \gets \textup{diag}\{ [r_{i_m}]_{m\in\{1,\dots,n\}} \}$;
		\State $\varepsilon_i \gets [\varepsilon_1,\dots,\varepsilon_n]^\top$;
		\State $\alpha_{i} \gets -K_i \rho_i^{-1} r_i \varepsilon_i$;				
		\EndFor
		\State $u \gets \alpha_k$;
		\EndFor
		\EndProcedure
	\end{algorithmic}
\end{algorithm}

\begin{remark}
	The control algorithm \eqref{eq:xi}-\eqref{eq:control law} resembles the function of reciprocal barriers used in optimization. That is, the intermediate reference and control signals \eqref{eq:alpha}, \eqref{eq:alpha_i}, \eqref{eq:control law} approach infinity as the errors $|e^\mathfrak{t}_j|$, $1-e^\mathfrak{r}_\ell$, $|e_{i_m}|$, $i\in\{2,\dots,k-1\}$ approach the respective funnel functions $\rho^\mathfrak{t}_j$, $\rho^\mathfrak{r}_\ell$, $\rho_{i_m}$, $j\in\{1,\dots,n_{tr}\}$, $\ell\in\{1,\dots,n_{r}\}$, $m\in\{1,\dots,n\}$, $i\in\{2,\dots,k\}$. Intuitively, this forces these errors to remain inside their respective funnels, by compensating for the unknown dynamic terms of \eqref{eq:dynamics}, which are assumed to be continuous and hence bounded in these funnels. {In addition, note that the control algorithm \textit{does not use} any information on the state- and time-dependent system dynamics $f_i(\cdot)$, $g_i(\cdot)$, giving rise to two important properties; firstly, it can be easily applied to a large variety of systems with different dynamic parameters; secondly, it is robust against unknown, possibly adversarial, time-varying disturbances. The latter is clearly illustrated in the performed experiments of Section \ref{sec:exps}.}
\end{remark}

{
\begin{remark}[Control gain selection] \label{rem:gains}
The control gain matrices $K^\mathfrak{t}$, $K^\mathfrak{r}$, $K_i$, $i\in\{2,\dots,k\}$ are chosen by the user and can be any positive definite matrices. It should be noted, however, that their choice affects both the quality of evolution of the errors inside  the  funnel envelopes as  well  as the  control  input  characteristics  (e.g.,  decreasing  the  gain values  leads  to  increased  oscillatory  behavior  within, which is improved  when  adopting  higher  values,  enlarging,  however, the  control  effort  both  in  magnitude  and  rate).  Additionally, fine tuning might be needed in real-time scenarios, to retain the required control input signals within the feasible range that can be implemented by the  actuators. 
In fact, by following the proof of correctness of the proposed control algorithm (Appendix \ref{app:ppc proof}), we can derive expressions connecting the control input magnitude with the control gains (an explicit derivation can be found in \cite{verginis2019robust}).
Hence, we can derive a closed-form rule for choosing the control gains for the control input to satisfy certain saturation constraints. Nevertheless, such expressions involve upper bounds of the unknown dynamic terms $f_i(\cdot)$, $g_i(\cdot)$, and hence such bounds must be known for the derivation of a closed-form rule.  In addition, the discrete nature of the microcontroller units of robot actuators prevents the \textit{continuous} application of the control law \eqref{eq:control law}, which might hinder the performance of the overall scheme. Therefore, tuning of the control gains towards optimal performance should be performed off-line or using a simulator.  
\end{remark}
}

The next theorem guarantees the correctness of the proposed protocol.

\begin{theorem} \label{th:control theorem} 
	Let the dynamics \eqref{eq:dynamics} as well as prescribed funnels $\rho^\mathfrak{t}_j$, $\forall j\in\{1,\dots,n_{tr}\}$, $\rho^\mathfrak{r}_\ell$, $\forall \ell\in\{1,\dots,n_r\}$ satisfying the prescribed initial constraints \eqref{eq:funnel constraints}. Then the control protocol \eqref{eq:xi}-\eqref{eq:control law} guarantees that 
	\begin{subequations} \label{eq:theorem bounds}
	\begin{align} 
	& |e^\mathfrak{t}_j(t)| < \rho^\mathfrak{t}_j(t), \forall j\in\{1,\dots,n_{tr}\}, \\
	&\eta^\mathfrak{r}_\ell(t) = 1 - \cos(e^\mathfrak{r}(t)) < \rho^\mathfrak{r}_\ell(t), \forall \ell \in \{1,\dots,n_r\}
	\end{align}
	\end{subequations}
	as well as the boundedness of all closed loop signals, 	$t\in[t_0,t_0+t_f]$. 
\end{theorem}

\begin{proof}
	The proof is given in Appendix \ref{app:ppc proof}.
\end{proof}

\begin{remark} [Funnel properties]  \label{rem:funnel}
Theorem \ref{th:control theorem} establishes a funnel around the 
desired trajectory $q_\text{d}$ where the state $q(t)$ will evolve in. This funnel will be used as clearance in the motion planner of the subsequent section to derive a collision-free path to the goal region. We stress that this funnel can be \textit{a priori chosen by a user}, in contrast to our previous work \cite{verginis21sampling}, where the corresponding funnel depends on the system's dynamic terms that are unknown to the user. The only hard constraint is the one imposed by \eqref{eq:funnel constraints initially} at $t_0$, i.e., $|e^\mathfrak{t}_j(t_0)| < \bar{\rho}^\mathfrak{t}_j$, $\eta^\mathfrak{r}_\ell(t_0) < \bar{\rho}^\mathfrak{r}_\ell$, $\forall j\in\{1,\dots,n_{tr}\}$, $\ell \in \{1,\dots,n_r\}$. Note, however, that 
the collision-free geometric trajectory $q_\text{d}$ of the motion planner will connect the initial condition $q(0)$ to the goal and hence it is reasonable to enforce $q_\textup{d}(0) = q(0)$, which implies that the aforementioned constraint is trivially satisfied. Moreover, the selection of $\rho^\mathfrak{t}_j$, $\rho^\mathfrak{r}_\ell$ can be chosen such that the respective funnels are arbitrarily small implying that the system evolves arbitrarily close to the derived trajectory $q_\text{d}$. It should be noted, nevertheless, that too shrunk or {very fast-converging} funnels might yield excessive control inputs that cannot be realized by the actuators in realistic systems. {Therefore, the funnel characteristics must be always chosen in accordance to the capabilities of the system. Similarly to the control gains (see Remark \ref{rem:gains}), the funnel characteristics can be explicitly connected to the system's control input, requiring, however, upper bounds of the unknown dynamics $f_i(\cdot)$, $g_i(\cdot)$. Hence, tuning can be attempted off-line or using a simulator. As an example, in our results of Section \ref{sec:exps}, where we perform computer simulations and hardware experiments using 6-DOF robotic manipulators, we choose the funnels as follows. For the computer simulations, we choose $\rho^\mathfrak{t}_j(t) = 0.05\exp(-0.01t) + 0.1$, $\rho^\mathfrak{r}_j(t) = 0.005\exp(-0.01t) + 0.005$, implying shrinking funnels from $0.1$ to $0.05$ and from $0.01$ to $0.005$ (rad), respectively, with exponential convergence dictated by $\exp(-0.01t)$, while for the hardware experiments we choose constant funnels with magnitude ranging from $0.2$ to $0.4$. }
\end{remark}

\subsection{Motion Planner} \label{sec:motion planner}

{We introduce now the framework of \textit{KinoDynamic motion planning via Funnel control}, or \textit{KDF} motion-planning framework; The framework uses the control design of Section \ref{sec:ppc control} to augment geometric sampling-based motion-planning algorithms and solve the kinodynamic motion-planning problem.}

{Before presenting the framework, we define the extended-free space, which will be used to integrate the results from the feedback control of the previous subsection. In order to do that, we define first the open polyhedron as 
\begin{align}\label{eq:polyhedron def}
\mathcal{P}(\mathsf{z},\bar{\rho}) \coloneqq \{ \mathsf{y}\in\mathbb{T} : & |\mathsf{y}^\mathfrak{t}_j - \mathsf{z}^\mathfrak{t}_j| < \bar{\rho}^\mathfrak{t}_j,  \notag \\ 
&\hspace{-5mm} 1 - \cos(\mathsf{y}^\mathfrak{r}_\ell - \mathsf{z}^\mathfrak{r}_\ell) < \bar{\rho}^r_\ell, \notag \\
&\hspace{-20mm} \forall j\in\{1,\dots,n_{tr}\},\ell \in\{1,\dots,n_r\} \}
\end{align}
where $\mathsf{y},\mathsf{z}\in\mathbb{T}$ consist of translational and rotational terms (similarly to $q_1$),
and $\bar{\rho}$ $\coloneqq$ $[\bar{\rho}^\mathfrak{t}_1,\dots,\bar{\rho}^\mathfrak{t}_{n_{tr}},\bar{\rho}^\mathfrak{r}_1,\dots,\bar{\rho}^\mathfrak{r}_{n_r} ]^\top$ $\in\mathbb{R}^{n_{tr}+n_r}$ is the vector of maximum funnel values. We define now, similarly to \cite{verginis2021adaptive}, the $\bar{\rho}$-extended free space 
\begin{align} \label{eq:extended free space}
	\bar{\mathcal{A}}_{\text{free}}(\bar{\rho}) \coloneqq \{\mathsf{z}\in\mathbb{T}:\bar{\mathcal{A}}(\mathsf{z},\bar{\rho}) \cap \mathcal{O} = \emptyset \},
\end{align}
where $\bar{\mathcal{A}}(\mathsf{z},\bar{\rho}) \coloneqq \bigcup_{\mathsf{y}\in \mathcal{P}(\mathsf{z},\bar{\rho})} \mathcal{A}(\mathsf{y})$. {Note that, for vectors $\rho_1$ and $\rho_2$ $\in\mathbb{R}^{n_{tr}+n_r}$, with $\rho_1 \succeq \rho_2$, with $\succeq$ denoting element-wise inequality, it holds that $\bar{\mathcal{A}}_{\text{free}}(\rho_1) \subseteq \bar{\mathcal{A}}_{\text{free}}(\rho_2)$}.

In addition, we need a distance metric that captures accurately the proximity of $\bar{\mathcal{A}}_{\text{free}}(\bar{\rho}) \subset \mathbb{T}$. As elaborated in the previous section, the Euclidean distance is not an appropriate distance metric in $\mathbb{T}$ due to the fact that the rotational part $\mathfrak{r}$ evolves on the $n_r$-dimensional sphere. Hence, having already defined the chordal metric $\bar{d}_C$ in \eqref{eq:chordal}, we define a suitable distance metric {for vectors $x = [(x^\mathfrak{t})^\top,(x^\mathfrak{r})^\top]^\top$, $y = [(y^\mathfrak{t})^\top,(y^\mathfrak{r})^\top]^\top$ $\in\mathbb{T}$} as $d_\mathbb{T}:\mathbb{T}^2 \to \mathbb{R}_{\geq 0}$, with 
\begin{equation} \label{eq:metric d}
	d_\mathbb{T}(x,y) = \|x^\mathfrak{t} - y^\mathfrak{r}\|^2 + \bar{d}_C(x^\mathfrak{r} - y^\mathfrak{r}). 
\end{equation}

{The intuition behind the KDF framework is as follows.}
The control scheme of the previous subsection guarantees that the robot can track a trajectory within the bounds \eqref{eq:theorem bounds}. {In other words, given a desired trajectory signal $q_\text{d}:[t_0,t_0+t_f]\to\mathbb{T}$, the control algorithm  \eqref{eq:xi}-\eqref{eq:control law} guarantees that $q_1(t) \in \mathcal{P}(q_\textup{d}(t),\bar{\rho})$, for all $t\in[t_0,t_0+t_f]$. Therefore, by the construction of $\bar{\mathcal{A}}_{\textup{free}}(\bar{\rho})$, if $q_\textup{d}(t)$ belongs to $\bar{\mathcal{A}}_\text{free}(\bar{\rho})$, $q_1(t)$ belongs to $\mathcal{A}_{\textup{free}}$. The proposed sampling-based framework aims at finding a path in $\bar{\mathcal{A}}_{\textup{free}}(\bar{\rho})$; this path will be then endowed with time constraints in order to form the trajectory $q_\textup{d}:[t_0,t_0+t_f] \to \bar{\mathcal{A}}_\textup{free}(\bar{\rho}) $, which will then safely tracked by the system using the designed controller.}

{Common geometric sampling-based motion-planning algorithms follow a standard iterative procedure that build a discrete network $\mathcal{G}=(\mathcal{V},\mathcal{E})$, (tree, roadmap) of points  in the free space connecting the initial configuration to the goal; $\mathcal{V}$ and $\mathcal{E}$ denote the nodes (points) and edges, respectively, of the network. Standard functions involved in such algorithms include $\mathsf{Sample}()$, $\mathsf{Nearest}(\mathcal{G},\mathsf{y})$, $\mathsf{Closest}(\mathcal{G},\mathsf{y},K)$, $\mathsf{Steer}(\mathsf{y},\mathsf{z})$, and $\mathsf{ObstacleFree}(\mathsf{y},\mathsf{z})$; $\mathsf{Sample}()$ samples a random point from a distribution in $\mathcal{A}_{\textup{free}}$; $\mathsf{Nearest}(\mathcal{G},\mathsf{y})$ and $\mathsf{Closest}(\mathcal{G},\mathsf{y},K)$ find the closest and $K$ closest, respectively, nodes of $\mathcal{G}$ to $\mathsf{y}$, according to some distance metric; $\mathsf{Steer}(\mathsf{y},\mathsf{z})$ computes a point lying on line from $\mathsf{z}$ to $\mathsf{y}$
and $\mathsf{ObstacleFree}(\mathsf{y},\mathsf{z})$ checks whether the path from $\mathsf{y}$ to $\mathsf{z}$ belongs to the free space $\mathcal{A}_\textup{free}$ (i.e., collision-free). }

{Among the aforementioned functions, the sampling and collision checking functions $\mathsf{Sample}()$, $\mathsf{ObstacleFree}(\mathsf{y},\mathsf{z})$ respectively, constitute the main building blocks upon which all sampling-based motion-planning algorithms are developed. The framework we propose in this work \textit{modifies} these functions, creating a basis for a new class of sampling-based motion-planning algorithms; such algorithms, in combination with the funnel controller of Section \ref{sec:ppc control}, are able to solve the \textit{kinodynamic} motion-planning problem only by sampling in an extended free \textit{geometric configuration} space, without resorting to sampling of control inputs or simulation of the (uncertain) system dynamics.}

{As stated before, we aim to find a path in the extended free space $\bar{\mathcal{A}}_\textup{free}(\bar{\rho})$. To this end, we need to sample points and perform collision checking in $\bar{\mathcal{A}}_\textup{free}(\bar{\rho})$. Therefore, we define the functions $\mathsf{SampleExt}(\bar{\rho})$ and $\mathsf{ObstacleFreeExt}(\mathsf{y},\mathsf{z},\bar{\rho})$; $\mathsf{SampleExt}(\bar{\rho})$ samples a point from a uniform distribution in the extended free space $\bar{\mathcal{A}}_\textup{free}(\bar{\rho})$; $\mathsf{ObstacleFreeExt}(\mathsf{y},\mathsf{z},\bar{\rho})$ checks whether the path {$X_\text{Line}:[0,\sigma]\to \mathbb{T}$, for some positive $\sigma$,} from $\mathsf{y}$ to $\mathsf{z}$ is collision free with respect to the extended free space, i.e., check whether $\mathsf{y}' \in \bar{\mathcal{A}}_{\text{free}}(\bar{\rho})$, $\forall \mathsf{y}' \in X_\text{Line}$. We elaborate on the exact collision checking procedure later in Remark \ref{rem:coll check}. }

{The new functions $\mathsf{SampleExt}(\cdot)$ and $\mathsf{ObstacleFreeExt}(\cdot)$ can be used in any geometric sampling-based motion planning algorithm, giving thus rise to a new family of algorithms, which produce a safe path in an extended free space $\bar{\mathcal{A}}_\textup{free}(\bar{\rho})$. This path is then tracked by the system using the control algorithm of Section \ref{sec:ppc control}. Note, however, that the control algorithm guarantees tracking of a \textit{time-varying} smooth 
(at least $k$-times continuously differentiable) trajectory $q_\textup{d}(t)$, whereas the output of the respective motion planning algorithm is a path, i.e., a sequence of points in $\mathbb{T}$. Therefore, we smoothen this path and endow it with a time behavior, producing hence a time trajectory. 
The combination of the aforementioned steps, namely the family of geometric sampling-based motion planning algorithms in $\bar{\mathcal{A}}_\textup{free}(\bar{\rho})$, the time endowment, and the funnel-control algorithm of Section \ref{sec:ppc control}, constitute the framework of \textit{KinoDynamic motion planning via Funnel control}, or \textit{KDF} motion-planning framework. This framework solves the kinodynamic motion-planning problem, without resorting to sampling of control inputs or employment of the system dynamics (either in the motion-planning or the control module).} {Note that the smoothening of the output path is not required to be performed by the sampling-based algorithm; it is required by the overall KDF framework due to to the smooth-trajectory requirement of the funnel controller. We elaborate later (see Remark \ref{rem:line segment}) how this smoothening restriction can be relaxed. }

{Next, we provide indicatively two examples of KDF motion-planning algorithms {(KDF-MP)}; we extend the popular RRT and PRM algorithms to form the KDF-RRT and KDF-PRM algorithms, which produce paths in the extended free space $\bar{\mathcal{A}}_\textup{free}(\bar{\rho})$. }
{The two algorithms are presented in Algorithms \ref{alg:main tree ppc} and \ref{alg:prm ppc}, respectively. As described above, the main differences are the sampling and collision-checking functions 
$\mathsf{SampleExt}(\cdot)$ and $\mathsf{ObstacleFreeExt}(\cdot)$, respectively, which operate in the extended free space $\bar{\mathcal{A}}_\textup{free}(\bar{\rho})$. The distance metric used in the several functions is $d_\mathbb{T}$, defined in \eqref{eq:metric d}.}
	

\begin{algorithm}
	\caption{KDF-RRT}\label{alg:main tree ppc}
	\begin{algorithmic}[1]
		\Require{$\bar{\rho}$, $\mathcal{A}_\textup{free}$, $Q_g$, $q_1(0)$}
		\Ensure{Tree $\mathcal{G}$ in $\bar{\mathcal{A}}_\text{free}(\bar{\rho})$}
		\Procedure{TREE}{}
		\State $\mathcal{V} \leftarrow \{q_1(0)\}$; $\mathcal{E} \leftarrow \emptyset$; 
		\State {$\mathsf{ReachGoal} \leftarrow \mathsf{False}$};
		\While { {not $\mathsf{ReachGoal}$} } 
		\State $\mathcal{G} \leftarrow (\mathcal{V},\mathcal{E})$;		
		\State $q_\text{rand} \leftarrow \mathsf{SampleExt}(\bar{\rho})$; 
		\State $q_\text{nearest} \leftarrow \mathsf{Nearest}(\mathcal{G},q_\text{rand})$;
		\State $q_\text{new} \leftarrow \mathsf{Steer}(q_\text{nearest},q_\text{rand})$;
		\If {$\mathsf{ObstacleFreeExt}(q_\text{nearest},q_\text{new},\bar{\rho})$}
		\State $V \leftarrow V\cup \{q_\text{new}\}$;
		\State	$\mathcal{E} \leftarrow \mathcal{E} \cup\{(q_\text{nearest},q_\text{new})\}$;
		\EndIf
		\For {{$q'\in \mathcal{V}$}} 
		\If { { $\mathsf{ObstacleFreeExt}(q',Q_g,\bar{\rho})$} } 
		\State ${V \leftarrow V\cup \{Q_g\}} $;
		\State	$ {\mathcal{E} \leftarrow \mathcal{E} \cup\{(q',Q_g)\}} $; 
		\State {$\mathsf{ReachGoal} \leftarrow \mathsf{True}$};
		\EndIf
		\EndFor
		\EndWhile
		\EndProcedure
	\end{algorithmic}
\end{algorithm}

\begin{algorithm}
	\caption{KDF-PRM}\label{alg:prm ppc}
	\begin{algorithmic}[1]
		\Require{$\bar{\rho}$, $\mathcal{A}_\textup{free}$, $N$}
		\Ensure{Graph $\mathcal{G}$ in $\bar{\mathcal{A}}_\text{free}(\bar{\rho})$}
		\Procedure{PRM}{}
		\State $\mathcal{V} \leftarrow \emptyset$; 
		\State $\mathcal{E} \leftarrow \emptyset$; 
		\State $i \leftarrow 0$;
		\While {$i < N$} 
		\State $\mathcal{G} \leftarrow (\mathcal{V},\mathcal{E})$;
		\State $q_\text{rand} \leftarrow \mathsf{SampleExt}(\bar{\rho})$; 
		\State $\mathcal{V} \leftarrow \mathcal{V}\cup \{ q_\text{rand} \} $;		
		\State $N_q \leftarrow \mathsf{Closest}(\mathcal{G},q_\text{rand},K)$;
		\ForAll{$q' \in N_q$}		
		\If {$\mathsf{ObstacleFreeExt}(q',q_\text{rand},\bar{\rho})$}
		\If {{$\{ (q',q_\text{rand}) \} \notin \mathcal{E}$}}
		\State $\mathcal{E} \leftarrow \mathcal{E} \cup \{ (q',q_\text{rand}) \} $;
		\State $i \leftarrow i+1$;	
		\EndIf 
		\EndIf
		\EndFor

		\EndWhile
		\EndProcedure
	\end{algorithmic}
\end{algorithm}

{The algorithms follow identical procedures as their original counterparts. KDF-RRT builds a tree $\mathcal{G}$ aiming to connect the initial point $q_1(0)$ to the goal $Q_g$ through the extended free space $\bar{\mathcal{A}}_\text{free}(\bar{\rho})$. It samples random points in $\bar{\mathcal{A}}_\text{free}(\bar{\rho})$ (line 7) and aims to extend $\mathcal{G}$ towards them (lines 7-11). The algorithm ends if $\mathcal{G}$ can be safely connected to the goal $Q_g$ via $\bar{\mathcal{A}}_\text{free}(\bar{\rho})$ (lines 12-16). After the execution of the algorithm, a standard search algorithm can be employed to find the sequence of edges that lead from $q_1(0)$ to $Q_g$ and concatenate them to produce a solution path.
Similarly, KDF-PRM builds a graph $\mathcal{G}$ of $N$ nodes in $\bar{\mathcal{A}}_\text{free}(\bar{\rho})$. It samples random points (line 7) which it aims to connect to $K$ closest points of $\mathcal{G}$ (lines 9-14). After the construction of {KDF-PRM}, requests of multiple queries, each consisting of a starting and a goal point in $\bar{\mathcal{A}}_{\text{free}}(\bar{\rho})$, are attempted to be solved. The query phase attempts to connect the starting and goal points to the same component of the PRM (e.g., using the $\mathsf{ObstacleFreeExt}(\cdot)$ function), and reports failure if it fails. Note that the probability of that succeeding increases with the density of the PRM. More details regarding the RRT and PRT algorithms can be found in the related literature, e.g., \cite{kuffner2000rrt,kavraki1996probabilistic}. The probabilistic completeness of the {KDF-MP} algorithms follows from the probabilistic completeness of their original counterparts (see \cite{verginis21sampling}).}

{Note that both {KDF-MP algorithms} take as input the original free space $\mathcal{A}_\textup{free}$ and the vector $\bar{\rho}$ that forms the extended free space $\bar{\mathcal{A}}_\text{free}(\bar{\rho})$ as in \eqref{eq:extended free space}.  This vector is the connection of the algorithms to the control module of Section \ref{sec:ppc control}, since it consists of the funnel bounds $[\bar{\rho}^\mathfrak{t}_1,\dots,\bar{\rho}^\mathfrak{t}_{n_{tr}},\bar{\rho}^\mathfrak{r}_1,\dots,\bar{\rho}^\mathfrak{r}_{n_r} ]^\top$, and can be chosen by the user. Intuitively, smaller funnel values lead to a larger extended free space $\bar{\mathcal{A}}_\text{free}(\bar{\rho})$ (note that if $\bar{\rho}$ consists of zeros, then $\bar{\mathcal{A}}_\text{free}(\bar{\rho}) = \mathcal{A}_\textup{free}$), giving the chance to navigate through potential narrow passages or with larger distance from the obstacles. Moreover, note that, in order to be able to connect the goal $Q_g$ to the data structure via $\bar{\mathcal{A}}_\text{free}(\bar{\rho})$ (e.g., lines 13-15 of Algorithm \ref{alg:main tree ppc}), the goal itself must belong to $\bar{\mathcal{A}}_\text{free}(\bar{\rho})$. In view of Assumption \ref{ass:path ass}, $Q_g$ belongs to the open set $\mathcal{A}_\textup{free}$. Therefore, by invoking continuity properties of the free space, we conclude that there exists a $\rho_\varepsilon \in\mathbb{R}^{n_{tr}+n_r}$ such that $Q_g \in \bar{\mathcal{A}}_\text{free}(\rho_{\varepsilon})$\footnote{{Since the free space $\mathcal{A}_\textup{free}$ and the goal configuration $Q_g$ are known,  such a $\rho_\varepsilon$ can be explicitly found.}}. Hence, by choosing $\bar{\rho}$ such that $\rho_\varepsilon \succeq \bar{\rho}$, one can achieve $\bar{\mathcal{A}}_\text{free}(\rho_{\varepsilon}) \subseteq \bar{\mathcal{A}}_\text{free}(\bar{\rho})$ and hence $Q_g \in \bar{\mathcal{A}}_\text{free}(\bar{\rho})$. As stressed before, however, tight funnels might need excessively large control inputs that might not be realizable by real actuators. Therefore, one must take into account the capabilities of the system when choosing $\bar{\rho}$ and the funnel functions of Section \ref{sec:ppc control}, as mentioned in Remark \ref{rem:funnel}. If it is not possible to select $\bar{\rho}$ such that $\rho_\varepsilon \succeq \bar{\rho}$ (e.g., if the goal $Q_g$ is too close to an obstacle), then one can consider a new goal $Q_g'$ that is close to $Q_g$ and belongs to $\bar{\mathcal{A}}_\textup{free}(\bar{\rho})$. That is, $Q_g'\coloneqq \arg\min_{q \in \mathbb{A}} d_\mathbb{T}(q,Q_g)$, where $\mathbb{A}$ is a compact subset of $\bar{\mathcal{A}}_\textup{free}(\bar{\rho})$ for a chosen $\bar{\rho}$.
}

The control protocol of Section \ref{sec:ppc control} guarantees tracking of a \textit{time-varying} smooth (at least $k$-times continuously differentiable) trajectory $q_\textup{d}(t)$, whereas the output of a {KDF-MP} algorithm (e.g., Algorithms \ref{alg:main tree ppc}, \ref{alg:prm ppc}) is a path, i.e., a sequence of points in $\mathbb{T}$. Therefore, we endow the latter with a time behavior, as follows. 

The output path is first converted to a smooth (at least $k$-times continuously differentiable) one. This is needed to smoothly interpolate the connecting points of the consecutive edges of the solution path that is obtained from the tree of Algorithm \ref{alg:main tree ppc}. {This smoothening procedure must be performed in accordance to the extended free space $\bar{\mathcal{A}}_{\text{free}}(\bar{\rho})$, so that the smoothed path still belongs in $\bar{\mathcal{A}}_{\text{free}}(\bar{\rho})$. }
Time constraints are then enforced on the smooth path to create a timed trajectory $q_\textup{d}:[0,t_f] \to \bar{\mathcal{A}}_\textup{free}(\bar{\rho})$, for some $t_f > 0$, which is then given to the control protocol of Section \ref{sec:ppc control} as the desired trajectory input. Note that $q_\textup{d}(0)$ must satisfy the funnel constraints \eqref{eq:funnel constraints}. 

\begin{remark} \label{rem:line segment}
{It is also possible to use the output path of the KDF motion-planning algorithm without any  post-processing steps, i.e., the {raw segments} that correspond to the edges of the respective data structure (tree, graph). Each one of these segments can be endowed with time constraints, as well as separate funnel functions. The control algorithm of Section \ref{sec:ppc control} is then applied separately for these segments, possibly with discontinuities at the connecting points. Although one avoids the use of post-processing steps on the output path, such discontinuities might be problematic for the actuators and might jeopardize the safety of the system.}
\end{remark}

\begin{remark}
	Note that the duration of the resulting trajectory $t_f$, and hence the respective velocity $\dot{q}_\textup{d}$ can be a prior chosen by a user and hence the robotic system can execute the path in a predefined time interval. There is also no constraint on this duration, since the control protocol of Section \ref{sec:ppc control} guarantees funnel confinement with respect to any arbitrarily fast time trajectory. Nevertheless, it should be noted that the physical limits of the system's actuators prevent the achievement of any time trajectory, and the latter should be properly defined in accordance to any such limits, {similarly to the selection of the control gains and the funnel characteristics (see Remarks \ref{rem:gains} and \ref{rem:funnel}).}
\end{remark}


{Algorithm \ref{alg:KDF} provides the overall KDF framework, including the {KDF-MP} algorithm, the conversion to a time-varying trajectory, and the application of the funnel control algorithm of Section \ref{sec:ppc control}. The algorithm extracts first the bounds $\bar{\rho} =$ $[\bar{\rho}^\mathfrak{t}_1,\dots,\bar{\rho}^\mathfrak{t}_{n_{tr}},\bar{\rho}^\mathfrak{r}_1,\dots,\bar{\rho}^\mathfrak{r}_{n_r} ]^\top$ (line 2) which are used in a {KDF-MP} (e.g., KDF-RRT or KDF-PRM as in Algorithms \ref{alg:main tree ppc},\ref{alg:prm ppc}), along with the free space $\mathcal{A}_\textup{free}$, and other potential arguments such as the goal $Q_g$ or a desired number of nodes $N$ (line 3). The output path is converted to a smooth time-varying trajectory with the desired duration $t_f$ (line 4), which is then tracked by the system using the funnel control algorithm (line 5).}

\begin{algorithm}
	\caption{KDF}\label{alg:KDF}
	\begin{algorithmic}[1]
		\Require{$\mathcal{A}_\textup{free}$, $Q_g$, $N$, $t_f$,  $k$, $q_1(0)$, $q_\textup{d}$, $\rho^\mathfrak{t}_j$, $\rho^\mathfrak{r}_\ell$, $K^\mathfrak{t}$, $K^\mathfrak{r}$, $K_i$, $j\in\{1,\dots,n_{tr}\}$, $\ell \in \{1,\dots,n_r\}$ }
		\Ensure{$u(t)$}
		\Procedure{KDF}{}		
		\State $\bar{\rho} \leftarrow \mathsf{Bounds}(\rho^\mathfrak{t}_j, \rho^\mathfrak{r}_\ell)$;
		\State $\mathbf{q}_p \leftarrow \textup{KDF}-\mathsf{MP}(\bar{\rho},\mathcal{A}_\textup{free},  Q_g, q_1(0), N)$;
		\State $q_\textup{d} \leftarrow \mathsf{TimeTraj}(\mathbf{q}_p,t_f)$;
		\State $u \leftarrow \mathsf{FunnelControl}(t_f, k, q_\textup{d}, \rho^\mathfrak{t}_j, \rho^\mathfrak{r}_\ell, K^\mathfrak{t}, K^\mathfrak{r}, K_i)$;		
		\EndProcedure
	\end{algorithmic}
\end{algorithm}

%

\begin{remark}[\textbf{Collision Checking in $\bar{\mathcal{A}}_\text{free}(\bar{\rho})$}]	 \label{rem:coll check} 

{The proposed feedback control scheme guarantees that $q(t) \in \mathcal{P}( q_\textup{d}(t), \bar{\rho})$ for any trajectory $q_\textup{d}(t)$, formed by the {several line segments $X_\text{Line}$ that connect the nodes in $\mathcal{V}$ sampled in {Algorithm} \ref{alg:main tree ppc}}. 
Therefore, checking whether the points $q_s \in X_\text{Line}$ belong to $\mathcal{A}_\text{free}$, as in standard motion planners \cite{kavraki1996probabilistic,karaman2010incremental}, is not sufficient. For each such point $q_s \in X_\text{Line}$, one must check whether $z \in \mathcal{A}_\text{free}$, $\forall z \in \mathcal{P}( q_s, \bar{\rho})$, which is equivalent to checking if $q_s \in \bar{\mathcal{A}}_\text{free}(\bar{\rho})$. For simple robotic structures, like, e.g., mobile robots or UAV (see Section \ref{subsec:computer sim}), whose volume can be bounded by convex shapes, one can enlarge this or the obstacles' volume by $\bar{\rho}$ and perform the collision checking procedure on the remaining free space. However, more complex structures e.g., robotic manipulators (see Section \ref{sec:exps}), necessitate a more sophisticated approach, since they can assume nonconvex complex shapes in various configurations. 
For such systems there are two {steps} one could follow. Firstly, for each $q_s$, a finite number of points $z$ can be sampled from a uniform distribution in $\mathcal{P}( q_s, \bar{\rho})$ and separately checked for collision. {For} a sufficiently high number of such samples, and assuming a certain {``fat"-structure of the workspace obstacles (e.g., there are no long and skinny obstacles such as wires, cables and tree branches, etc., see (\cite{van1993complexity}) for more details)}, this approach can be considered to be complete, i.e., the resulting path will belong to the extended free space $\mathcal{A}_\text{free}$.  
Secondly, we calculate the limit poses of each link of the robot, based on the lower and upper bounds by the joints that affect it, as defined by $\bar{\rho}$. Subsequently, we compute the convex hull of these limit poses, which is expanded by an appropriate constant to yield an over-approximation of the swept volume of the potential motion of the link, as described in \cite{schulman2014motion}. The resulting shape is then checked for collisions for each link separately.}
\end{remark}

\begin{figure}
	\centering
	\includegraphics[width=.45\textwidth]{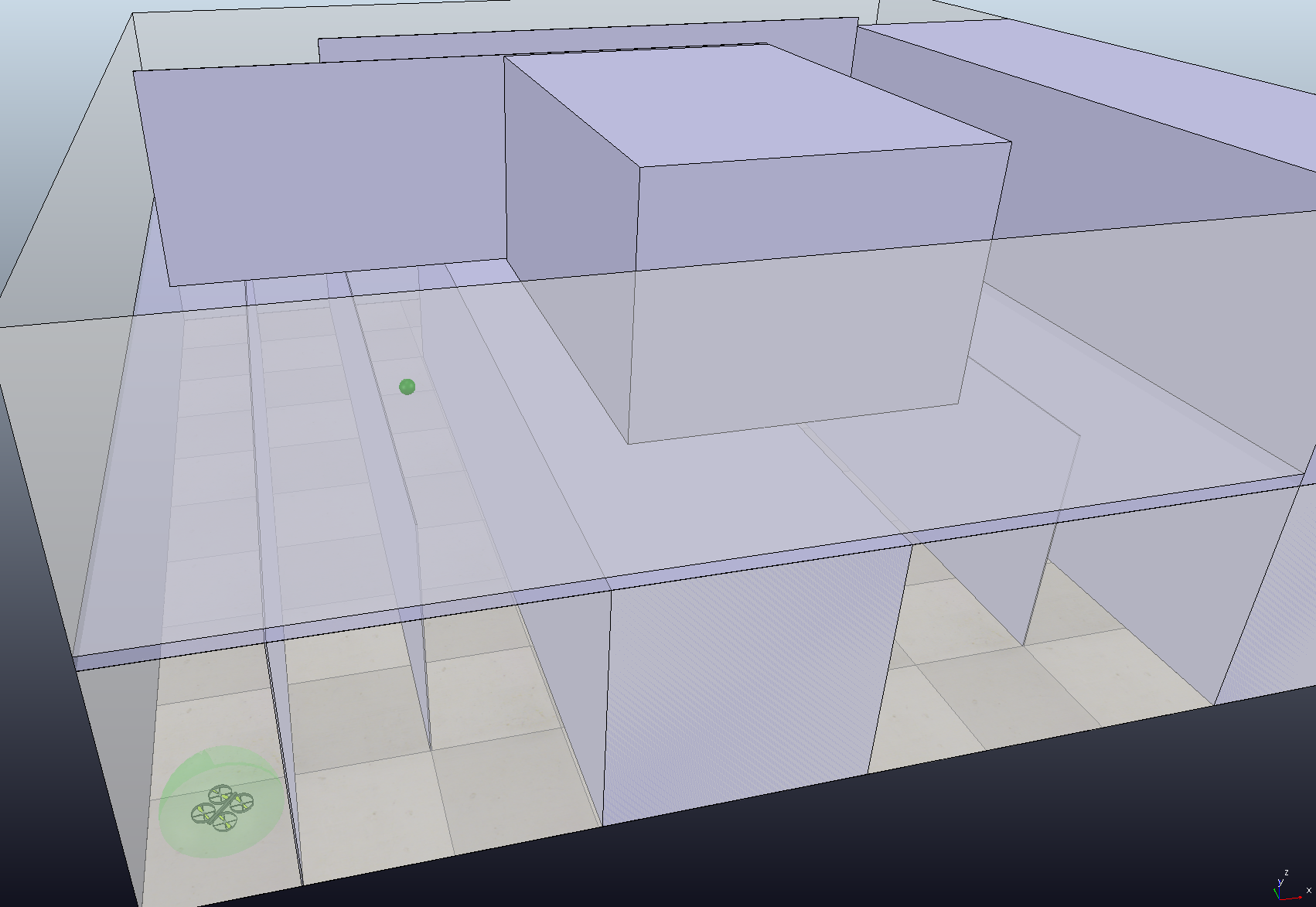}
	\caption{{The obstacle-cluttered 3D workspace and the starting position of the UAV, along with its augmented volume (green sphere around the UAV) to account for $\bar{\mathcal{A}}_\text{free}(\bar{\rho})$, and the goal configuration (smaller green sphere).} }
	\label{fig:uav}
\end{figure}

\section{Experimental Results} \label{sec:exps}

This section is devoted to experimental results that validate the theoretical findings. Firstly, we present computer simulation results from the application {of the KDF framework} to a UAV moving in $\mathbb{R}^3$, as well as a UR5 robot, in obstacle-cluttered environments. {We use KDF-RRT as the motion planner and we compare the efficiency with a standard geometric and kinodynamic RRT algorithms}.

Secondly, we present experimental results using the {KDF framework on a $6$DOF HEBI manipulator}. We compare the performance of the proposed funnel control algorithm with our previous work \cite{verginis21sampling} as well as a standard PID controller.

\subsection{Computer Simulations} \label{subsec:computer sim}

We apply here the {KDF framework by using a KDF-RRT motion planner and the funnel control algorithm presented in Sections \ref{sec:motion planner} and \ref{sec:ppc control}}, respectively, in two computer simulated scenarios by using the V-REP robotic simulator \cite{Vrep}. In both cases, the KDF-RRT was implemented using the algorithms of the OMPL library \cite{OMPL}, which was appropriately interfaced with V-REP. The control algorithm was implemented via a ROS node in MATLAB environment, communicating with the V-REP scenes using ROS messages at a frequency of 100kHz. The V-REP scenes were updated at a frequency of 1kHz. \\

\underline{\textbf{Unmanned Autonomous Vehicle}}\\

The first case consists of a UAV moving in an obstacle-cluttered $3$D space, as shown in Fig.  \ref{fig:uav}. In order to comply with the dynamic model of Section \ref{sec:pf}, we view the UAV as a fully actuated rigid body with dynamics 
\begin{subequations} \label{eq:sim dynamics}
\begin{align} 
	&\dot{q}_1 = f_1(q_1,t) + g_1(q_1,t)q_2 \\
	&\dot{q}_2 = f_2(q_1,q_2,t) + g_2(q_1,q_2,t)u
\end{align}
\end{subequations}
where $q_1=[q_1^\mathfrak{t},q_2^\mathfrak{t},q_3^\mathfrak{t}]^\top$, $q_2$ $\in\mathbb{R}^3$ are the linear position and velocity of the UAV, $u$ is the $3$D force, acting as the control input, and $f_1$, $g_1,$ $f_2$, $g_2$ are unknown functions satisfying Assumption \ref{ass:dynamics}. The incorporation of standard underactuated UAV dynamics in the proposed framework consists part of our future work. 

\begin{figure}[t]
	\centering
	\includegraphics[width=.5\textwidth]{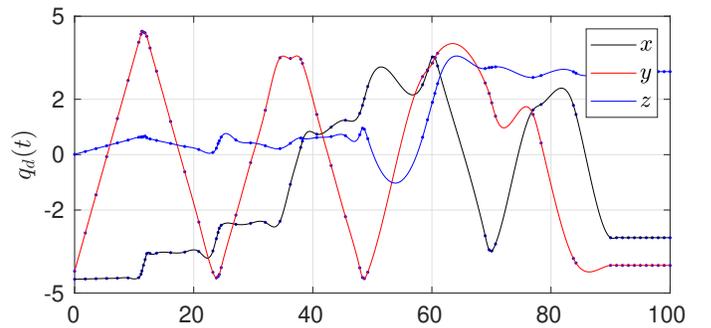}
	\caption{ {The 3D (x,y,z) output path  of the KDF-RRT algorithm (blue points), along with the smoothened time-varying trajectory, in meters, for the UAV scenario.} }
	\label{fig:q_d_3D}
\end{figure}

The UAV aims to navigate safely to a goal position $Q_g = [-3,-4,3]^\top$, starting from $q_1(0) = [-3.5,-4,0.01]^\top$ within $20$ seconds and space bounds $(-5,5)$, $(-5,5)$, $(0,4)$ in x-, y-, and z-dimensions respectively (the x- and y- dimensions correspond to the floor dimensions of Fig. \ref{fig:uav}). For the safe path tracking, we choose the exponentially decaying funnel functions $\rho^\mathfrak{t}_1(t)= \rho^\mathfrak{t}_2(t)=\rho^\mathfrak{t}_3(t) = \rho(t) \coloneqq 0.15\exp(-0.1t) + 0.05$ $\in[0.05,0.2]$ (meters), implying $\bar{\rho} = 0.2[1,1,1]^\top$, as well as 
$\rho_{2_j}(t) \coloneqq (\max\{2|e_{2_j}(0)|,0.5\}-0.1)\exp(-0.1t) + 0.1$ $\in[0.1,\max\{2|e_{2_j}(0)|,0.5\}]$ (meters/second) for $j\in\{1,\dots,3\}$. Hence, the minimum distance from the obstacles and the path output by the {KDF-RRT} algorithm must be larger than $0.2$ meters. In the conducted simulation, this was achieved by enlarging the radius of the UAV volume sphere  by $0.2$, which was then checked for collision (see the green sphere in Fig. \ref{fig:uav}).

The obtained path consists of $50$ points in $\bar{\mathcal{A}}_\text{free}(\bar{\rho})$ and is converted to a smooth time trajectory as follows. 
We construct $q_\text{d}:[0,100] \to \mathbb{R}^3$, such that $q_\text{d}(0) = q_1(0)$ and $q_\text{d}(t) = Q_g$, for $t\in[90,100]$, using a standard fitting procedure. 
For the construction of $q_\textup{d}(t)$, each pair of two path points $\mathsf{h}=[\mathsf{h}_1,\mathsf{h}_2,\mathsf{h}_3]^\top$, $\mathsf{w}=[\mathsf{w}_1,\mathsf{w}_2,\mathsf{w}_3]^\top$ $\in\bar{\mathcal{A}}_\text{free}(\bar{\rho})$  is endowed with a time duration proportional to their distance, i.e., equal to $\frac{1}{90}\max_{i\in\{1,2,3\}}\{ |\mathsf{h}_i - \mathsf{w}_i| \}$. 
{For the specific control and scene update frequencies and chosen funnels, the control gains that yield satisfactory behavior (reasonable control inputs and avoidance of oscillations) were found via offline tuning to be $K^\mathfrak{t} = 2I_3$ and $K_2 = 35I_3$. } 


The signals of the resulting motion of the UAV for the two different cases are depicted in Figs. \ref{fig:errors_UAV_smooth} and \ref{fig:distance_u_UAV_smooth}. In particular, Fig. \ref{fig:errors_UAV_smooth} shows the evolution of the errors $e^\mathfrak{t}_j(t)$, $e_{2_m}(t)$ {(in meters and meters/second, respectively)} along with the respective funnel functions $\rho^\mathfrak{t}_j(t)$, $\rho_{2_m}(t)$, 
It can be verified that the errors always respect the respective funnels, guaranteeing thus the successful execution of the respective trajectories. 
Moreover, Fig. \ref{fig:distance_u_UAV_smooth} 
depicts the distance of the UAV from the obstacles $D_{UAV}(t)$ (in meters) as well as the resulting control inputs {$u = [u_1,u_2,u_3]^\top$ for the three spatial dimensions (in Newton). Although the output path was smoothened without taking into account the extended free space $\bar{\mathcal{A}}_\text{free}(\bar{\rho})$, the UAV was able to successfully navigate to the goal configuration safely. Moreover, the funnel controller produced reasonable control inputs, without  excessive oscillations or extreme magnitude.}\\

\begin{figure}
	\centering
	\includegraphics[width=.5\textwidth]{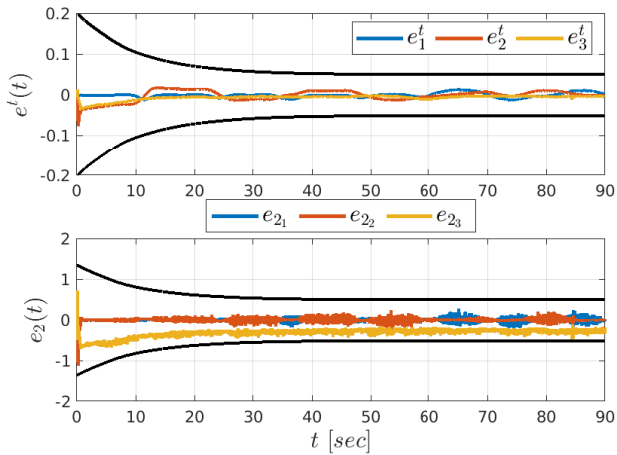}
	\caption{{Top: the evolution of the errors $e^\mathfrak{t}_1(t)$, $e^\mathfrak{t}_2(t)$, $e^\mathfrak{t}_3(t)$ (in meters), along with the funnels $\rho^\mathfrak{t}_1(t) = \rho^\mathfrak{t}_2(t) = \rho^\mathfrak{t}_3(t)$, shown in black, for $t \in [0,90]$ seconds. Bottom: the evolution of the velocity errors $e_{2_1}(t)$, $e_{2_2}(t)$, $e_{2_3}(t)$ (in meters/second), along with the funnels $\rho_{2_1}(t) = \rho_{2_2}(t) = \rho_{2_3}(t)$, shown in black, for $t \in [0,90]$ seconds.} }
	\label{fig:errors_UAV_smooth}
\end{figure}

\begin{figure}
	\centering
	\includegraphics[width=.5\textwidth]{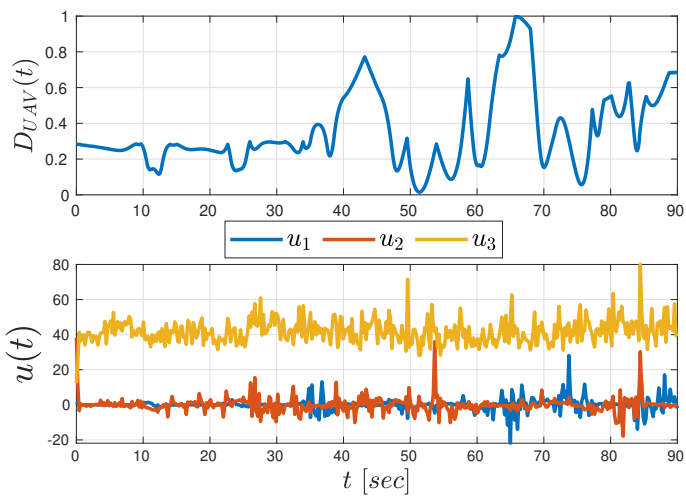}
	\caption{{Top: the distance $D_{UAV}(t)$ (in meters) of the UAV from the obstacles for $t \in [0,90]$ seconds. Bottom: the evolution of the control inputs $u_1(t), u_2(t), u_3(t)$ (in Newton) for the three spatial dimensions and $t \in [0,90]$ seconds.}}
	\label{fig:distance_u_UAV_smooth}
\end{figure}

%

\underline{\textbf{UR5 Robotic Manipulator}}\\

\begin{figure}
	\centering
	\includegraphics[width=.45\textwidth]{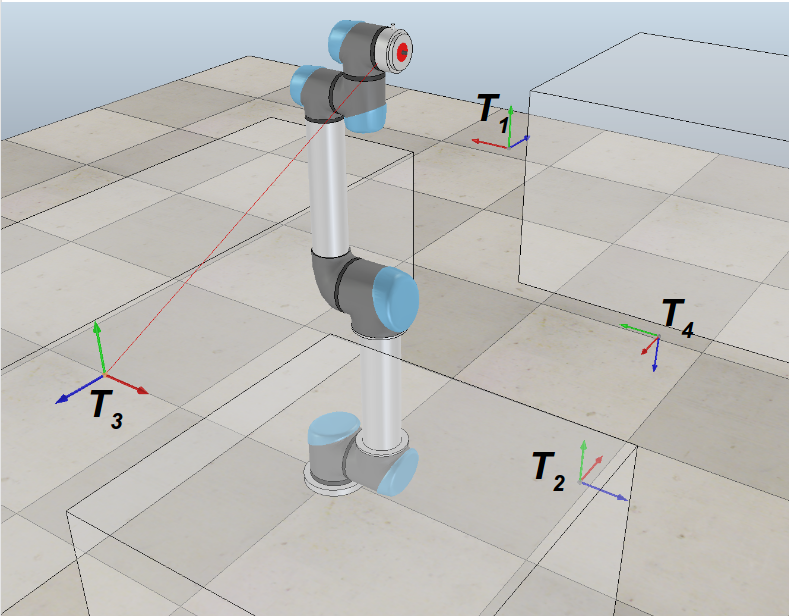}
	\caption{{The initial configuration of the UR5 robot in an obstacle-cluttered environment with four targets.}}
	\label{fig:ur5}
\end{figure}

The second case consists of a UR5 $6$DOF robotic manipulator, whose dynamics are considered to have the form \eqref{eq:sim dynamics} and whose end-effector aims to sequentially navigate to four points in $\mathbb{R}^6$ (position, orientation), as pictured in Fig. \ref{fig:ur5}. By using inverse kinematics algorithms, we translate these points to desired points for the joint variables of the manipulator, which are then used in a sequential application of the proposed scheme. We consider here that the base joint of the manipulator operates in the unit circle $[0,2\pi)$, whereas the rest of the joints operate in  $[-\pi,\pi] \subset \mathbb{R}$ defined by mechanical and structural limits, resulting in $q_1 = [q_{1_1},\dots,q_{1_6}]^\top = [q^\mathfrak{t}_1,\dots,q^\mathfrak{t}_5,q^\mathfrak{r}_1]^\top$, based on the notation of Section \ref{sec:pf}. 

We consider that the robot end-effector has to sequentially navigate from its initial configuration $q_0 = [0,0,0,0,0,0]^\top$ to the following four target points (shown in Fig. \ref{fig:ur5}).
\begin{itemize}
	\item Target 1: $T_1 = [-0.15,-0.475,0.675]^\top$ and Euler-angle orientation $[\frac{\pi}{2},0,0]^\top$, which yields the configuration $q_1 = [-0.07, -1.05, 0.45, 2.3, 1.37, -1.33]^\top$.
	\item Target 2: $T_2 = [-0.6,0,2.5]^\top$ and Euler-angle orientation $[0,-\frac{\pi}{2},-\frac{\pi}{2}]^\top$, which yields the configuration $q_2 = 
	[1.28, 0.35, 1.75, 0.03, 0.1, -1.22]^\top$
	\item Target 3: $T_3 = [-0.025,0.595,0.6]^\top$ and Euler-angle orientation $[-\frac{\pi}{2},0,\pi]^\top$, which yields the configuration $q_3 = [-0.08, 0.85, -0.23, 2.58, 2.09, -2,36]^\top$ 
	\item Target 4: $T_4 = [-0.525,-0.55,0.28]^\top$ and Euler-angle orientation $[\pi,0,-\frac{\pi}{2}]^\top$, which yields the configuration $q_4 = [-0.7, -0.76, -1.05, -0.05, -3.08, 2.37]\top$
\end{itemize} 

{The target points where chosen such that they yield increasing difficulty with respect to the navigation path of the robot.}
The paths for each pair are computed on the fly {using the KDF-RRT algorithm}, after the manipulator reaches 
each target. For the safe tracking of the four output paths, we choose the funnel functions such that $\bar{\rho} = 0.01[1,15,15,15,15,15]$, as will be elaborated later. 
Regarding the collision checking in $\bar{\mathcal{A}}_\text{free}(\bar{\rho})$ of KDF-RRT, {we check a finite number of samples around each point of the resulting path for collision.}
We run {KDF-RRT} with $10$ and $50$ such samples and we compared the results to a standard {geometric} RRT algorithm in terms of time per number of nodes. The results for {$30$ runs} of the algorithms are given in Figs. \ref{fig:basic_times}-\ref{fig:basic_nodes} for the four paths, in logarithmic scale. One can notice that the average nodes created do not differ significantly among the different algorithms. As expected, however, {KDF-RRT} requires more time than the standard {geometric} RRT algorithm, since it checks the extra samples in $\bar{\mathcal{A}}_\text{free}(\bar{\rho})$ for collision. One can also notice that the time increases with the number of samples. However, more samples imply greater coverage of $\bar{\mathcal{A}}_\text{free}(\bar{\rho})$ and {hence the respective solutions are more likely to be complete with respect to collisions.}

\begin{figure}[t]
	\centering
	\subcaptionbox{}
	{\includegraphics[width = 0.24\textwidth]{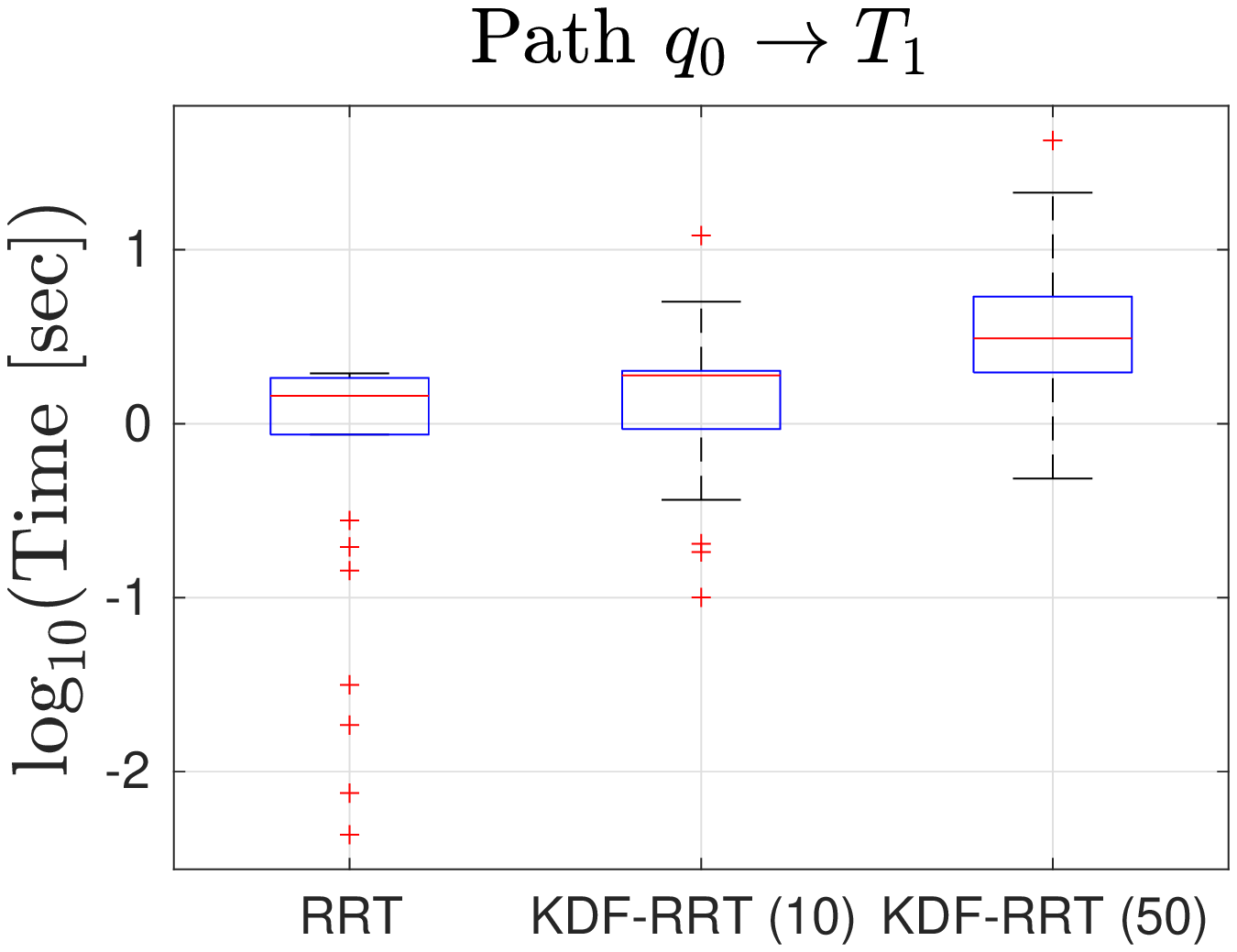}}
	\subcaptionbox{}
	{\includegraphics[width = 0.24\textwidth]{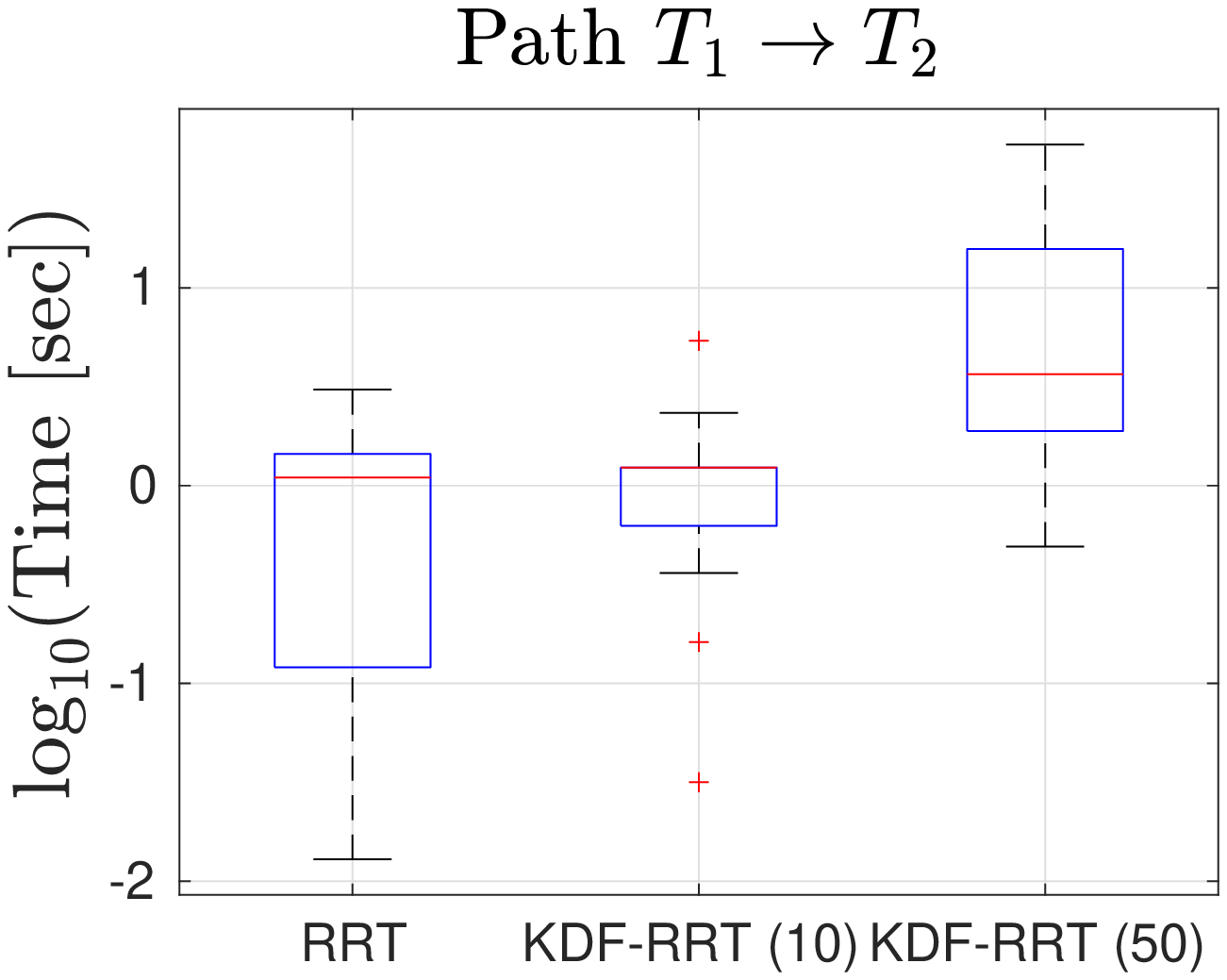}}
	\subcaptionbox{}
	{\includegraphics[width = 0.24\textwidth]{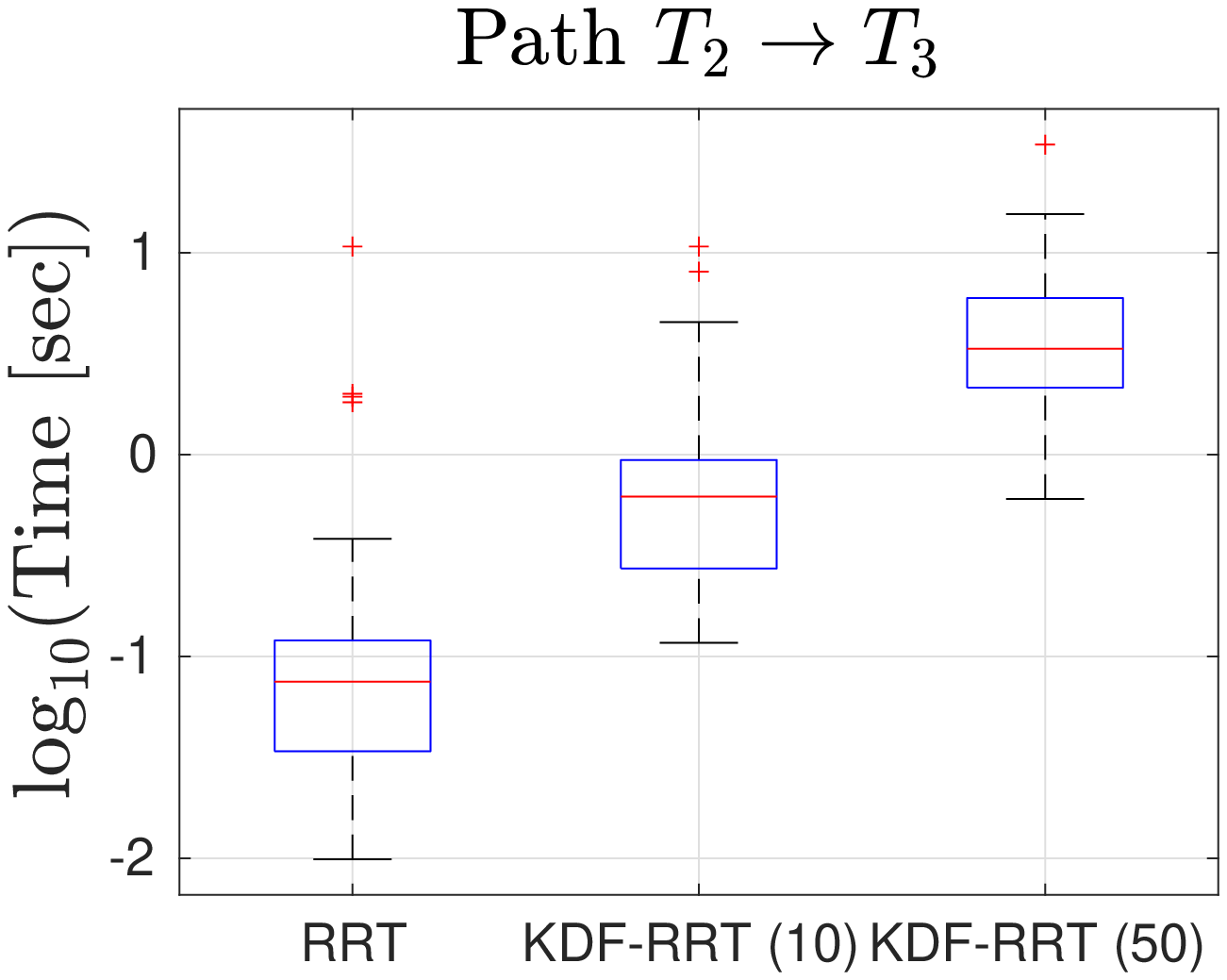}}
	\subcaptionbox{}
	{\includegraphics[width = 0.24\textwidth]{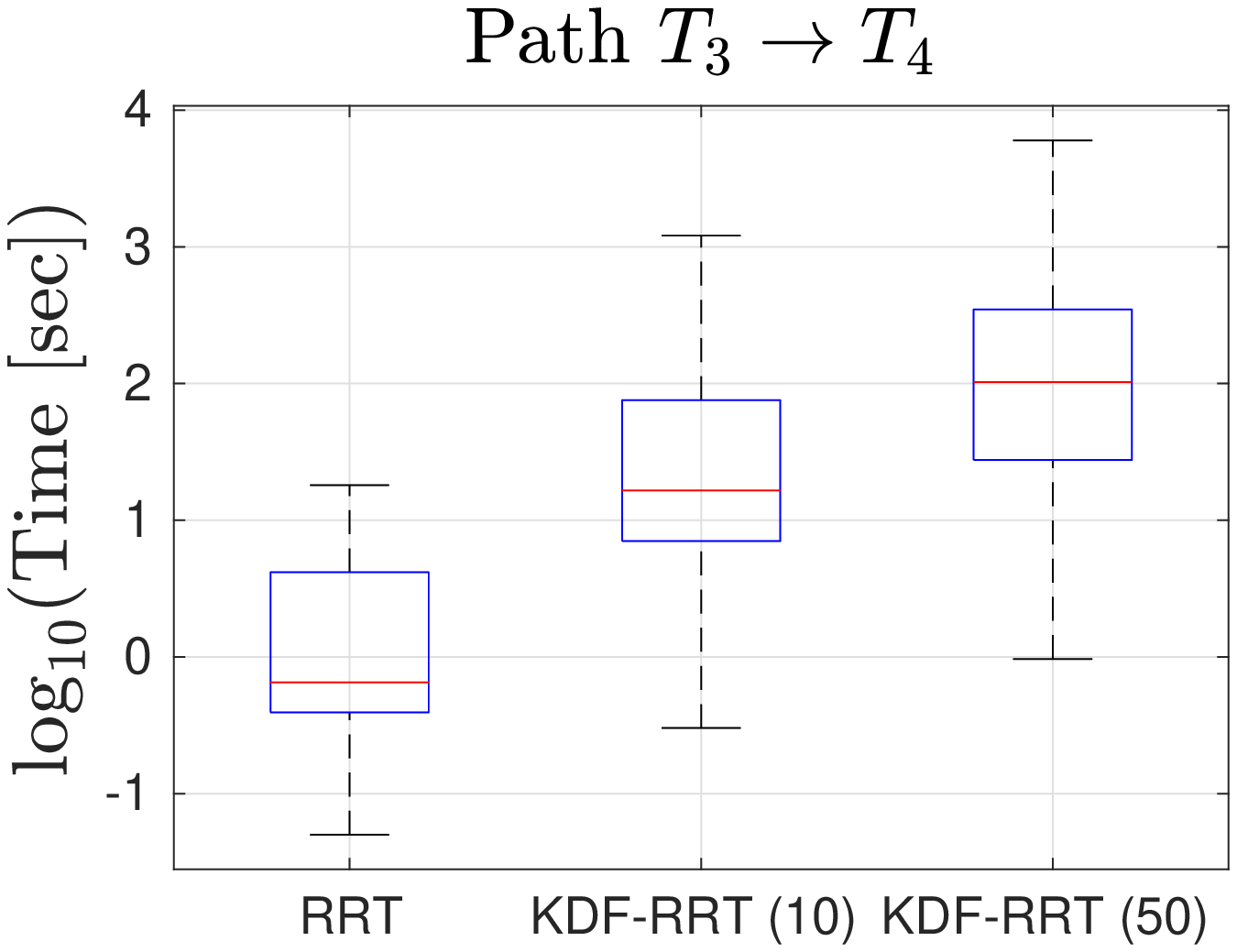}}
	\caption{Box plots showing the execution time of the three algorithms (in logarithmic scale) for the four paths; '+' indicate the outliers. } \label{fig:basic_times}
\end{figure}

{Since, in contrast to the standard geometric RRT, KDF-RRT implicitly takes into account the robot {dynamics \eqref{eq:sim dynamics}} through the  designed tracking control scheme and the respective extended free space $\bar{\mathcal{A}}_\text{free}(\bar{\rho})$, we compare the results to a standard kinodynamic RRT algorithm that simulates forward the robot dynamics, assuming {known} dynamical parameters. In particular, we run the algorithm only for the first two joints, with initial configuration $[0,0]^\top$ and a {randomly chosen goal} configuration at $[-\frac{\pi}{18},\frac{\pi}{4}]^\top$ rad, while keeping the other joints fixed at $0$. For the forward simulation of the respective dynamics we chose a sampling step of $10^{-3}$ sec and total simulation time $30$ sec for each constant control input. The termination threshold distance was set to $0.25$ (with respect to the distance $d_\mathbb{T}$), i.e., the algorithm terminated when the forward simulation reached a configuration closer than $0.25$ units to the goal configuration. 	
{The results for 10 runs of the algorithm are depicted in Fig. \ref{fig:kinod_time_nodes}, which provides the execution time and number of nodes created in logarithmic scale. Note that, even for this simple case (planning for only two joints), the execution time is comparable to the KDF-RRT case of $50$ samples in the fourth path scenario $q_3 \to q_4$.} Running the kinodynamic RRT for more than two joints resulted in unreasonably large execution times {(more than 1 hour)} and hence they are not included in the results. 
{This can be attributed to the randomized inputs and complex robot dynamics; the bias-free random sample of constant inputs and forward simulation of the complex robot dynamics requires a significant amount of time to sufficiently explore the $12$-dimensional state space. One may argue that other accelerated algorithms can be used (e.g., BIT \cite{gammell2015batch}), however it is not trivial to reduce such long running time.}}

\begin{figure}[t]
	\centering
	\subcaptionbox{}
	{\includegraphics[width = 0.24\textwidth]{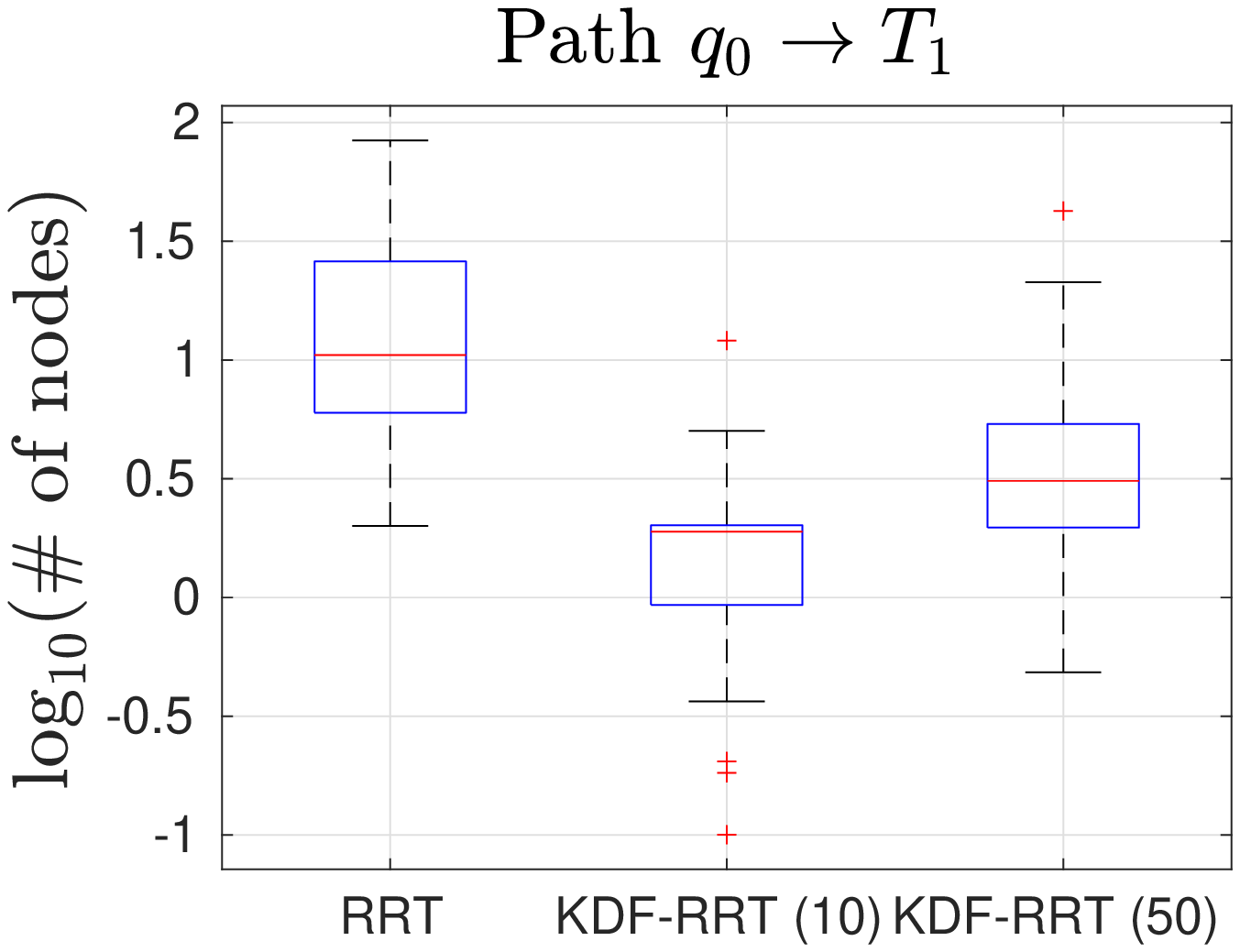}}
	\subcaptionbox{}
	{\includegraphics[width = 0.24\textwidth]{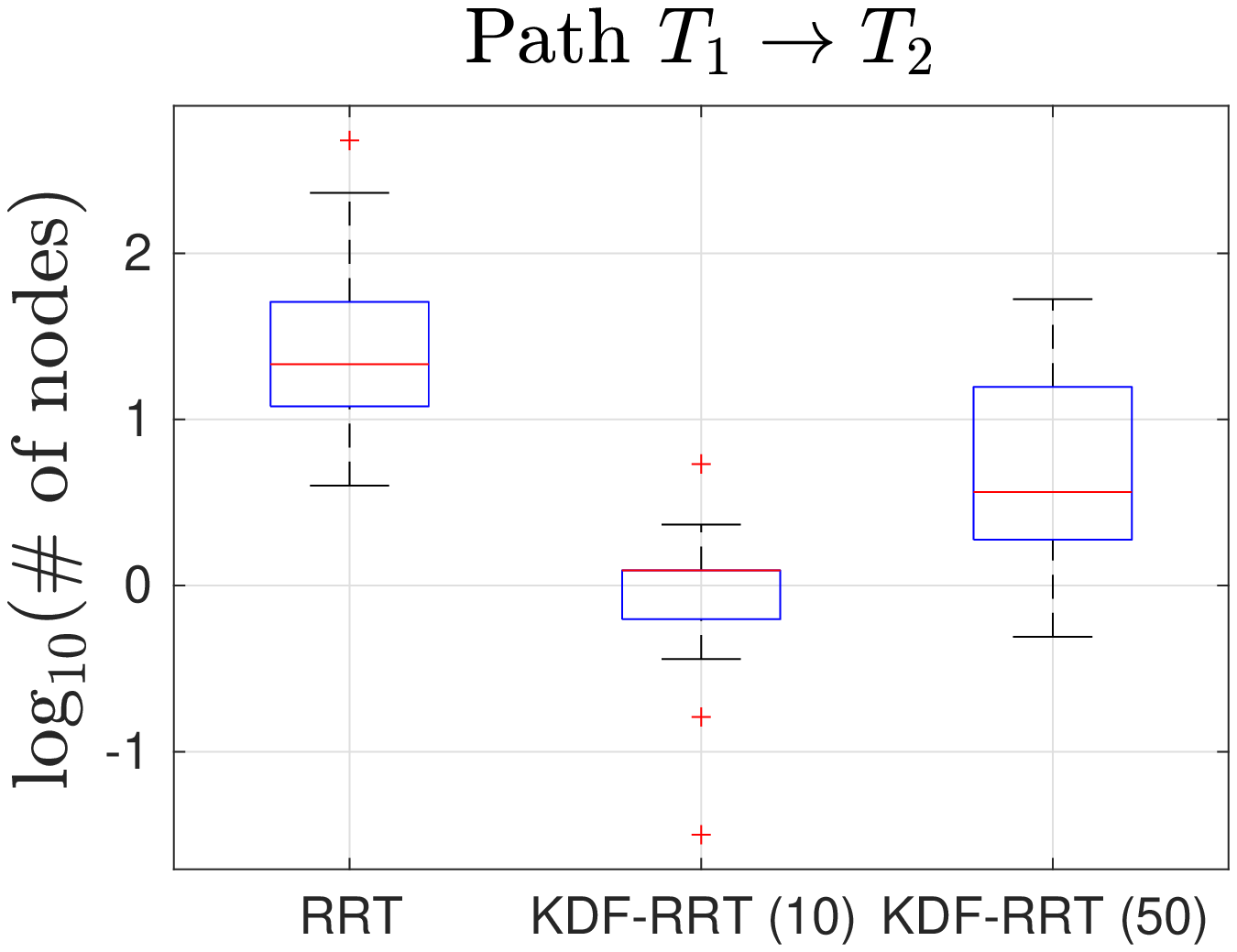}}
	\subcaptionbox{}
	{\includegraphics[width = 0.24\textwidth]{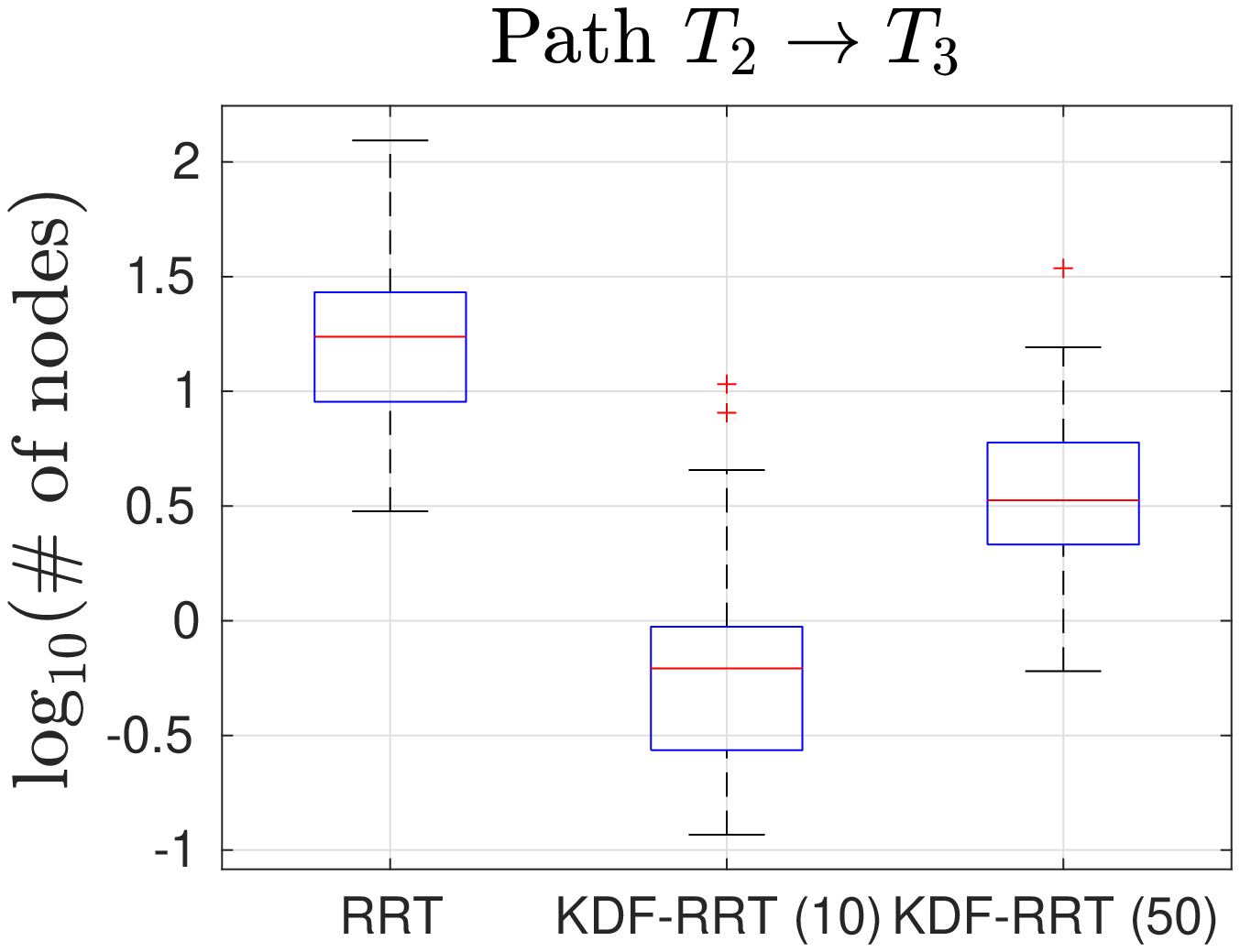}}
	\subcaptionbox{}
	{\includegraphics[width = 0.24\textwidth]{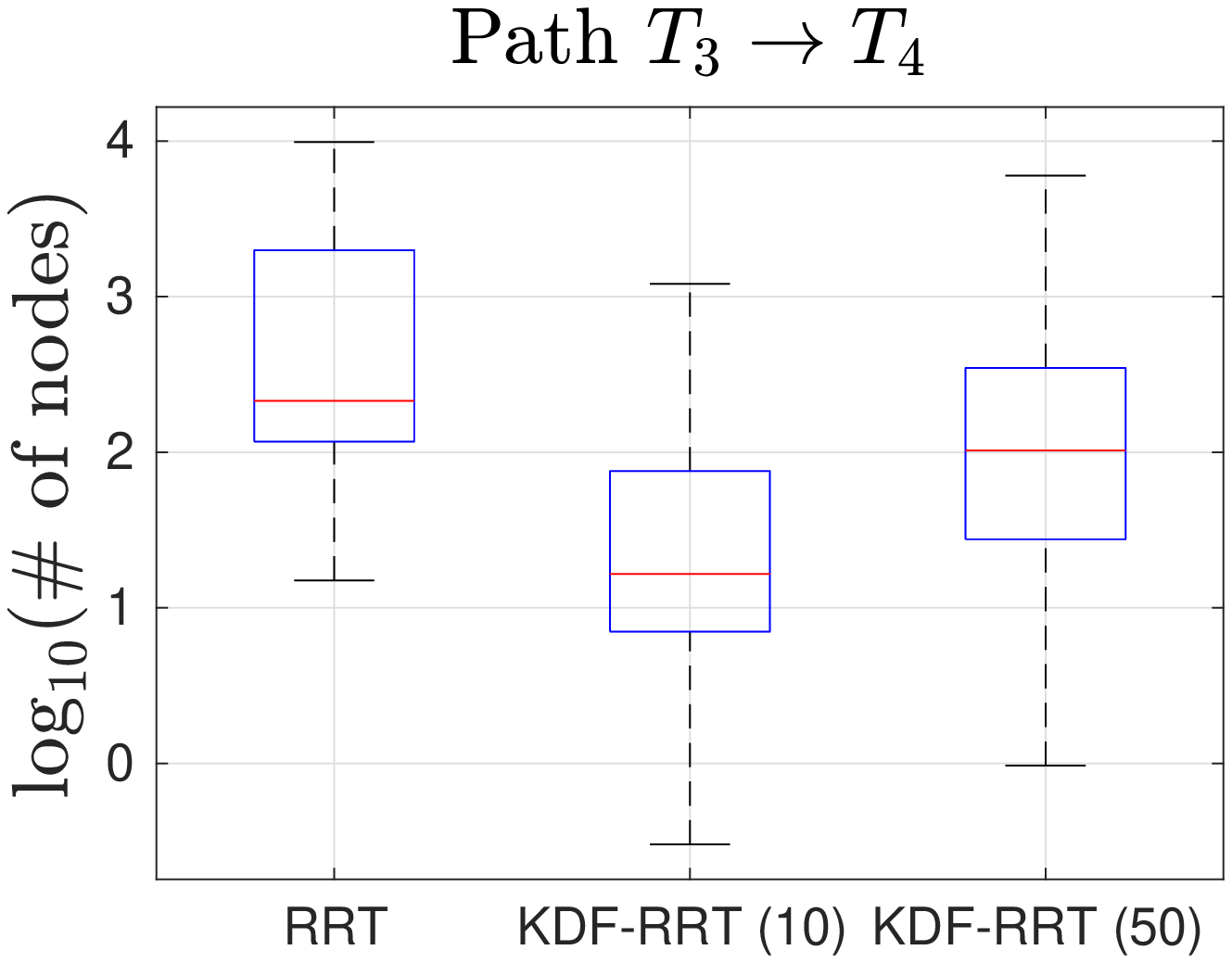}}
	\caption{Box plots showing the number of nodes created in the three algorithms in logarithmic scale for the four paths; '+' indicate the outliers. } \label{fig:basic_nodes}
\end{figure}

\begin{figure}[t]
	\centering
	\subcaptionbox{}
	{\includegraphics[width = 0.24\textwidth]{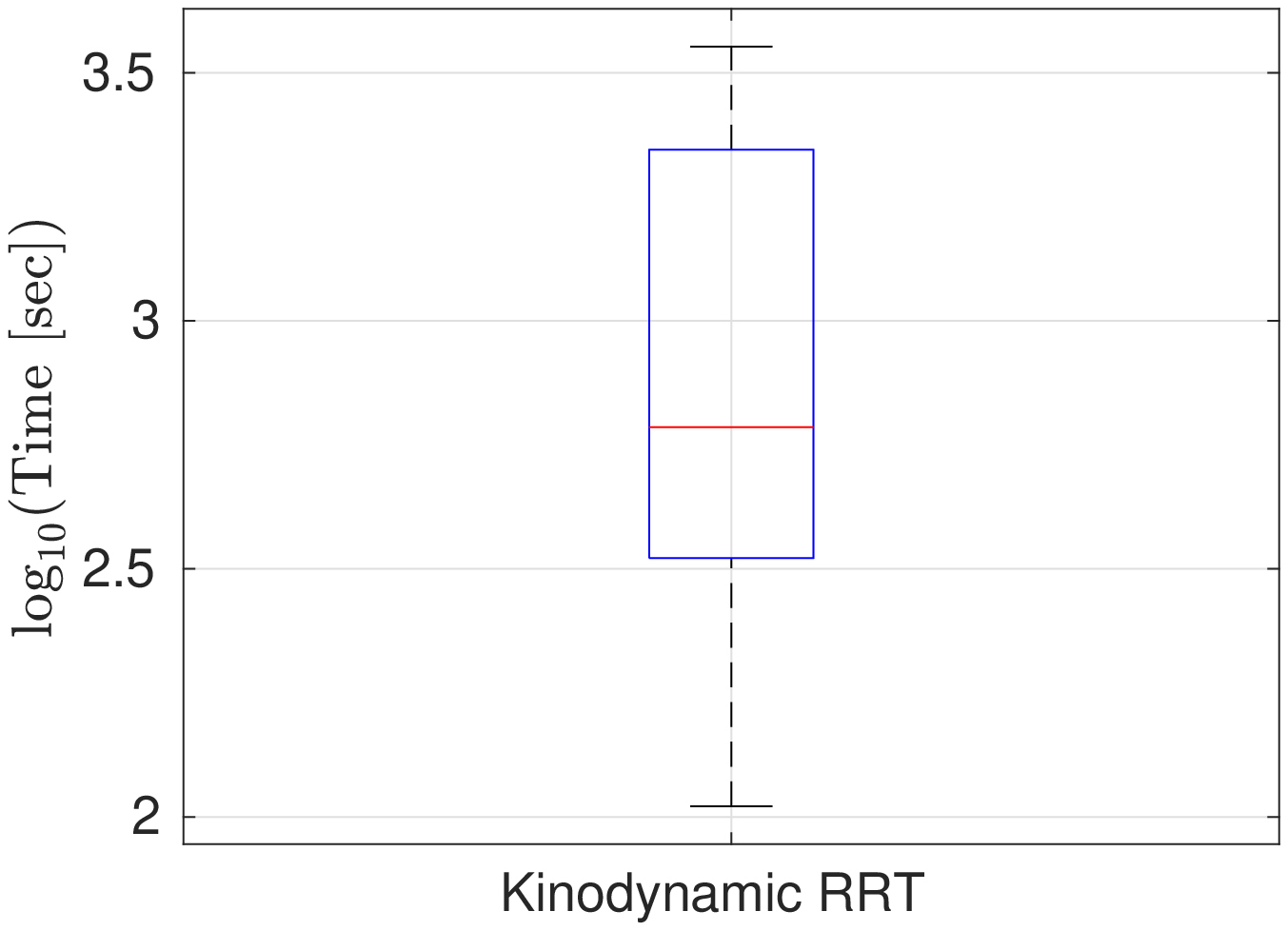}}
	\subcaptionbox{}
	{\includegraphics[width = 0.24\textwidth]{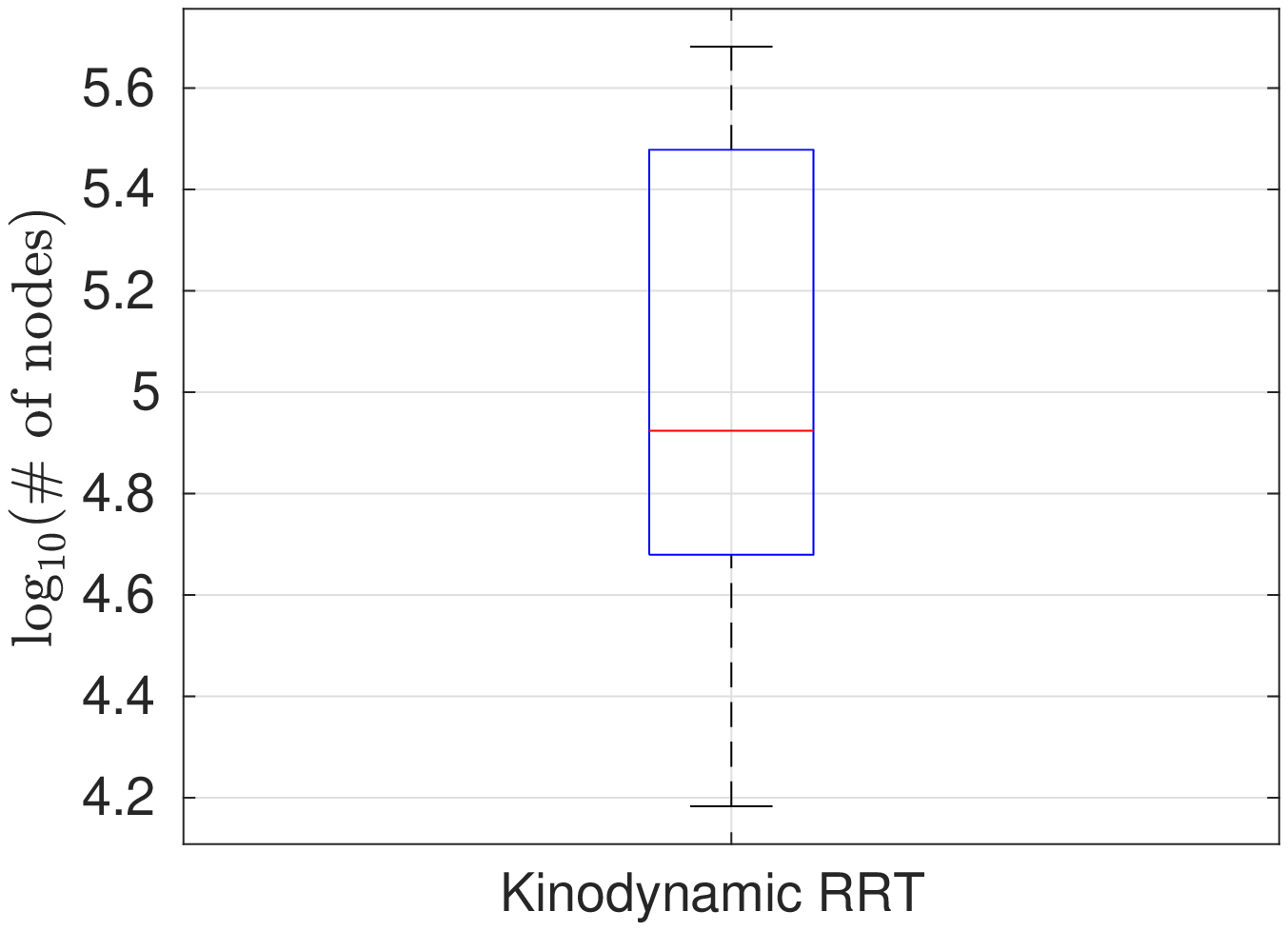}}
	\caption{Box plots showing the execution time (a) and number of nodes (b) created for the kinodynamic RRT in logarithmic scale (for the first two joints). } \label{fig:kinod_time_nodes}
\end{figure}

Next, we illustrate the motion of the robot through the four target points via the control design of Section \ref{sec:ppc control}. 
For each sub-path $T_i \to T_{i+1}$, with $T_0=q_0$ we fit desired trajectories {$q^{\mathfrak{r},i}_{\textup{d}_1}$, $q^{\mathfrak{t},i}_{\textup{d}_j}$, $j\in\{1,\dots,5\}$, with time duration $t_f^i = 11$ seconds,}
as depicted in Fig. \ref{fig:qd_ur5}, where the extra superscript stands for the path $i\in\{0,\dots,3\}$.
For safe tracking, we choose the exponentially decaying functions $\rho^{\mathfrak{t},i}_j(t) = 0.05\exp(-0.01(t-t_{s})) + 0.1 \in [0.1,0.15]$ (rad), for all $j\in\{1,\dots,5\}$, and $\rho^{\mathfrak{r},i}_1(t) = 0.005\exp(-0.01(t-t_{\mathsf{p}_i})) + 0.005 \in [0.005,0.01]$, (implying hence $\bar{\rho} = 0.01[1,15,15,15,15,15]$ as mentioned before), as well as $\rho^i_{2_j}(t) = 2\max\{\max_{j\in\{1,\dots,6\}}\{|e_{2_j}(t_{\mathsf{p}_i})|\}, 0.25\}$, and $\{t_{\mathsf{p}_0},t_{\mathsf{p}_1},t_{\mathsf{p}_2},t_{\mathsf{p}_3}\} \coloneqq\{0,11,22,33,44\}$ are the starting times of the four paths.

\begin{figure}[t]
	\centering
	\subcaptionbox{}
	{\includegraphics[width = 0.24\textwidth]{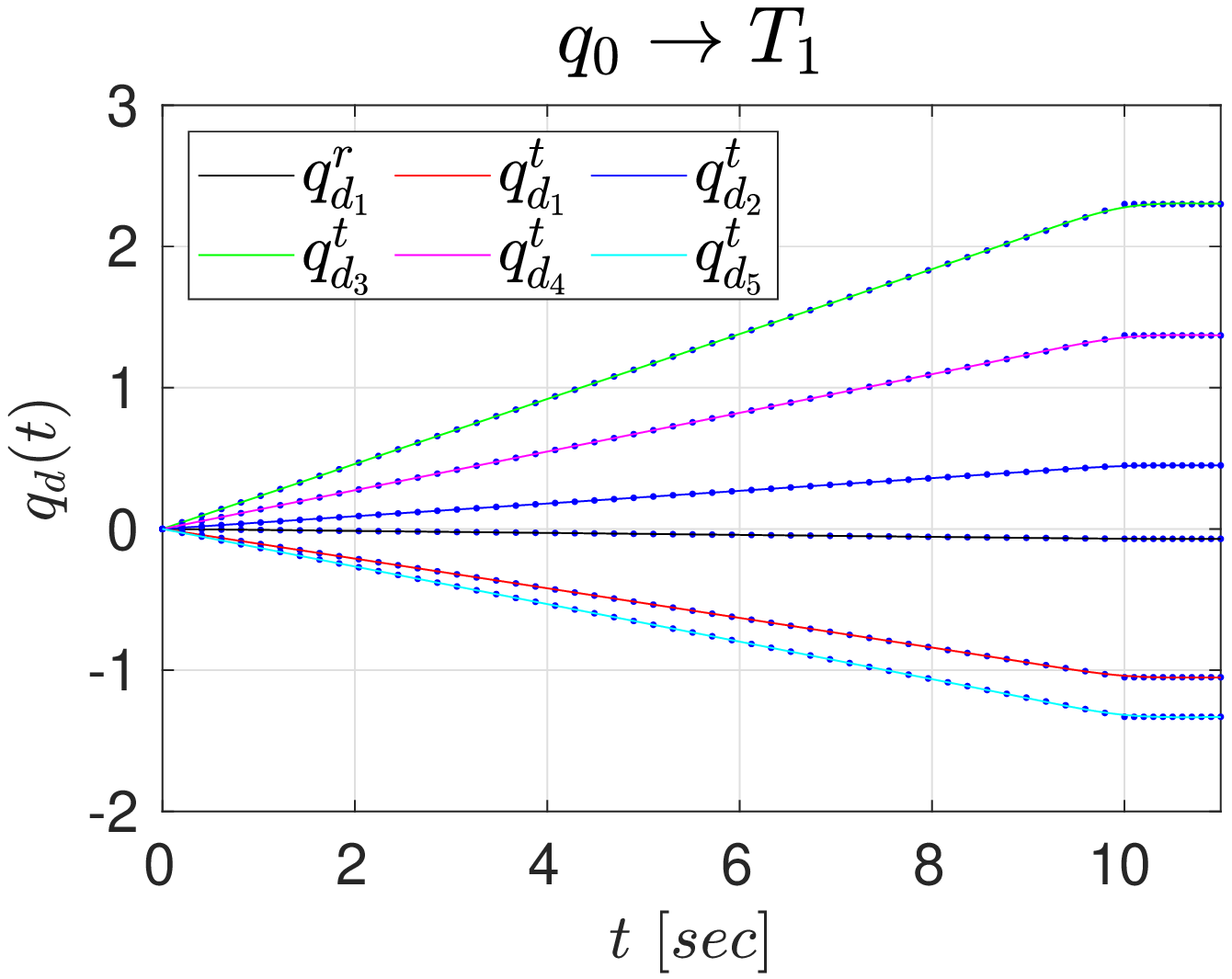}}
	\subcaptionbox{}
	{\includegraphics[width = 0.24\textwidth]{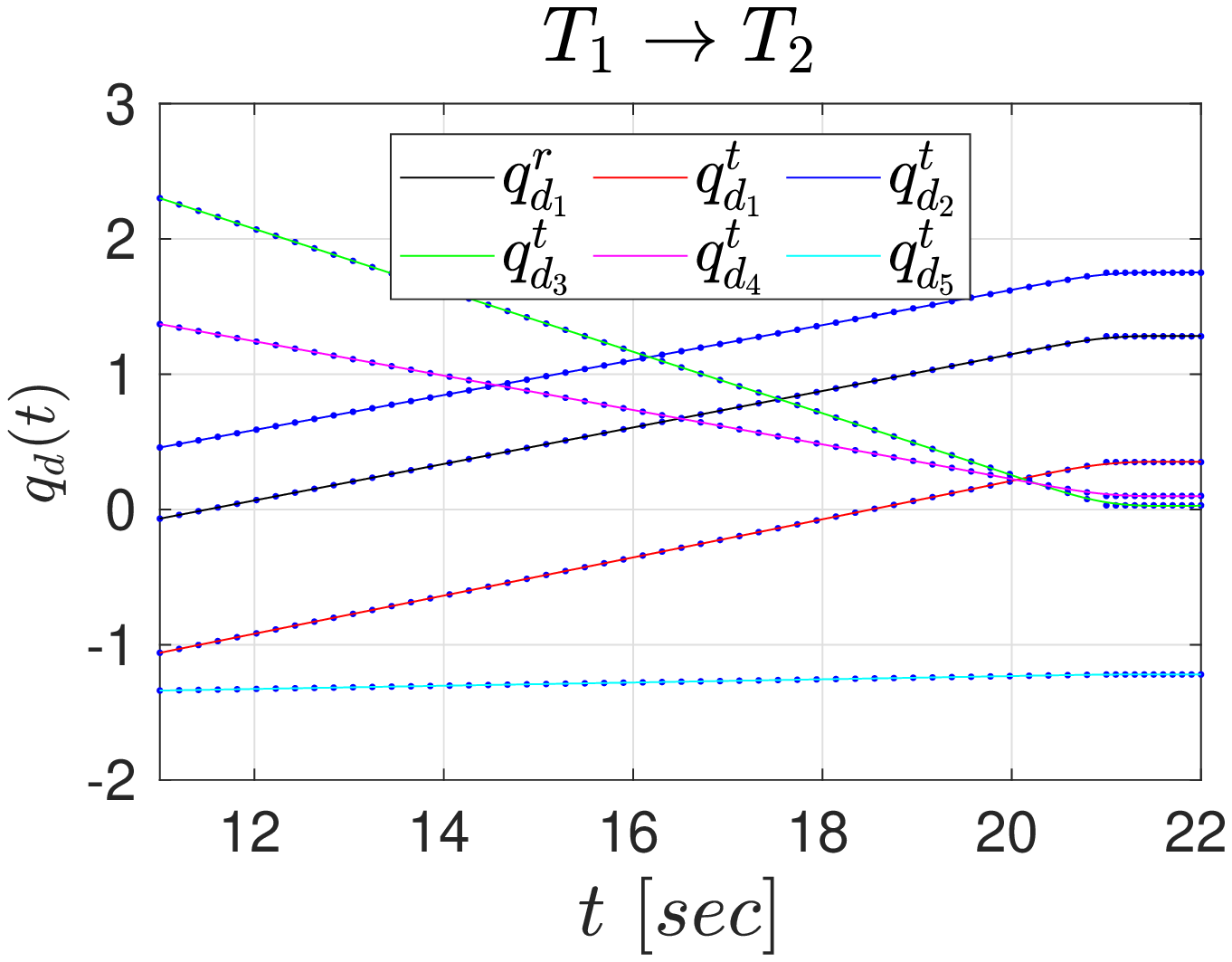}}
	\subcaptionbox{}
	{\includegraphics[width = 0.24\textwidth]{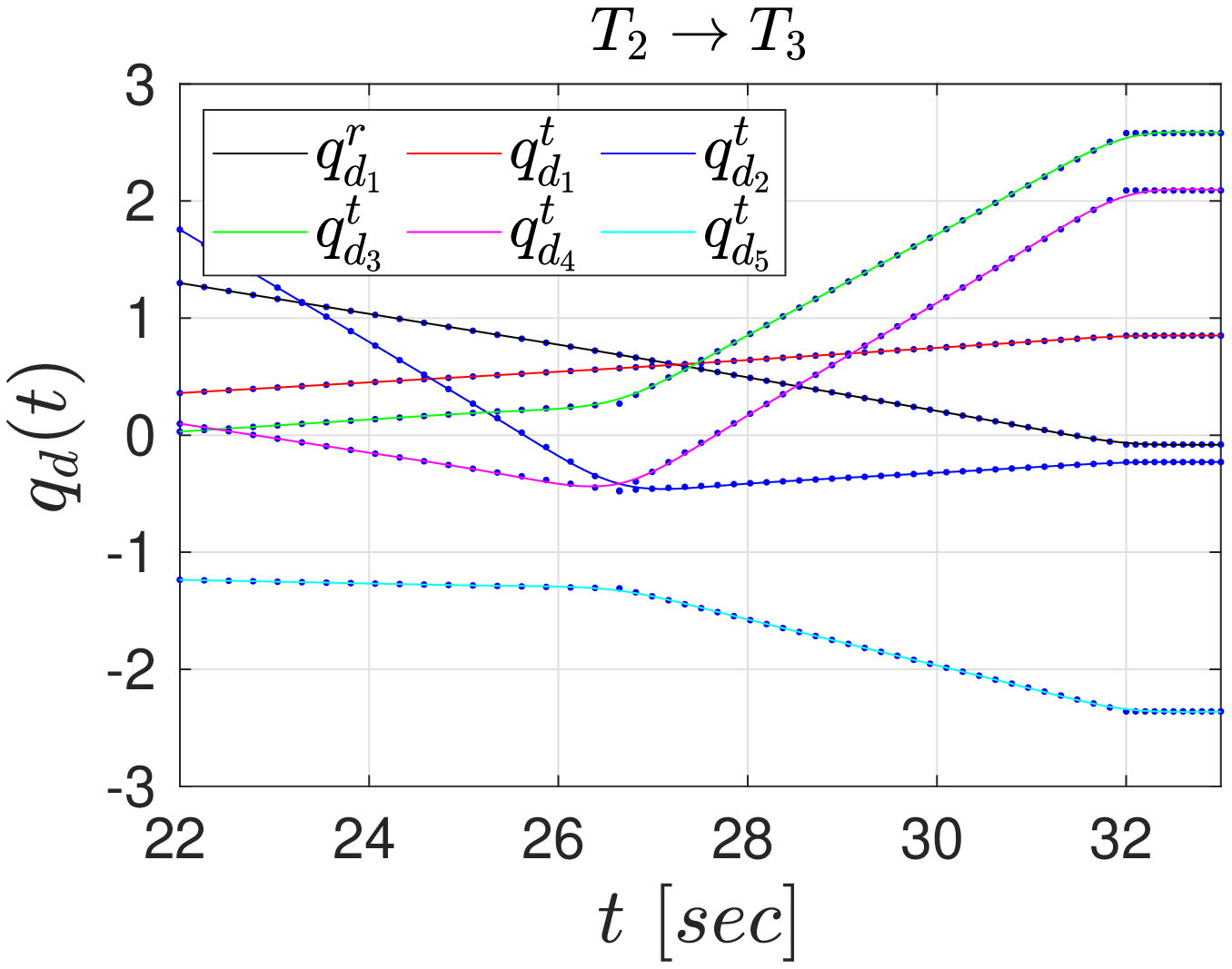}}
	\subcaptionbox{}
	{\includegraphics[width = 0.24\textwidth]{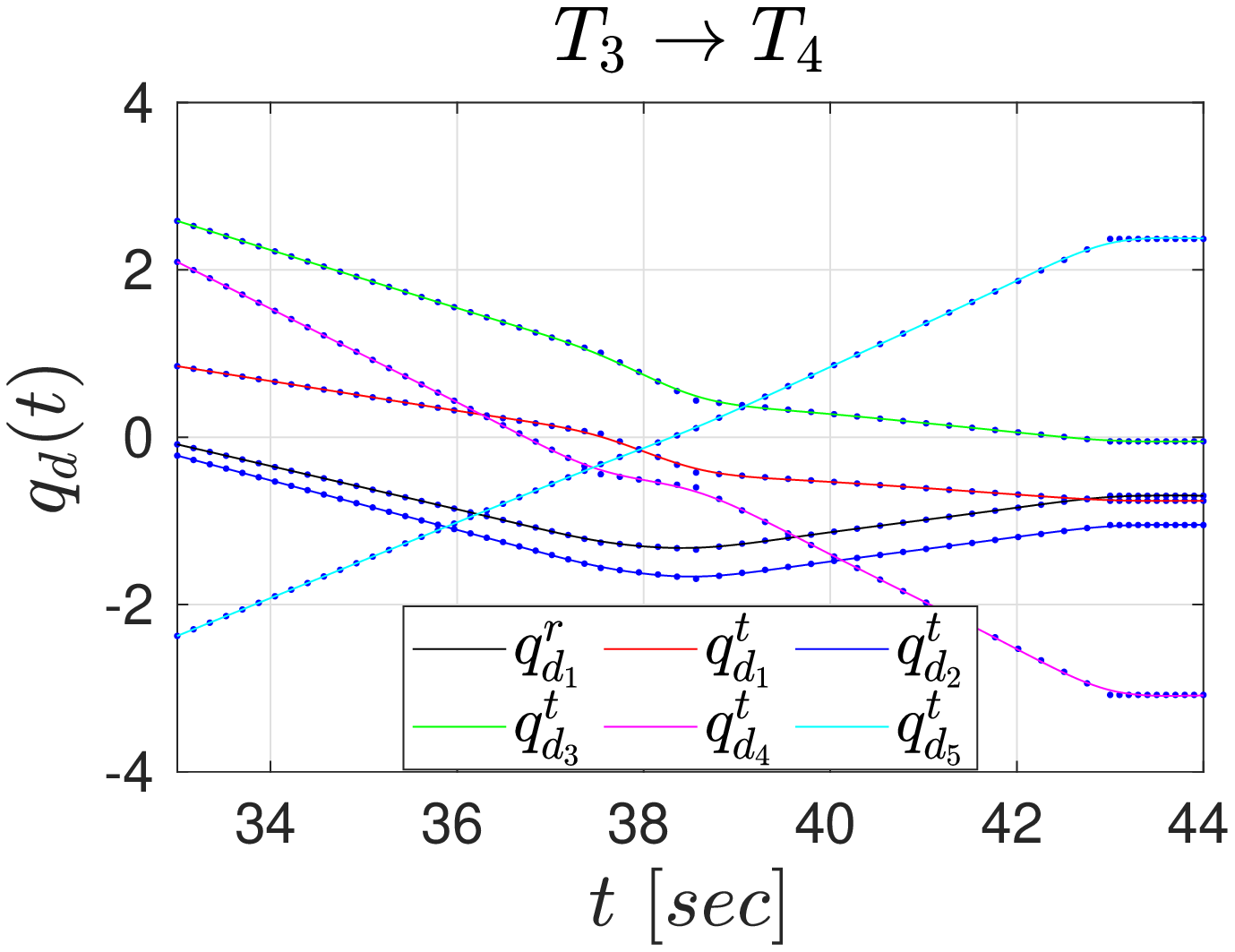}}
	\caption{{The output paths and the respective time-varying trajectories for the four paths. }} \label{fig:qd_ur5}
\end{figure}

\begin{figure}[!ht]
	\centering
	\includegraphics[width=.525\textwidth]{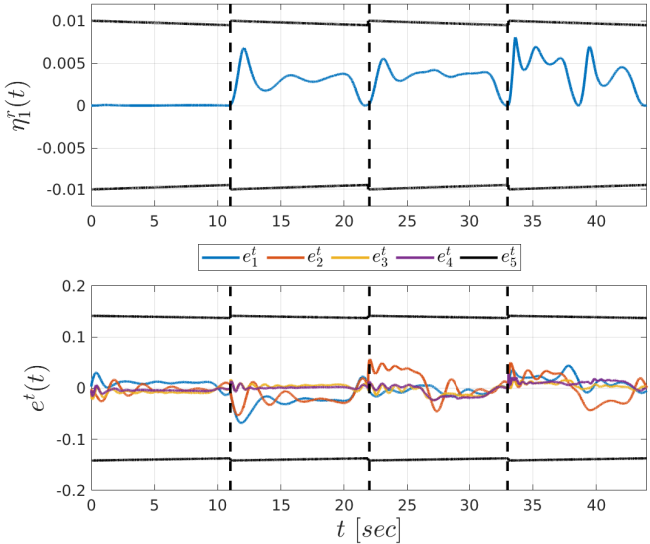}
	\caption{{Top: the evolution of the error $\eta_1^\mathfrak{r}(t)$ (in $\cos(\text{rad})$), along with the respective funnel $\rho^\mathfrak{r}_1(t)$, shown in black, for the four paths. Bottom: the evolution of the errors $e_j^\mathfrak{t}(t)$ (in rad), along with the respective funnel $\rho^\mathfrak{t}_j(t)$, shown in black, for the four paths.}}
	\label{fig:eq_ur5}
\end{figure}

The control gains were chosen as $K^\mathfrak{t} = \text{diag}(1,1,1,5,5,5)$ and $K_2 = 0.1I_6$.

The results of the experiment are depicted in Figs. \ref{fig:eq_ur5}-\ref{fig:D_and_u_ur5}. In particular, Fig. \ref{fig:eq_ur5} depicts the evolution of the errors $\eta^\mathfrak{r}_1(t)$ $e^\mathfrak{t}_j(t)$ (top and bottom in $\cos(\text{rad})$ and rad, respectively), which always satisfy the funnels defined by the respective performance functions $\rho^\mathfrak{r}_1(t)$, $\rho^\mathfrak{t}_j(t)$, $j\in\{1,\dots,5\}$. Similarly, Fig. \ref{fig:e2_ur5} depicts the evolution of the errors $e_{2_j}(t)$ (in rad/seconds), evolving inside the funnel defined by $\rho_{2_j}(t)$, $j\in\{1,\dots,6\}$. Finally, Fig. \ref{fig:D_and_u_ur5} illustrates the minimum distance (in meters) of the UR5 from the obstacles in the environment (top), which is always positive and verifies thus the safety of the framework, and the evolution of the control inputs $u(t)=[u_1,\dots,u_6]^\top$ (in Newton $\cdot$ meters) for the six joint actuators (bottom).

\begin{figure}
	\centering
	\includegraphics[width=.5\textwidth]{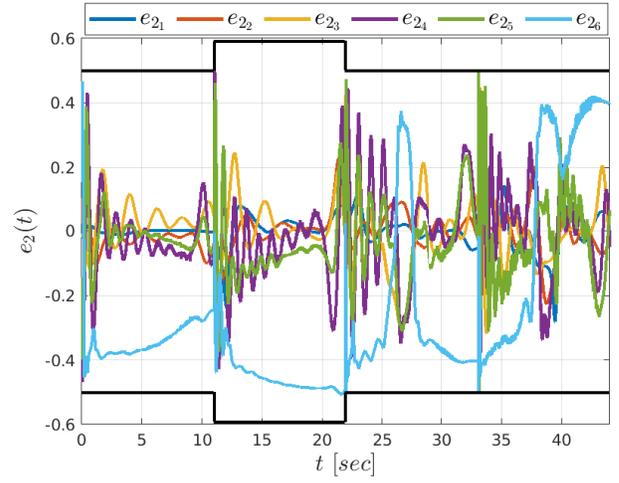}
	\caption{{The evolution of the velocity errors $e_{2_j}(t)$ (in rad/seconds), along with the respective funnels $\rho_{2_j}(t)$ (in black), $j\in\{1,\dots,6\}$, for the four paths. }}
	\label{fig:e2_ur5}
\end{figure}

\begin{figure}
	\centering
	\includegraphics[width=.5\textwidth]{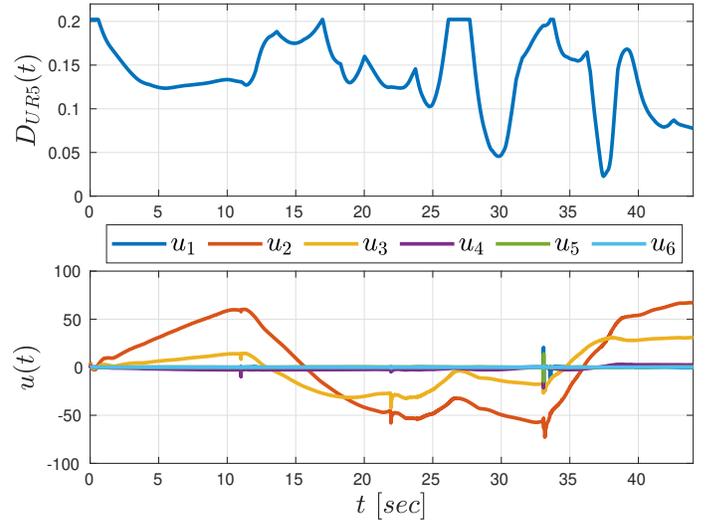}
	\caption{{Top: the distance $D_{UAV}(t)$ (in meters) of the robot from the obstacles for the four paths. Bottom: the evolution of the control inputs $u(t) = [u_1(t), \dots, u_6(t)]^\top$ (in Newton $\cdot$ meters) for the four paths.}}
	\label{fig:D_and_u_ur5}
\end{figure}

\subsection{Hardware Experiments}

\begin{figure}
	\centering
	\includegraphics[width=.4\textwidth]{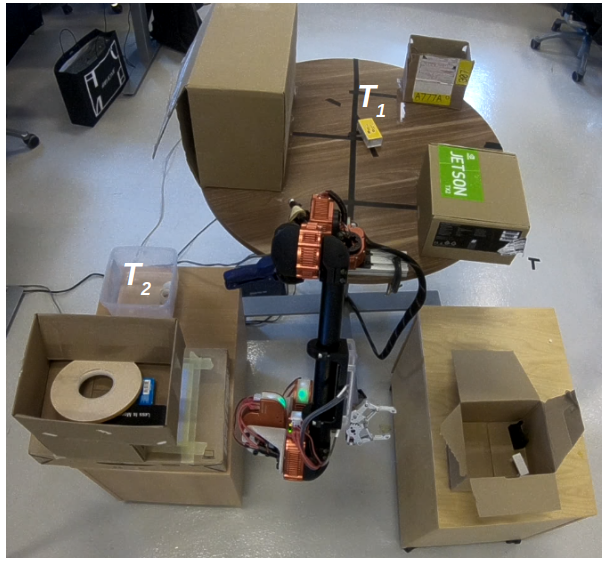}
	\caption{The initial configuration of the HEBI robot in an obstacle-cluttered environment.}
	\label{fig:exp_init}
\end{figure}

\begin{figure}[!ht]
	\centering
	\subcaptionbox{}
	{\includegraphics[width = 0.24\textwidth]{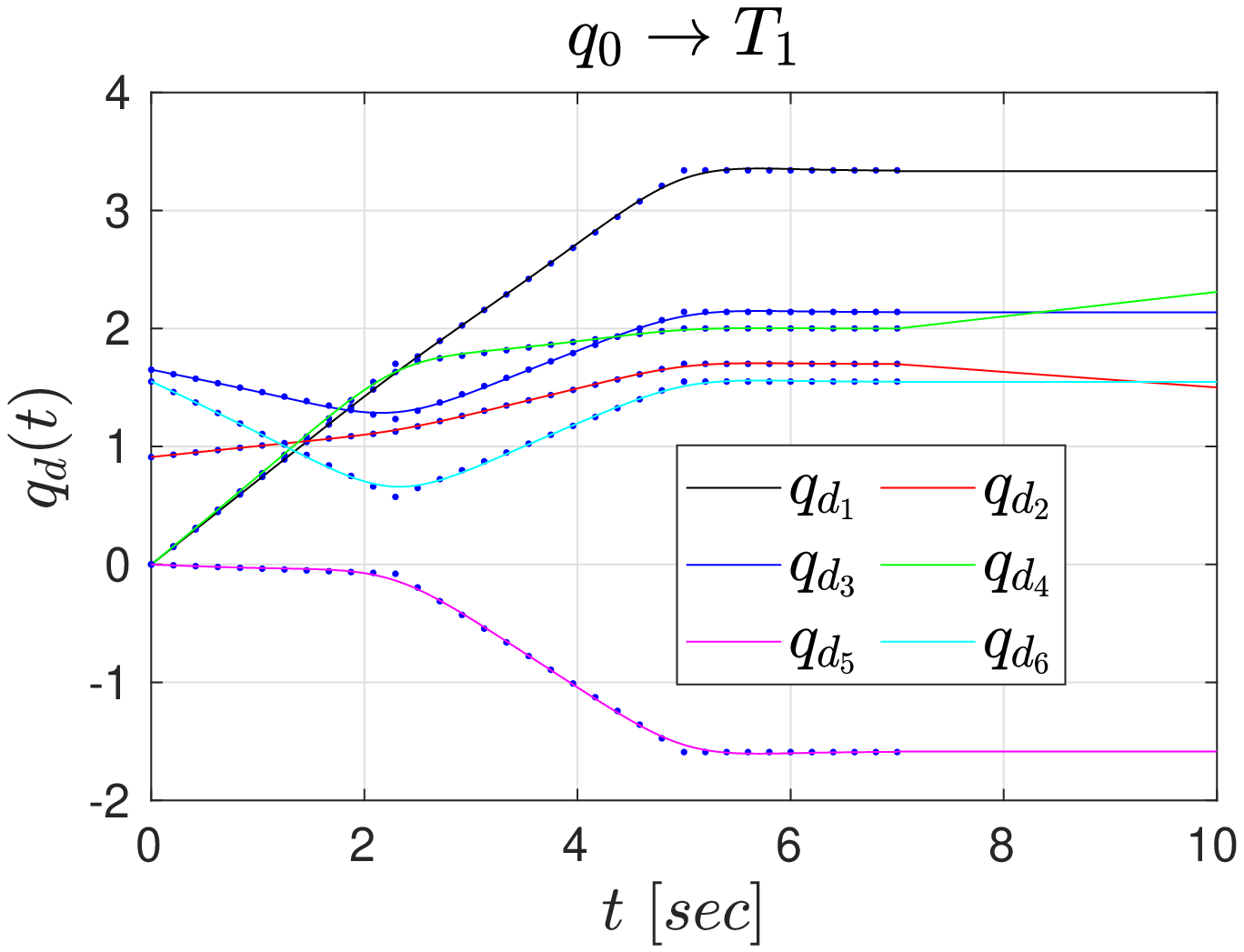}}
	\subcaptionbox{}
	{\includegraphics[width = 0.24\textwidth]{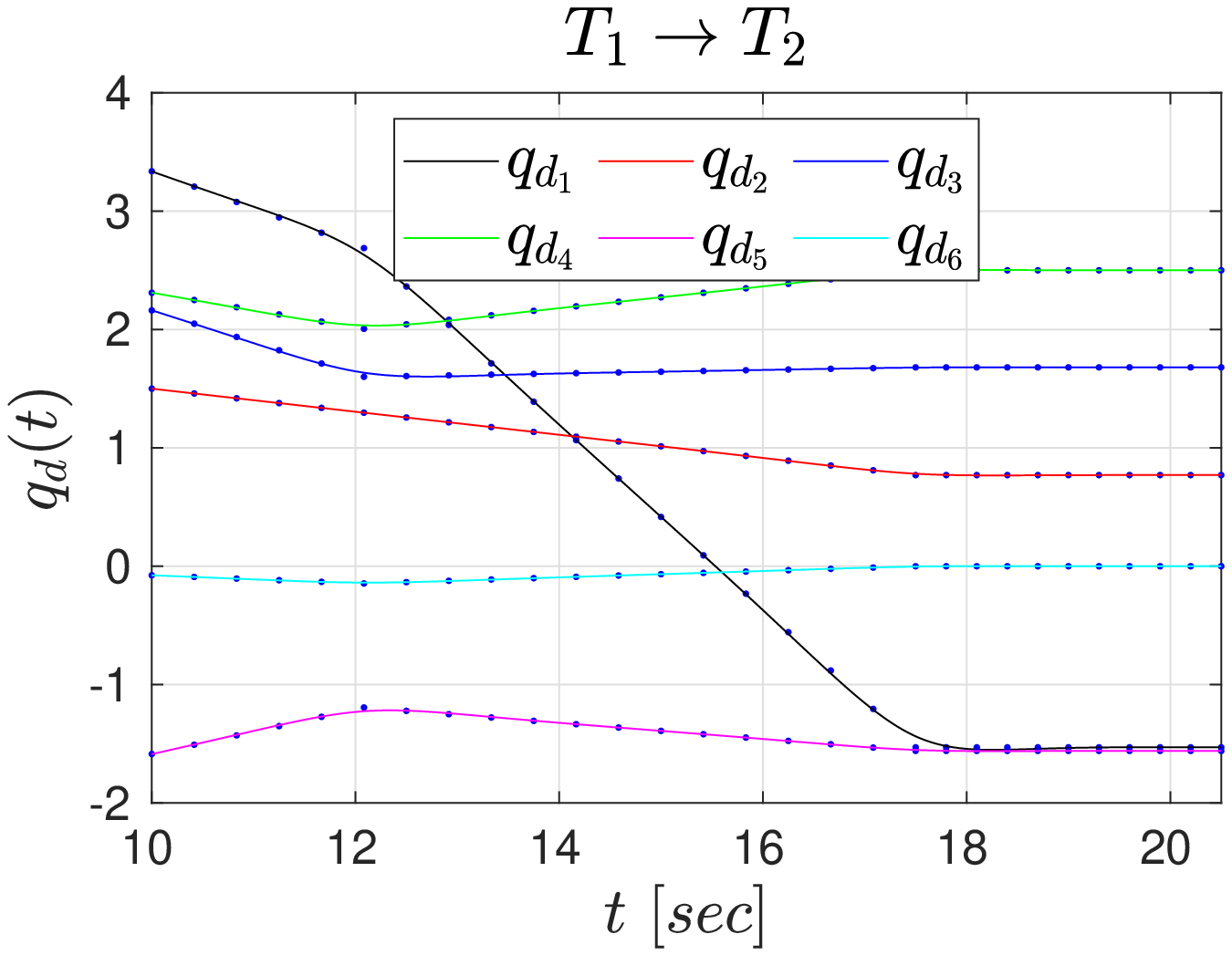}}
	\caption{{The output sequence of points and the respective trajectories for the two paths of the hardware experiment. }} \label{fig:qd_exp}
\end{figure}

This section is devoted to the experimental validation of the proposed framework using a $6$DOF manipulator from HEBI-Robotics subject to $2$nd-order dynamics (as in \eqref{eq:sim dynamics}), which consists of $6$ rotational joints (see Fig. \ref{fig:exp_init}) operating in $[-\pi,\pi]$, resulting in $q_1 = [q^\mathfrak{t}_1,\dots,q^\mathfrak{t}_6]^\top$. 

We consider that the robot has to perform a pick-and-place task, where it has to pick an object from $T_1$ and deliver it in $T_2$ (see Fig. \ref{fig:exp_init}). 
We use the KDF-RRT algorithm, with $\bar{\rho} = [0.15,0.1,0.1,0.2,0.2,0.2]$ rad, to generate two paths: from the initial configuration to a point close to $T_1$ (to avoid collision with the object), and from $T_1$ to $T_2$. Regarding the collision checking in $\bar{\mathcal{A}}_\text{free}(\bar{\rho})$, we check 10 samples around each point of the resulting path for collision.  We next fit smooth trajectories for the two paths $q^{\mathfrak{r},1}_\textup{d}(t)$, $q^{\mathfrak{r},2}_\textup{d}(t)$, with duration of $t_{f_1} = 7$ and $t_{f_2} = 11$ seconds, respectively, as shown in Fig. \ref{fig:qd_exp}.
For grasping the object, we use a simple linear interpolation to create an additional time-varying trajectory segment to $T_1$ with duration of 3 seconds (see Fig. \ref{fig:qd_exp}(a) for $t\in[7,10]$).

For the execution of the control algorithm, we choose constant funnel functions 
$\rho^{\mathfrak{t},i} = [\rho^{\mathfrak{t},i}_{1},\dots,\rho^{\mathfrak{t},i}_{6}]^\top = \bar{\rho} = [0.15,0.1,0.1,0.2,0.2,0.2]$ rad, for the two paths $i\in\{1,2\}$.
Moreover, we choose $\rho_{2_j}  = 15$ for all $j\in\{1,\dots,6\}$, and the control gains as $K^\mathfrak{t} = \text{diag}(1.25,1.5,1,2,1,1)$, $K_2 = \text{diag}(250, 200, 150, 50, 20, 10)$. 

\begin{figure}[t]
	\centering
	\includegraphics[width = 0.55\textwidth]{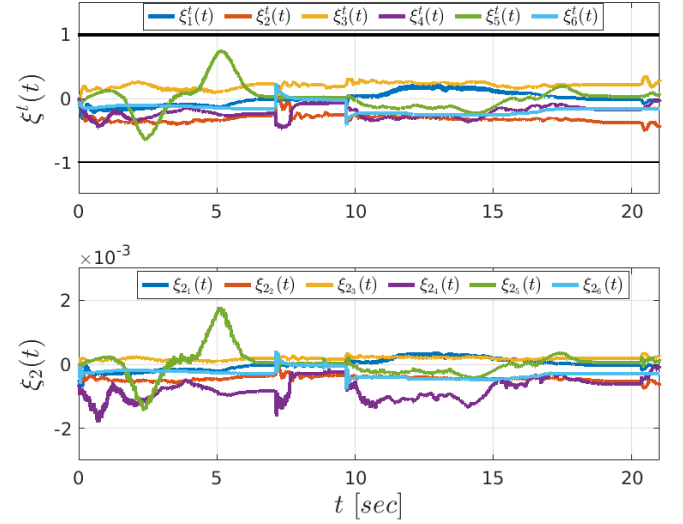}
	\caption{{The evolution of the normalized errors $\xi_j^\mathfrak{t}$ (top) and $\xi_{2_j}(t)$, for $j\in\{1,\dots,6\}$, of the hardware experiment.} } \label{fig:ksi_exp}
\end{figure}

\begin{figure}[!ht]
	\centering
	\includegraphics[width = 0.4\textwidth]{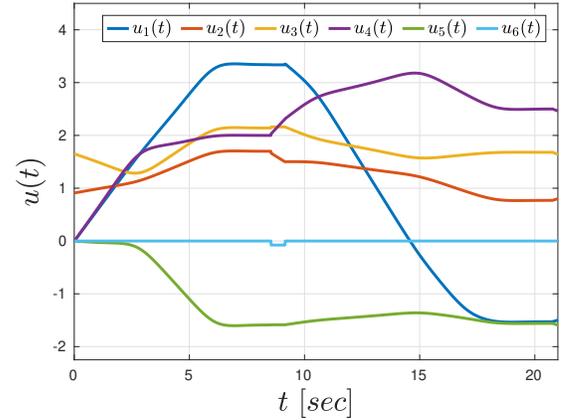}
	\caption{{The evolution of the control inputs $u(t) = [u_1(t),\dots,u_6(t)]^\top$ of the hardware experiment.} } \label{fig:u_exp}
\end{figure}

The results of the experiment are depicted in Figs \ref{fig:ksi_exp} and \ref{fig:u_exp}. In particular, Fig. \ref{fig:qd_exp} depicts the normalized signals $\xi^\mathfrak{t}(t) = [\xi^\mathfrak{t}_1,\dots,\xi^\mathfrak{t}_6]^\top$ and $\xi_2(t) = [\xi_{2_1},\dots,\xi_{2_6}]^\top$ (top and bottom, respectively) for $t \in [0,21]$ seconds. It can be observed that for the entire motion, it holds that $\xi^\mathfrak{t}_j \in (-1,1)$, $\xi_{2_j}(t) \in (-1,1)$, for all $j\in\{1,\dots,6\}$, which implies that $-\rho^\mathfrak{t}_j < e^\mathfrak{t}_j(t) = q_1(t) - q_\text{d}(t) < \rho^\mathfrak{t}_j$,  $-\rho_{2_j} < e_2(t) = q_{2_j}(t) - \alpha_{1_j}(t) < \rho^\mathfrak{t}_j$, for all $j\in\{1,\dots,6\}$ and $t \in [0,21]$ seconds, with $\alpha_1$ as in \eqref{eq:alpha}. Therefore, we conclude that the robot tracks the path output by the KDF-RRT algorithm within the prescribed funnel, avoiding thus collisions. Snapshots of the path execution are given in Fig. \ref{fig:snapshots_exp} at two time instances, namely $t = 10$, and $t=19$ seconds.

\begin{figure*}[!ht]
	\centering
	\subcaptionbox{}
	{\includegraphics[width = 0.4\textwidth]{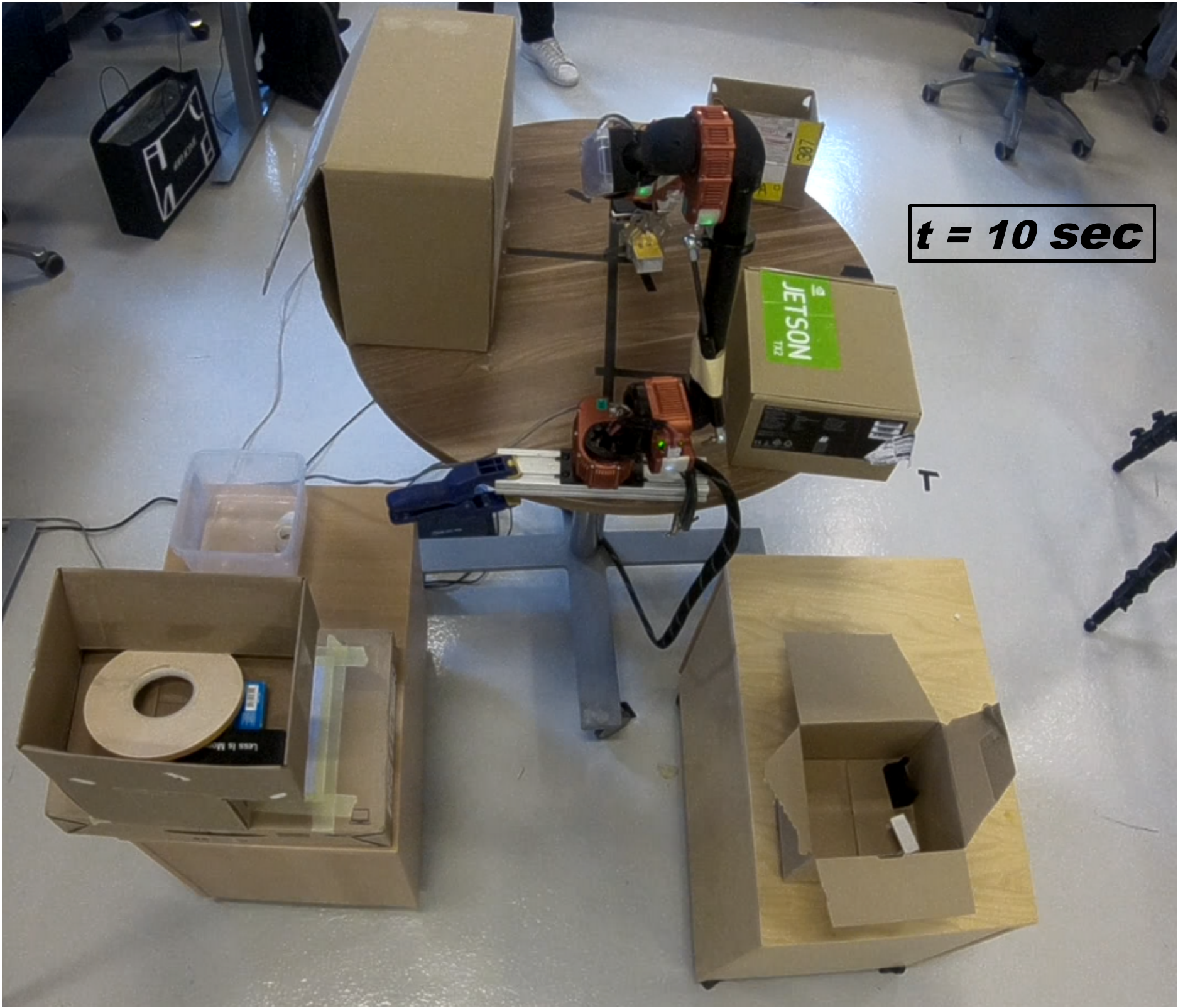}}
	\subcaptionbox{}
	{\includegraphics[width = 0.4\textwidth]{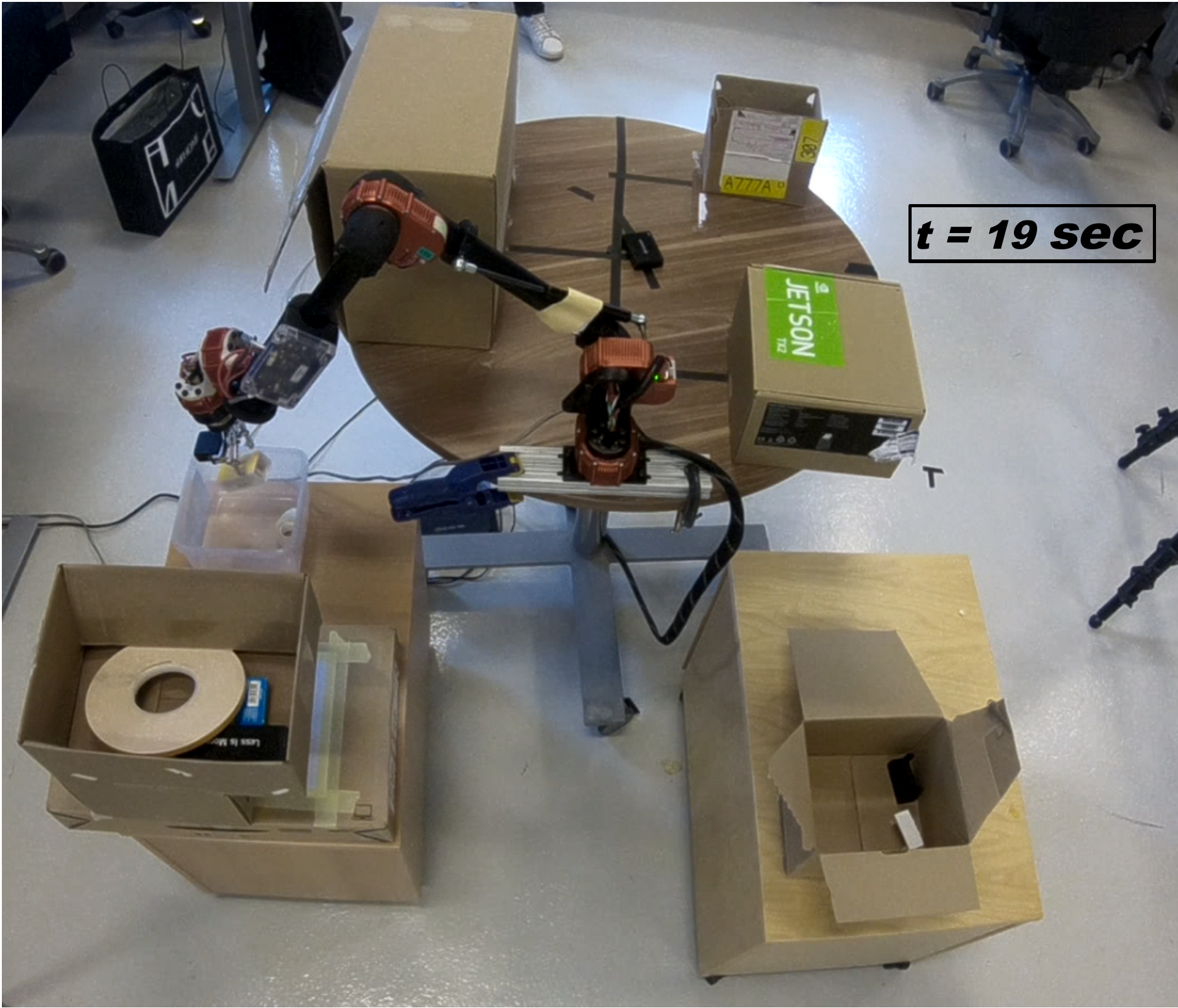}}
	\caption{{Snapshots of the hardware experiment at $t=10$ (a), and $t=19$ (b) seconds. }} \label{fig:snapshots_exp}
\end{figure*}

We further test the robustness of the proposed control scheme against \textit{adversarial} disturbances. In particular, we disturb the manipulator using a rod three times during the execution of the aforementioned trajectory (see Fig. \ref{fig:hebi dist}). In order to prevent the control scheme from having invalid values (see the domain of definition of \eqref{eq:epsilon+r} and \eqref{eq:xi_i,epsilon_i,r_i}), we set $\xi^\mathfrak{t}_j = \max\left\{ \min\left\{1, \frac{e^\mathfrak{t}_j}{\rho^\mathfrak{t}_j} \right\}, -\frac{e^\mathfrak{t}_j}{\rho^\mathfrak{t}t_j} \right\}$, $\xi_{2_j} = \max\left\{ \min\left\{1, \frac{e_{2_j}}{\rho_{2_j}} \right\}, -\frac{e_{2_j}}{\rho_{2_j}} \right\}$ for all $j\in\{1,\dots,6\}$. The evolution of the signals $\xi^\mathfrak{t}(t)$, $\xi_2(t)$ are depicted in Fig. \ref{fig:ksi_exp_dist} for $21$ seconds, with vertical black dashed lines depicting the instants of the disturbance, which affects mostly the first joint of the system; {note from Fig. \ref{fig:ksi_exp_dist} that $\xi^\mathfrak{t}_1$ and $\xi_{1_1}$ are excessively increased with respect to their nominal values shown in Fig. \ref{fig:ksi_exp}, implying a large increase in the respective errors $e^\mathfrak{t}_1$ and $e_{2_1}$.}
Nevertheless, one can conclude that, despite the presence of adversarial disturbances, the system 
manages to successfully recover and complete the derived path.

In order to further evaluate the proposed control algorithm, we compared our results with a standard well-tuned PID controller as well as the parametric adaptive control scheme (PAC) of our previous work \cite{verginis21sampling}. The signals $\xi^\mathfrak{t}$ for these two control schemes are depicted in Fig. \ref{fig:comaprison_exp}. Note that the controllers fail to retain the normalized errors $\xi^\mathfrak{t}_j(t)$ in the interval $(-1,1)$. Although in the particular instance this did not lead to collisions, it jeopardizes the system motion, since it does not comply with the bounds set in the KDF-RRT algorithm.

\begin{figure}[t]
	\centering
	\includegraphics[width = 0.45\textwidth]{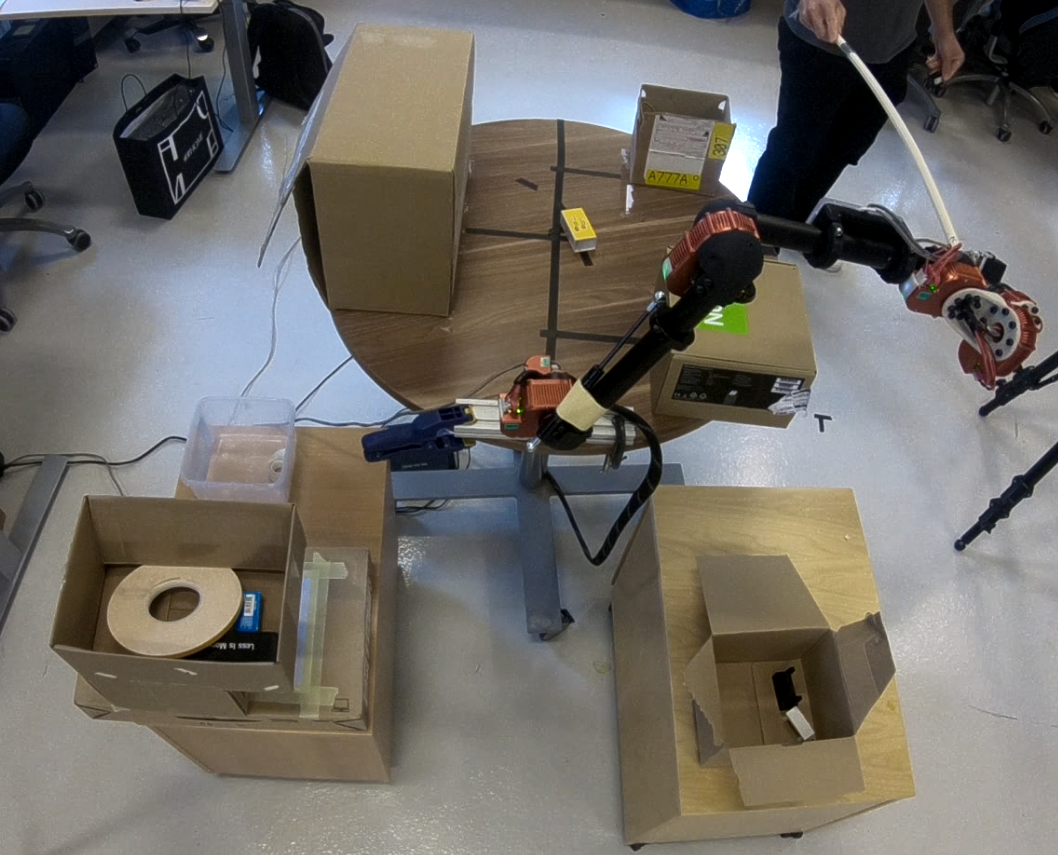}
	\caption{{Application of adversarial disturbances in the hardware experiment.} } \label{fig:hebi dist}
\end{figure}

\begin{figure}[t]
	\centering
	\includegraphics[width = 0.55\textwidth]{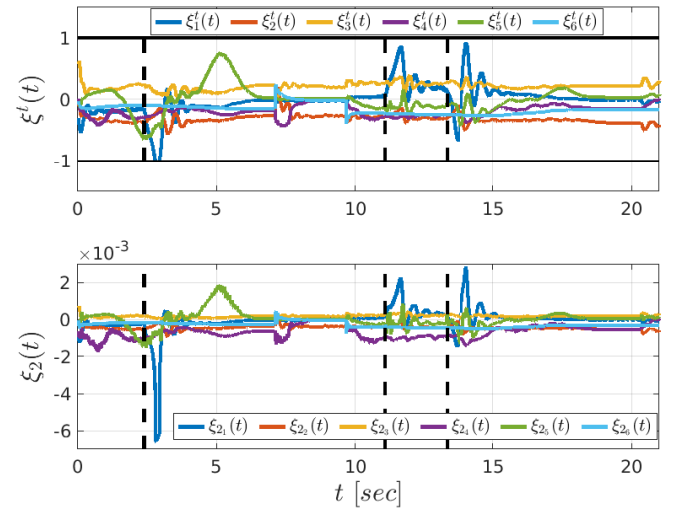}
	\caption{{The evolution of the normalized errors $\xi_j^\mathfrak{t}$ (top) and $\xi_{2_j}(t)$, for $j\in\{1,\dots,6\}$, of the hardware experiment, in the case of adversarial disturbances. The time instants of the disturbance application are shown with vertical dashed lines.} } \label{fig:ksi_exp_dist}
\end{figure}

\begin{figure}[t]
	\centering
	\includegraphics[width = 0.475\textwidth]{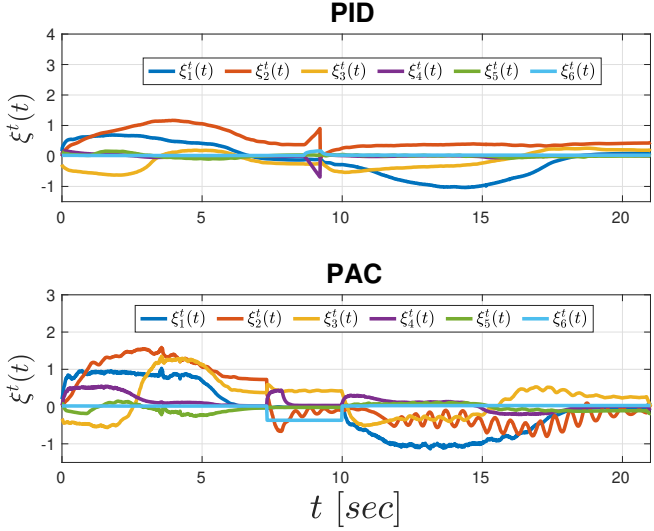}
	\caption{{The evolution of the normalized errors $\xi_j^\mathfrak{t}$ (top) and $\xi_{2_j}(t)$, for $j\in\{1,\dots,6\}$, of the hardware experiment, when using a PID controller (top), and the adaptive control algorithm from \cite{verginis21sampling}}. } \label{fig:comaprison_exp}
\end{figure}

\section{Conclusion}
{We develop KDF, a new framework for solving the kinodynamic motion-planning problem for complex systems with uncertain dynamics. The framework comprises of three modules: first, a family of geometric sampling-based motion planners that produce a path in an extended free space; secondly, a smoothening and time endowment procedure that converts the path into a smooth time-varying trajectory; and finally, a funnel-based feedback control scheme that guarantees safe tracking of the trajectory. Neither of the modules uses any information on the system dynamics. Experimental results demonstrate the effectiveness of the proposed method. Future directions will focus on extending KDF to systems with nonholonomic and underactuated dynamics and taking into account explicit input constraints.
}


%

\appendices
\section{Proof Theorem \ref{th:control theorem}} \label{app:ppc proof}
Consider the non-empty open set 
\begin{align}
	\Omega \coloneqq \big\{& (\bar{q},t) \in \mathbb{T}\times\mathbb{R}^{n(k-1)}\times[t_0,t_0+t_f) : \xi^\mathfrak{t}_j \in (-1,1), \notag \\ 
	& \xi^\mathfrak{r}_\ell \in [0,1), \xi_i \in (-1,1)^n, \forall j\in\{1,\dots,n_{tr}\}, \notag \\
	& \ell\in\{1,\dots,n_r \}, i\in\{2,\dots,k\} \big\},
\end{align}
where we implicitly write the $\xi$ variables as function of $\bar{q}$ and time $t$. The constraints \eqref{eq:funnel constraints} imply that $(\bar{q}(t_0),t_0) \in \Omega$. By substituting the control law \eqref{eq:control law} in the dynamics \eqref{eq:dynamics}, we obtain a closed-loop system $\dot{\bar{q}} = f_\textup{cl}(\bar{q},t)$ and one can verify, based on Assumption \ref{ass:dynamics},  that $f_\textup{cl}$ is continuously differentiable in $\bar{q}$ and continuous $t$ on $\Omega$. Therefore, the conditions of \cite[Theorem 2.1.3]{bressan2007introduction} are satisfied and we conclude the existence of a \textit{maximal} solution $\bar{q}(t)$ for $t\in I_t \coloneqq [t_0,t_0+t_{\max})$, with $t_{\max} > 0$, satisfying $(\bar{q}(t),t) \in \Omega$ for all $t\in I_t$. 

Hence, for $t\in I_t$, the transformed errors $\varepsilon^\mathfrak{t}_j$, $\varepsilon^\mathfrak{r}_\ell$, $\varepsilon_i$ are well defined. We proceed inductively with the following steps.

\textit{Step} 1. Consider the positive definite and radially unbounded candidate Lyapunov function 
\begin{equation}
	V_1 \coloneqq \frac{1}{2}(\varepsilon^\mathfrak{t})^\top K^\mathfrak{t} \varepsilon^\mathfrak{t} + \sum_{\ell\in\mathcal{L}_r} k^\mathfrak{r}_\ell \varepsilon^\mathfrak{r}_\ell,
\end{equation}
where $\mathcal{L}_r\coloneqq \{1,\dots,n_r\}$, and $K^\mathfrak{t}$, $k^\mathfrak{r}_\ell$ are gain-related terms introduced in \eqref{eq:alpha}. Let also the first equation of \eqref{eq:dynamics} be partitioned as 
\begin{equation*}
	\begin{bmatrix}
	\dot{q}^\mathfrak{t} \\
	\dot{q}^\mathfrak{r}
	\end{bmatrix} 
	= 
	\begin{bmatrix}
	f^\mathfrak{t}(q_1,t) \\ 
	f^\mathfrak{r}(q_1,t)
	\end{bmatrix} 
    + \begin{bmatrix}
    g_{11}(q_1,t) & g_{12}(q_1,t) \\ g_{13}(q_1,t) & g_{14}(q_1,t)
    \end{bmatrix} q_2.
\end{equation*} 
  Differentiating $V_1$ yields
\begin{align*}
	\dot{V}_1 =& (\varepsilon^\mathfrak{t})^\top K^\mathfrak{t} \widetilde{r}^\mathfrak{t} (\widetilde{\rho}^\mathfrak{t})^{-1}\big(f^\mathfrak{t}+ [g_{11} \ g_{12} ]q_2 - \dot{q}^\mathfrak{t}_\text{d} - \dot{\widetilde{\rho}}^\mathfrak{t} \xi^\mathfrak{t} \big) + \\
	& 
	(r^\mathfrak{r})^\top K^\mathfrak{r} (\widetilde{\rho}^\mathfrak{r})^{-1} \bigg[ \widetilde{s}^\mathfrak{r}\big( f^\mathfrak{r} + [g_{13} \ g_{14}]q_2 - \dot{q}^\mathfrak{r}_\textup{d} \big)  -\dot{\widetilde{\rho}}^\mathfrak{r} \xi^\mathfrak{r} \bigg],
\end{align*}
where we further define 
$\xi^\mathfrak{t}\coloneqq [\xi^\mathfrak{t}_1,\dots,\xi^\mathfrak{t}_{n_{tr}}]^\top$, $\xi^\mathfrak{t}\coloneqq [\xi^\mathfrak{r}_1,\dots,\xi^\mathfrak{r}_{n_{r}}]^\top$. By using $q_2 = \alpha_1 + e_2$ from \eqref{eq:e_i} and substituting \eqref{eq:alpha}, we obtain after straightforward manipulations 
\begin{align*}
\dot{V}_1 =& - \sigma^\top K \widetilde{R}\widetilde{\rho}^{-1} \widetilde{S}  g_1 \widetilde{S} \widetilde{\rho}^{-1}\widetilde{R}  K \sigma + \\ 
& \sigma^\top K \widetilde{R} \widetilde{\rho}^{-1}  \bigg[\widetilde{S}\big( f_1 +  g_1e_2 - \dot{q}_\textup{d} \big) - \dot{\widetilde{\rho}} \xi \bigg] \\
=:& T_n + T_b,
\end{align*}
where $\sigma \coloneqq [(\varepsilon^\mathfrak{t})^\top, (r^\mathfrak{r})^\top]^\top$, $K \coloneqq \text{blkdiag}\{ K^\mathfrak{t},K^\mathfrak{r} \}$, $\widetilde{R} \coloneqq \text{blkdiag}\{ \widetilde{r}^\mathfrak{t}, I \}$, $\widetilde{S} \coloneqq \text{blkdiag}\{I, \widetilde{s}^\mathfrak{r}\}$, $\widetilde{\rho} \coloneqq \text{blkdiag}\{ \widetilde{\rho}^\mathfrak{t}, \widetilde{\rho}^\mathfrak{r} \}$, and $\xi \coloneqq [(\xi^\mathfrak{t})^\top, (\xi^\mathfrak{r})^\top]^\top$.

Since $T_n$ is a quadratic form, it holds that $T_n = - \frac{1}{2} \sigma^\top K \widetilde{R} \widetilde{\rho}^{-1}\widetilde{S} (g_1 + g_1^\top) \widetilde{S} \widetilde{\rho}^{-1}\widetilde{R} K \sigma$, and in view of Assumption \ref{ass:g pd}, $T_n \leq -\underline{g}\|K \widetilde{R}\widetilde{\rho}^{-1} \widetilde{S} \sigma\|^2$, where $\underline{g} \coloneqq \frac{1}{2}\lambda_{\min}(g_1 + g_1^\top) > 0$. Therefore, $T_n$ becomes
\begin{align*}
	T_n \leq -\underline{g}  \|K^\mathfrak{t} \widetilde{r}^\mathfrak{t} (\widetilde{\rho}^\mathfrak{t})^{-1} \varepsilon^\mathfrak{t}\|^2 - \underline{g} \sum_{\ell \in\mathcal{L}_r} k^\mathfrak{r}_\ell \left(\frac{r^\mathfrak{r}_\ell}{\rho^\mathfrak{r}_\ell} \sin(e^\mathfrak{r}_\ell)\right)^2 
\end{align*}  

From \eqref{eq:xi}, we obtain $1 - \cos(e^\mathfrak{r}_\ell) = \rho^\mathfrak{r}_\ell \xi^\mathfrak{r}_\ell$ implying $\sin^2(e^\mathfrak{r}_\ell) = \rho^\mathfrak{r}_\ell \xi^\mathfrak{r}_\ell(1+\cos(e^\mathfrak{r}_\ell))$, for all $\ell\in\mathcal{L}_r$. By further defining $\underline{k} \coloneqq \underline{g}\lambda_{\min}(K\widetilde{\rho})$, we obtain 
\begin{align*}
	T_n \leq -\underline{k}\|\widetilde{r}^\mathfrak{t}\varepsilon^\mathfrak{t}\|^2 - \underline{k}\sum_{\ell \in\mathcal{L}_r} (r^\mathfrak{r}_\ell)^2\xi^\mathfrak{r}_\ell(1+\cos(e^\mathfrak{r}_\ell)) .
\end{align*}

Note that, for $t\in I_t$, it holds that $\xi^\mathfrak{t}_\ell \in (-1,1)$ and hence $\eta^\mathfrak{r}_\ell(t) = 1 - \cos(e^\mathfrak{r}_\ell(t)) < \rho^\mathfrak{r}_\ell(t) \leq \bar{\rho}^\mathfrak{r}_\ell < 2$, for all $ \ell\in\mathcal{L}_r$ (see \eqref{eq:funnel constraints}). Therefore, it holds that $1 + \cos(e^\mathfrak{r}_\ell) \geq 2 - \bar{\rho}^\mathfrak{r}_\ell =: \underline{e}^\mathfrak{r}_\ell > 0$, for all $\ell\in\mathcal{L}_r$. By further defining $\underline{e}^\mathfrak{r} \coloneqq \min_{\ell\in\mathcal{L}_r}\{\underline{e}^\mathfrak{r}_\ell\}$, we obtain 
\begin{align*}
	T_n \leq& -\underline{k}\|\widetilde{r}^\mathfrak{t}\varepsilon^\mathfrak{t}\|^2 - \underline{k}\underline{e}^\mathfrak{r}\sum_{\ell \in\mathcal{L}_r} (r^\mathfrak{r}_\ell)^2\xi^\mathfrak{r}_\ell \\
	\leq & -\underline{m} \| \kappa \|^2,
\end{align*}
where $\kappa \coloneqq [ (\widetilde{r}^\mathfrak{t}\varepsilon^\mathfrak{t} )^\top, r^\mathfrak{r}_1\sqrt{\xi^\mathfrak{r}_1},\dots, r^\mathfrak{r}_{n_r}\sqrt{\xi^\mathfrak{r}_{n_r}} ]^\top$, and $\underline{m} \coloneqq \min\{\underline{k},\underline{k}\underline{e}^\mathfrak{r}\}$.


Moreover, the fact that $q_\textup{d}(t)$ is bounded and $(\bar{q}(t),t)\in \Omega$ for $t\in I_t$ implies that $q^\mathfrak{t}_1(t)$ is bounded as $\|q^\mathfrak{t}(t)\| \leq \sup_{t\geq t_0} \|q^\mathfrak{t}_\textup{d}(t) \|+ \sqrt{n_{tr}}\max_{j\in\{1,\dots,n_{tr}\}}\{\bar{\rho}^\mathfrak{t}_j\}$ and $\|e_2(t)\| \leq \sqrt{n} \max_{m\in\{1,\dots,n\}}\{\bar{\rho}_{2_m}\}$, for $t\in I_t$. Note that the aforementioned bounds do not depend on $t_{\max}$. Hence, we conclude by Assumption \ref{ass:dynamics} that $f_1(q_1(t),t)$, $g_1(q_1(t),t)$ are bounded in $I_t$, by bounds independent of $t_{\max}$. Next, owing to the boundedness of $q^\mathfrak{r}_1(t)$ and $\dot{q}_\text{d}$, $\widetilde{\rho}^{-1}$ (by definition and assumption, respectively), as well as by using $\xi^\mathfrak{r}_\ell < \sqrt{\xi^\mathfrak{r}_\ell} < 1$, for all $\ell\in\mathcal{L}_r$, we conclude that there exists a positive finite constant $\bar{B}_1$, independent of $I_t$, satisfying
$T_b \leq \bar{B}_1 \| \kappa \|$, for all $t\in I_t$. Therefore, $\dot{V}_1$ becomes
\begin{align*}
	\dot{V}_1 \leq -\underline{m}\|\kappa\|^2 + \bar{B}_1 \|\kappa\|
\end{align*}
for all $t\in I_t$. Therefore, $\dot{V}_1$ is negative when $\|\kappa\| > \frac{\bar{B}_1}{\underline{m}}$, i.e., when 
\begin{align} \label{eq:Vdot1 condition}
	\sqrt{\sum_{j\in\mathcal{L}_t} (r^\mathfrak{t}_j \varepsilon^\mathfrak{t}_j)^2 + \sum_{\ell\in\mathcal{L}_r} (r^\mathfrak{r}_\ell)^2 \xi^\mathfrak{r}_\ell} > \frac{\bar{B}_1}{\underline{m}},
\end{align}
with $\mathcal{L}\coloneqq \{1,\dots,n_{tr}\}$. From the definition of $r^\mathfrak{t}_j$ in \eqref{eq:epsilon+r}, it holds that $r^\mathfrak{t}_j(t) \geq 2$, for all $j\in\mathcal{L}_t$ and $\forall t\in I_t$. Moreover, one can conclude by inspection that the function $\mathsf{f}(\mathsf{x}) = \frac{1}{(1 - \mathsf{x})^2}\mathsf{x} - \ln\left(\frac{1}{1-\mathsf{x}}\right)$ is positive for positive $\mathsf{x}$, Therefore, since by definition $\xi^\mathfrak{r}_\ell \geq 0$ it holds that $(r^\mathfrak{r}_\ell)^2\xi^\mathfrak{r}_\ell \geq  \varepsilon^\mathfrak{r}_\ell$, for all $\ell \in \mathcal{L}_r$. Therefore, it holds that $\sqrt{\sum_{j\in\mathcal{L}_t} (r^\mathfrak{t}_j \varepsilon^\mathfrak{t}_j)^2 + \sum_{\ell\in\mathcal{L}_r} (r^\mathfrak{r}_\ell)^2 \xi^\mathfrak{r}_\ell} \geq \sqrt{\sum_{j\in\mathcal{L}_t} (\varepsilon^\mathfrak{t}_j)^2 + \sum_{\ell\in\mathcal{L}_r} \varepsilon^\mathfrak{r}_\ell}$ and a sufficient condition for $\dot{V}_1$ to be negative is $\sqrt{\sum_{j\in\mathcal{L}_t} (\varepsilon^\mathfrak{t}_j)^2 + \sum_{\ell\in\mathcal{L}_r} \varepsilon^\mathfrak{r}_\ell} > \frac{\bar{B}_1}{\underline{m}_1}$, from which we conclude, by applying Theorem 4.18 of \cite{khalil_nonlinear_systems}, that there exists a positive constant $\bar{\varepsilon}$ such that $\varepsilon^\mathfrak{t}_j(t)$ and $\varepsilon^\mathfrak{r}_\ell(t)$ are bounded as 
\begin{align*}
	|\varepsilon^\mathfrak{t}_j(t)| \leq  \bar{\varepsilon} \\
	\varepsilon^\mathfrak{r}_\ell(t) \leq  \bar{\varepsilon},
\end{align*}
for all $t \in I_t$, $j\in\mathcal{L}_t$, $\ell\in\mathcal{L}_r$, which implies via \eqref{eq:epsilon+r} that 
\begin{align} \label{eq:xi bounded}
	|\xi^\mathfrak{t}_j(t)| &\leq \bar{\xi}^\mathfrak{t}\coloneqq\frac{\exp(\bar{\varepsilon}) - 1}{\exp(\bar{\varepsilon} + 1} < 1 \\
	\xi^\mathfrak{r}_\ell(t) &\leq \bar{\xi}^\mathfrak{r} \coloneqq \frac{\exp(\bar{\varepsilon})-1}{\exp(\bar{\varepsilon})} < 1,
\end{align}
for all $t \in I_t$, $j\in\mathcal{L}_t$, $\ell\in\mathcal{L}_r$. Hence, $\alpha_1(t)$, as designed in \eqref{eq:alpha}, is bounded, for all $t \in I_t$, from which we also conclude the boundedness of $q_2 = e_2 + \alpha_1$, since $\|e_2(t)\| = \|\rho_i(t)\xi_2(t)\| \leq \sqrt{n}\max_{m\in\{1,\dots,n\}}\{\bar{\rho}_{2_m}\}$ for all $t \in I_t$. Moreover, by invoking \eqref{eq:xi bounded}, it is straightforward to also conclude the boundedness of $\dot{\alpha}_1$, for all $t \in I_t$.

\textit{Step} $i\in \{2,\dots, k\}$: We apply recursively the aforementioned line proof for the remaining step. By considering the function $V_i = \frac{1}{2}\varepsilon_i^\top K_i \varepsilon$, we obtain
\begin{align*}
	\dot{V}_i \leq & -\varepsilon_i^\top r_i \rho_i^{-1} K_i g_i K_i \rho_i^{-1} r_i \varepsilon_i  \\
	& + \|  r_i \rho_i^{-1} K_i \varepsilon_i\|\| f_i + g_i e_{i+1} - \dot{\alpha}_{i-1} - \dot{\rho}_i \xi_i  \|,
\end{align*}
for $i \in \{2,\dots,k-1\}$, and 
\begin{align*}
\dot{V}_k \leq & -\varepsilon_k^\top r_k \rho_k^{-1} K_kg_k K_k \rho_k^{-1} r_k \varepsilon_k  \\
& + \|  r_k \rho_k^{-1} K_k \varepsilon_k\|\| f_k - \dot{\alpha}_{k-1} - \dot{\rho}_k \xi_k  \|,
\end{align*}
from which we conclude the boundedness of $\varepsilon_i$ and $\xi_i$ as 
\begin{align} \label{eq:xi_i bounded}
	\|\varepsilon_i(t)\| \leq \bar{\varepsilon}_i \Rightarrow \|\xi_i(t)\| \leq \bar{\xi}_i \coloneqq \frac{\exp(\bar{\varepsilon}_i)-1}{\exp(\bar{\varepsilon}_i)+1},
\end{align}
for all $t \in I_t$ for positive finite constants $\bar{\varepsilon}_i$. As a consequence, all intermediate signals $\alpha_i$ and system states $q_{i+1}$, $i\in\{2,\dots,k-1\}$, as well as the control law \eqref{eq:control law} remain bounded for all $t\in I_t$.

What remains to be shown is $t_{\max} = \infty$. Notice that \eqref{eq:xi bounded} and \eqref{eq:xi_i bounded} imply that the system remains bounded in a compact subset of $\Omega$, i.e., $(\bar{q}(t),t) \in \bar{\Omega} \subset \Omega$, for all $t \in I_t$. Since $\bar{q}(t)$ has been proven bounded, the conditions of \cite[Theorem 2.1.4]{bressan2007introduction} hold and we conclude hence that $\tau_{\max} = \infty$.


%


\section*{Acknowledgment}

The authors would like to thank Robin Baran for his decisive help in the hardware experiments.

\ifCLASSOPTIONcaptionsoff
  \newpage
\fi

\bibliographystyle{IEEEtran}
\bibliography{bibl}
\end{document}